\documentclass{article} % For LaTeX2e

\newif\ifreport
\reporttrue
\usepackage{iclr2026_conference,times}
\setcitestyle{numbers,square,comma,aysep={,},yysep={;}}
\usepackage{fontspec}
\usepackage{amsmath,amsfonts,bm}

\def\eqref#1{equation~\ref{#1}}
\def\1{\bm{1}}

\def\vx{{\bm{x}}}

\def\mI{{\bm{I}}}

\DeclareMathAlphabet{\mathsfit}{\encodingdefault}{\sfdefault}{m}{sl}
\SetMathAlphabet{\mathsfit}{bold}{\encodingdefault}{\sfdefault}{bx}{n}

\def\gN{{\mathcal{N}}}

\def\gR{{\mathcal{R}}}

\newcommand{\E}{\mathbb{E}}

\newcommand{\R}{\mathbb{R}}

\DeclareMathOperator*{\argmax}{arg\,max}

\usepackage{titletoc} % book-style appendix contents (partial toc)
\usepackage{hyperref}
\usepackage{url}
\usepackage{graphicx}
\usepackage{float}   % [H] for the pinned page-1 teaser
\usepackage{wrapfig}  % text-wrapped small figures (space utilization for narrow single-panel plots)
\usepackage{rotating} % sidewaystable for the wide per-metric CI table (appendix)
\usepackage{booktabs}
\usepackage{multirow}
\usepackage{amsmath,amssymb}
\usepackage{xcolor}
\usepackage{colortbl} % \rowcolor/\cellcolor for the table canon (design lane D)
\usepackage{xspace}   % smart trailing space for the \cachedsearch wordmark (design lane E)
\usepackage{subcaption}
\usepackage{algorithm}
\usepackage{algpseudocode}
\newtheorem{proposition}{Proposition}
\newtheorem{corollary}{Corollary}
\newtheorem{remark}{Remark}
\newtheorem{assumption}{Assumption}

\ifreport
  \newcommand{\reportonly}[1]{#1}
  \newcommand{\confonly}[1]{}
\else
  \newcommand{\reportonly}[1]{}
  \newcommand{\confonly}[1]{#1}
  \usepackage{enumitem}
  \setlist[itemize]{itemsep=2pt, topsep=3pt, parsep=0pt}
  \makeatletter
  \long\def\@makecaption#1#2{%
    \vskip\abovecaptionskip
    \small
    \sbox\@tempboxa{#1: #2}%
    \ifdim \wd\@tempboxa >\hsize
      #1: #2\par
    \else
      \global \@minipagefalse
      \hb@xt@\hsize{\hfil\box\@tempboxa\hfil}%
    \fi
    \vskip\belowcaptionskip}
  \makeatother
\fi
\newcommand{\tabsize}{\ifreport\small\else\footnotesize\fi}

\ifreport
  \newfontfamily\heros{texgyreheros}[Extension=.otf,
    UprightFont=*-regular, BoldFont=*-bold]
  \definecolor{reptitle}{HTML}{0F172A}
  \usepackage[most]{tcolorbox}
  \definecolor{repabsbg}{HTML}{EFE7D8}   % abstract-panel tint (warm parchment)
  \providecolor{NavyBlue}{HTML}{1F4E79}
  \hypersetup{colorlinks=true, citecolor=repaccent, linkcolor=NavyBlue!70!black, urlcolor=repaccent}
  \definecolor{repaccent}{HTML}{BF5700}  % UT burnt orange accent
  \makeatletter
  \def\section{\@startsection {section}{1}{\z@}{-2.0ex plus
      -0.5ex minus -.2ex}{1.5ex plus 0.3ex
  minus0.2ex}{\large\bfseries\raggedright\color{repaccent}}}
  \def\subsection{\@startsection{subsection}{2}{\z@}{-1.8ex plus
  -0.5ex minus -.2ex}{0.8ex plus .2ex}{\normalsize\bfseries\raggedright\color{repaccent}}}
  \makeatother
  \newtcolorbox{titlepanel}{enhanced, breakable, colback=repabsbg,
    colframe=repaccent!35, boxrule=0.4pt, arc=2.5mm,
    left=12pt, right=12pt, top=10pt, bottom=10pt,
    boxsep=0pt, before skip=4pt, after skip=14pt}
  \newtcolorbox{takeawaybox}{enhanced, breakable, colback=black!4,
    colframe=black!28, boxrule=0.4pt, arc=1.5mm,
    left=7pt, right=7pt, top=5pt, bottom=5pt,
    before skip=9pt plus 2pt, after skip=9pt plus 2pt}
  \long\def\takeaway#1{\begin{takeawaybox}\noindent\textbf{Takeaway.} #1\end{takeawaybox}}
\else
  \long\def\takeaway#1{\smallskip\noindent\textbf{Takeaway.} #1}
\fi

\newcommand{\tauc}{\tau}                       % caching threshold
\newcommand{\rhosp}{\rho}                      % Spearman rank correlation
\newcommand{\cachedsearch}{\textsc{CachedSearch}\xspace}
\newcommand{\wan}{Wan2.1-T2V-1.3B}
\newcommand{\ImageRewardMetric}{ImageReward~\citep{xu2023imagereward}}
\newcommand{\VideoScoreMetric}{VideoScore~\citep{he2024videoscore}}
\newcommand{\VQAScoreMetric}{VQAScore~\citep{lin2024vqascore}}
\newcommand{\VBenchMetric}{VBench~\citep{huang2023vbench}}
\newcommand{\VBenchTwoMetric}{VBench-2.0~\citep{zheng2025vbench2}}
\newcommand{\LPIPSMetric}{LPIPS~\citep{zhang2018lpips}}
\newcommand{\PSNRMetric}{PSNR~\citep{wang2004ssim}}
\newcommand{\SSIMMetric}{SSIM~\citep{wang2004ssim}}
\newif\ifshowci\showcifalse
\newcommand{\ci}[1]{\ifshowci{}_{[#1]}\fi}
\definecolor{mcWanA}{HTML}{2A78D6}  % Wan2.1-1.3B
\definecolor{mcWanB}{HTML}{56A0E8}  % Wan2.2-TI2V-5B
\definecolor{mcWanC}{HTML}{0D366B}  % Wan2.1-14B
\definecolor{mcCog}{HTML}{C98500}   % CogVideoX-5B
\definecolor{mcHun}{HTML}{199E70}   % HunyuanVideo-13B
\definecolor{mcLtx}{HTML}{E34948}   % LTX-Video-2B
\newcommand{\modelchip}[1]{\textcolor{#1}{\rule[-0.18ex]{1.3ex}{1.3ex}}}
\newcommand{\logomark}[1]{\raisebox{-0.28ex}{\includegraphics[height=1.5ex]{#1}}}
\makeatletter
\newcommand{\logocell}[1]{\@nameuse{logocell@#1}}
\newcommand{\chipcell}[1]{\@nameuse{chipcell@#1}}
\@namedef{logocell@wan13b}{\logomark{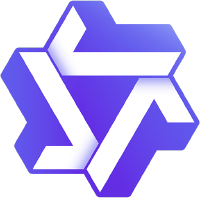}\,\textsc{Wan2.1-1.3B}}
\@namedef{logocell@wan22}{\logomark{figs/logos/wan.png}\,\textsc{Wan2.2-TI2V-5B}}
\@namedef{logocell@wan14b}{\logomark{figs/logos/wan.png}\,\textsc{Wan2.1-14B}}
\@namedef{logocell@cog}{\logomark{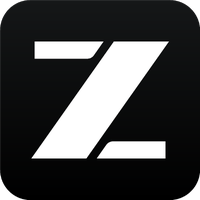}\,\textsc{CogVideoX-5B}}
\@namedef{logocell@hunyuan}{\logomark{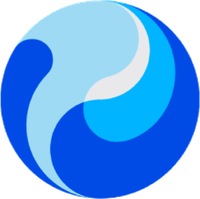}\,\textsc{HunyuanVideo-13B}}
\@namedef{logocell@ltx}{\logomark{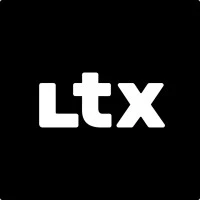}\,\textsc{LTX-Video-2B}}
\@namedef{chipcell@wan13b}{\modelchip{mcWanA}\,Wan2.1-1.3B}
\@namedef{chipcell@wan22}{\modelchip{mcWanB}\,Wan2.2-TI2V-5B}
\@namedef{chipcell@wan14b}{\modelchip{mcWanC}\,Wan2.1-14B}
\@namedef{chipcell@cog}{\modelchip{mcCog}\,CogVideoX-5B}
\@namedef{chipcell@hunyuan}{\modelchip{mcHun}\,HunyuanVideo-13B}
\@namedef{chipcell@ltx}{\modelchip{mcLtx}\,LTX-Video-2B}
\makeatother

\title{CachedSearch: Training-Free Cached Exploration\\ for Test-Time Search in Video Diffusion}

\ifreport
\iclrfinalcopy
\makeatletter
\def\@maketitle{}
\makeatother
\fi
\author{Anonymous authors\\
Paper under double-blind review}

\begin{document}

% Title sits alone above the panel (report: \@maketitle override; conference:
% classic ICLR title page). The report panel then wraps the author block, the
% abstract paragraph (no "Abstract" heading), and the metadata + logos band.
\maketitle
% Report edition uses \iclrfinalcopy; correct the running header, which the
% style otherwise sets to the (false) "Published as..." string. No venue
% name anywhere in the report edition (user decision 2026-07-14).
\ifreport\lhead{Technical report}\fi

\ifreport
\begin{titlepanel}
% Title block inside the panel (LumaFlux house style): bold, left-aligned,
% own break point so it sets in exactly TWO lines.
{\heros\bfseries\fontsize{17pt}{21.5pt}\selectfont\raggedright\color{reptitle}CachedSearch: Training-Free Cached Exploration\\ for Test-Time Search in Video Diffusion\par}
\vskip 9pt
% Author block (left-aligned, on the panel).
{\raggedright
{\heros\bfseries\normalsize Shreshth Saini\textsuperscript{1,2,$\dagger$}\quad Neil Birkbeck\textsuperscript{2}\quad Yilin Wang\textsuperscript{2}\quad Balu Adsumilli\textsuperscript{2}\quad Alan C.\ Bovik\textsuperscript{1,3}\par}
\vskip 3pt
{\small \textsuperscript{1}The University of Texas at Austin \qquad \textsuperscript{2}Google, Inc. \qquad \textsuperscript{3}University of Colorado Boulder\par}
\par}
\vskip 11pt
% Abstract paragraph, no heading (justified, full panel width).
% ===========================================================================
% ABSTRACT - clarity/slimness rewrite; all claims and numbers remain measured.
% Sources: PROGRESS.md / pillars/B-budget-optimal/PLAN.md (gate v0, tau sweep,
% strategy simulation, official-suite replication, scale, and stacking).
% ===========================================================================
Test-time search lets small video diffusion models rival larger ones, but costs $2$--$10\times$ more. All candidates are fully denoised, although most are discarded. Training-free caching makes each rollout $2$--$3\times$ faster at near-lossless quality. Composition is safe only if lossy caching preserves verifier rankings. We present the first study of whether caching corrupts candidate ranking in video test-time search. On \wan{} with an adaptive caching wrapper (${\sim}2\times$ per-candidate speedup), ImageReward scores seed-matched cached and full rollouts. Median per-prompt Spearman rank correlation is $0.905$, with $72\%$ top-1 agreement on the VBench suite.\reportonly{ VBench-2.0 replicates this result on a harder suite.} Recomputing the cached winner at full compute retains $90$--$94\%$ of the full-search gain. Errors cluster among near-tied candidates, making corruption self-limiting. This finding leads to \cachedsearch{}. It explores every candidate with aggressive caching, then re-generates only the winner at full compute. At $N{=}8$, it captures $94.7\%$ of best-of-$N$'s gain at $63\%$ of the cost. Capture rises with width. At matched budget, it searches twice as wide for $38\%$ more gain. The result holds from $1.3$B--$14$B across six models and four families: Wan, LTX, CogVideoX, and Hunyuan. Wan2.1-14B matches the 1.3B model's fidelity. Mid-trajectory pruning multiplies the exploration saving to $3.11\times$ at $88.6\%$ capture. Ports to other model families require recalibrating a single parameter, showing that fidelity tracks architecture rather than parameter count. \cachedsearch{} is training-free, verifier-agnostic, and orthogonal to the search algorithm, making it a plug-in multiplier for test-time scaling.

\vskip 11pt
% Metadata band: links left, YouTube + UT Austin logos right.
{\small
\noindent
\begin{minipage}[c]{0.63\linewidth}
\textbf{Date:} \today\\[2.5pt]
% Links restored per user 2026-07-14 ("just add the links, I will update"):
% repo is private and the project page launches later; URLs are the targets.
% arXiv ID added 2026-07-27, on announcement of v1.
\textbf{arXiv:} \href{https://arxiv.org/abs/2607.23159}{\textcolor{repaccent}{\texttt{arXiv:2607.23159}}}\\[2.5pt]
\textbf{Code:} \href{https://github.com/shreshthsaini/CachedSearch}{\textcolor{repaccent}{\texttt{github.com/shreshthsaini/CachedSearch}}}\\[2.5pt]
\textbf{Project page:} \href{https://shreshthsaini.github.io/CachedSearch}{\textcolor{repaccent}{\texttt{shreshthsaini.github.io/CachedSearch}}}
\end{minipage}\hfill
\begin{minipage}[c]{0.35\linewidth}
\raggedleft
\raisebox{-0.5\height}{\includegraphics[height=12pt]{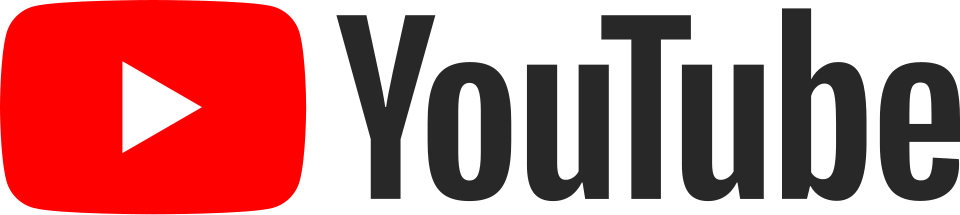}}%
\hspace{10pt}%
\raisebox{-0.5\height}{\includegraphics[height=16pt]{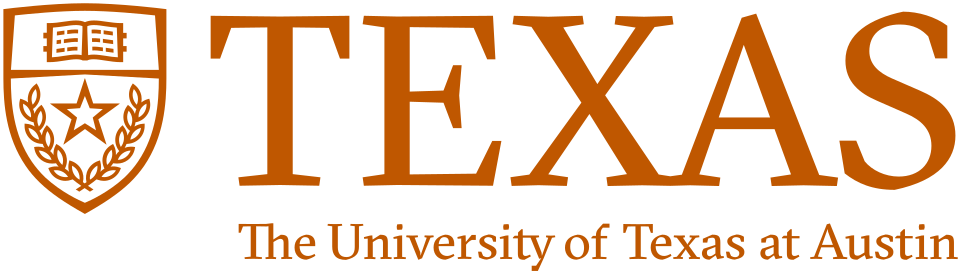}}
\end{minipage}}
\end{titlepanel}
% Dagger note as a real page-1 footnote (user 2026-07-24); unnumbered.
{\let\thefootnote\relax\footnotetext{\textsuperscript{$\dagger$}Work done while at The University of Texas at Austin.}}
% The anonymous panel is shorter than the authored one; without this, the
% page-1 teaser [t]-float floats ABOVE the title. Defer top floats to p.2.
\suppressfloats[t]
\else
\begin{abstract}

\end{abstract}
\fi

% ======================================================================
% PAGE-1 TEASER (design-overhaul Lane A; PLAN D4). One mini-panel per
% contribution. Placed between the abstract and Section 1 (user 2026-07-24):
% `H` pins it in place so it cannot float above the title on page 1.
% Built by code/paper_figs/make_fig1.py; every number
% traces to ci_numbers.json / b1_gate_v0 / x_taucal.tex. Relocated here from
% sections/intro.tex (old two-panel fig1 retired).
% ======================================================================
% Placement differs by edition (user 2026-07-24):
%   conference: [H] pins it between the abstract and Section 1;
%   report:     [t] float, deferred by \suppressfloats to head page 2.
% The float option is read before macro expansion, so the environment is
% written once per edition with a shared body:
%   conference: [H] pins it between the abstract and Section 1;
%   report:     [t] float, which \suppressfloats defers to page 2.
\newcommand{\figonebody}{%
    \centering
    \includegraphics[width=\ifreport 1.0\else 0.92\fi\linewidth]{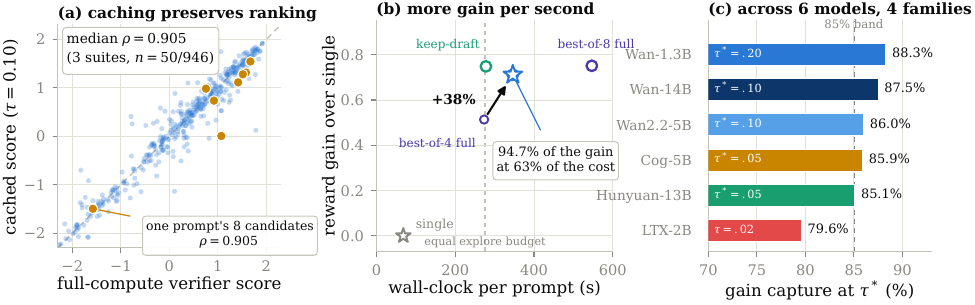}
    \caption{\textbf{\cachedsearch{} explores cheaply and commits at full compute.} It keeps most
    of test-time search's quality at a fraction of the cost and generalizes
    across models. \textbf{(a)}~Cached and full-compute rollouts of the same seeds
    rank candidates alike (\wan{} \citep{wan2025wan}, $\tauc{=}0.10$\reportonly{; each dot a candidate,
    one prompt's $8$ highlighted}): median per-prompt Spearman
    $\rhosp=0.905$\reportonly{, recurring on three independent suites}.
    \textbf{(b)}~At $N{=}8$, committing full compute to the cached winner captures
    $94.7\%$ of best-of-$8$'s reward gain at $63\%$ of its wall-clock
    cost\reportonly{; at matched budget it searches twice as wide for $+38\%$ gain
    over full-compute best-of-$4$}\reportonly{ ($95\%$ prompt-bootstrap CI,
    $B{=}10^4$, $n{=}50$)}. \textbf{(c)}~Gain capture at each model's calibrated
    $\tau^*$ across six models and four architecture families\reportonly{; five of
    six clear the $85\%$ band, LTX-Video-2B \citep{hacohen2024ltxvideo} is the honest boundary}.}
    \label{fig:fig1}
}
\ifreport
\begin{figure}[t]\figonebody\end{figure}
\else
\begin{figure}[H]\figonebody\end{figure}
\fi

% ===========================================================================
% INTRODUCTION - clarity/slimness rewrite. All numbers remain measured.
% The conference and report editions share the core story; report-only text is
% limited to complete sentences. Citation keys resolve in references.bib.
% ===========================================================================
\section{Introduction}
\label{sec:intro}

Video diffusion is the most expensive mainstream generative workload per output. In our setting, one 81-frame, 50-step rollout of \wan{} \citep{wan2025wan}, a 1.3B-parameter model, takes 68.3\,s on a modern GPU (Section~\ref{sec:experiments}). Test-time search multiplies this cost. It samples $N$ candidates, scores them with a verifier, and keeps the best \citep{ma2025inference,liu2025videot1,he2025evosearch}. Even pruning-based video searches pay $2$--$10\times$ the single-sample cost \citep{liu2025videot1,latsearch2026}. Every candidate is generated at full cost, including those the method discards.
Training-free caching attacks the cost of one rollout. It reuses features across denoising steps and yields $2$--$3\times$ speedups with near-lossless fidelity \citep{zhao2024pab,liu2024teacache,lv2024fastercache,zhou2025easycache,lemica2025}. Yet this literature studies only seed-matched fidelity. It asks whether acceleration produces the same video, not whether it preserves a later selection decision.
No prior work composes caching with test-time search. The obstacle is subtle. Caching is lossy, but search does not require exact candidates. It requires only their relative ordering under the verifier. If caching reshuffles that ordering, search chooses the wrong candidate and the apparent saving disappears. If the ordering survives, search can explore much more cheaply.

We test this directly with seed-matched cached and full-compute twins, both scored by ImageReward \citep{xu2023imagereward}. Rankings survive. On the VBench suite \citep{huang2023vbench}, the median per-prompt Spearman correlation is $\rhosp=0.905$ and top-1 agreement is $72\%$ (Section~\ref{sec:exp-vbench}).\reportonly{ VBench-2.0 provides a third replication \citep{zheng2025vbench2}.} Ranking errors also have low stakes. They concentrate on prompts whose candidates are nearly tied, where choosing the wrong candidate costs little (Section~\ref{sec:abl-corruption}). Section~\ref{sec:exp-setup} defines ranking, gain, regret, and cost formally.
This result leads to \cachedsearch{}: \textbf{explore cheap, commit full}. Generate all $N$ candidates with aggressive caching. Score the drafts, then re-generate only the winner at full compute from its seed. At $N{=}8$, this retains $94.7\%$ of full best-of-$8$'s reward gain at $63\%$ of its cost. At the budget of full best-of-$4$, it instead searches eight candidates and gains $38\%$ more reward. It searches twice as wide for the same money.
\cachedsearch{} changes only candidate rollout. It is cache-agnostic and orthogonal to both the search algorithm and verifier. The protocol transfers across six models, four architecture families, and the $1.3$B--$14$B scale range. Off-family models such as CogVideoX-5B \citep{yang2024cogvideox} need only a recalibrated $\tauc$. It also stacks with candidate pruning, reaching a $3.11\times$ exploration speedup. These properties make caching a plug-in multiplier for test-time search.

\paragraph{Contributions.}
\begin{itemize}
    \item \textbf{The first ranking-preservation study for caching under search:} a seed-matched protocol evaluated on the gate grid, the VBench suite, and VBench-2.0. Median $\rhosp$: $0.905$/$0.905$/$0.881$; outcomes are scored as regret in delivered quality, not just rank correlation.
    \item \textbf{When ranking breaks, and why it rarely matters:} corruption concentrates on low-spread prompts (corr(spread, $\rhosp$) $=+0.31$; Section~\ref{sec:abl-corruption}). Median regret is $0$, and ${\sim}94\%$ of search gains remain. A spread-based adaptive per-prompt $\tauc$ does not beat a fixed threshold (Section~\ref{sec:abl-adaptive}).
    \item \textbf{\cachedsearch{}.} A training-free explore-cheap, commit-full method that retains $94.7\%$ of full best-of-$8$ gain at $63\%$ cost and gives $38\%$ more reward at iso-cost.
    \item \textbf{A speed--fidelity trade-off law:} the $\tauc$ frontier spans $1.58$--$2.41\times$ speedup while keeping capture $\geq88\%$. Its shape follows the self-limiting corruption above.
    \item \textbf{Composability and generality:} pruning stacks multiplicatively ($3.11\times$ exploration speedup at $88.6\%$ capture; Section~\ref{sec:abl-stack}). Wan2.1-14B \citep{wan2025wan} again reaches median $\rhosp=0.905$. One $\tauc$ recalibration transfers the protocol to CogVideoX-5B and, more broadly, six models and four families spanning $1.3$B--$14$B (Section~\ref{sec:abl-cog}). A second video-native verifier changes capture by at most $2.5$ points (Section~\ref{sec:abl-verifier}).
\end{itemize}

% =============================================================================
% PRELIMINARIES - split out of Method (user 2026-07-24): Section 2 states the
% model, best-of-N search, and cost notation; Section 3 is the method proper.
% Same single-source x_method_flow block as before (conference glue inline).
% =============================================================================
\section{Preliminaries}
\label{sec:prelim}  % keep label: refs across sections

We fix the generative model, the search primitive, and the cost notation
used throughout.

\ifreport
% SHARED BLOCK (both editions; single source). Report edition: inlined in
% sections/method.tex Section Preliminaries. Conference edition: hosted in the
% appendix under app:method-details (sections/appendix.tex).
\paragraph{Flow-matching video diffusion.}
We consider latent video diffusion models trained with flow matching, as
instantiated by \wan{} \citep{lipman2023flow,liu2022rectified,wan2025wan}.
A video is represented by a spatio-temporally compressed
latent $\vx \in \R^{C \times F \times H \times W}$; a diffusion transformer
(DiT) $v_\theta(\vx_t, t, c)$ predicts the velocity field transporting a
Gaussian sample $\vx_1 \sim \gN(0, \mI)$ to a data sample $\vx_0$,
conditioned on a text prompt $c$. Sampling integrates the probability-flow
ODE over a decreasing grid of $T$ timesteps
$1 = t_1 > \dots > t_T > t_{T+1} = 0$ with an Euler-type update,
\begin{equation}
\vx_{t_{k+1}} \;=\; \vx_{t_k} + (t_{k+1} - t_k)\,\tilde v_k ,
\label{eq:euler}
\end{equation}
followed by VAE decoding of $\vx_{t_{T+1}}$ into pixels. With classifier-free
guidance (CFG) \citep{ho2022classifierfree} at scale $w$, each step evaluates
the transformer once per guidance branch and combines
\begin{equation}
\tilde v_k \;=\; v_\theta(\vx_{t_k}, t_k, \varnothing)
 + w \big( v_\theta(\vx_{t_k}, t_k, c) - v_\theta(\vx_{t_k}, t_k, \varnothing) \big),
\label{eq:cfg}
\end{equation}
so a rollout costs $2T$ transformer evaluations, which dominate wall-clock
time. Throughout, a \emph{rollout} is the map $G(c, s) \mapsto y$ from a
prompt $c$ and a seed $s$ (which fixes $\vx_1$) to a decoded video $y$; the
ODE solver is deterministic, so the seed fully determines the sample.

\else
% Conference glue: compact notation; the full sampling equations live in
% Appendix app:method-details (sections/x_method_flow.tex).
\paragraph{Model and rollouts.}
We consider latent video diffusion models trained with flow matching, as
instantiated by \wan{} \citep{lipman2023flow,liu2022rectified,wan2025wan}.
A diffusion transformer (DiT) $v_\theta$ follows a
deterministic $T$-step probability-flow-ODE schedule with classifier-free
guidance (CFG; \citealp{ho2022classifierfree}). Each rollout uses $2T$
transformer evaluations, one per guidance branch per step. These evaluations
dominate wall-clock time. A \emph{rollout} maps prompt $c$ and seed $s$ to a
decoded video, $G(c, s) \mapsto y$. The seed fixes the initial noise. The
deterministic solver therefore fixes the sample
(Appendix~\ref{app:method-details}).
\fi

\paragraph{Best-of-$N$ search with a verifier.}
Verifier-guided test-time search draws $N$ candidates from distinct seeds
\citep{ma2025inference,liu2025videot1}. It scores $y_i = G(c, s_i)$ with
$V(y, c) \in \R$. It returns $y_{i^\star}$, where
$i^\star = \argmax_i V(y_i, c)$. Our verifier is ImageReward
\citep{xu2023imagereward}, averaged over $K{=}8$ uniformly spaced frames.
We follow prior video test-time-search work in using a reward model as the
verifier \citep{liu2025videot1,ma2025inference,he2025evosearch}.
\reportonly{Frame-level verifiers can be reward-hacked \citep{ma2025inference}.
Section~\ref{sec:abl-verifier} repeats the audit with a video-native verifier
(VideoScore \citep{he2024videoscore}). Section~\ref{sec:abl-temporal} uses
direct temporal measurements.}

\paragraph{Cost model.}
Let $C_f$ denote the cost of one full-compute rollout. Let $C_c < C_f$ denote
the cost of one cached rollout. Their ratio is $\gamma = C_c / C_f$.
Verification is shared across strategies and costs little relative to
generation, so we omit it. \reportonly{Verification scores $8$ frames, whereas
DiT sampling takes tens of seconds.} Full-compute best-of-$N$ costs $N C_f$.

% =============================================================================
% METHOD - owned by the paper-method agent.
% Faithful to: code/videogen1/caching.py, code/videogen1/gen.py,
%              code/experiments/b1_gate.py, code/experiments/b1_simulate.py.
% Measured constants: PROGRESS.md + pillars/B-budget-optimal/PLAN.md only.
% Provenance ledger: paper/NOTES-method-agent.md.
%
% PREAMBLE REQUIREMENTS (skeleton agent, please add to main.tex):
%   \usepackage{algorithm}
%   \usepackage{algpseudocode}
% and bib entries: lipman2023flow, liu2022rectified, ho2022classifierfree
% (details in NOTES-method-agent.md). Theorem envs are defined locally in
% sections/theory.tex with \@ifundefined guards; moving them to main.tex is
% fine and the guards will yield automatically.
% =============================================================================

\section{Method}
\label{sec:method}  % section agents: KEEP this label - other sections \ref it

\cachedsearch{} separates the two jobs in test-time search. \emph{Exploration}
generates candidates for ranking. \emph{Delivery} generates the final video.
Ranking can survive even when pixel fidelity does not. We therefore cache
every exploration rollout aggressively. We then regenerate only the winning
seed at full compute (Figure~\ref{fig:method-flow}). \reportonly{Using the notation of Section~\ref{sec:prelim}, we define the
caching rule (Section~\ref{sec:caching}), then present the algorithm
(Section~\ref{sec:algorithm}) and its cost (Section~\ref{sec:cost}). }Appendix~\ref{app:theory}
analyzes how ranking noise becomes search regret.

% [design-overhaul Lane A] Method schematic promoted to its own full-width
% figure. Real \wan{} frames (b1_gate_vbench); bar heights are recorded
% ImageReward scores. Built by code/paper_figs/make_fig2.py.
\begin{figure}[t]
    \centering
    \includegraphics[width=\ifreport 1.0\else 0.78\fi\linewidth]{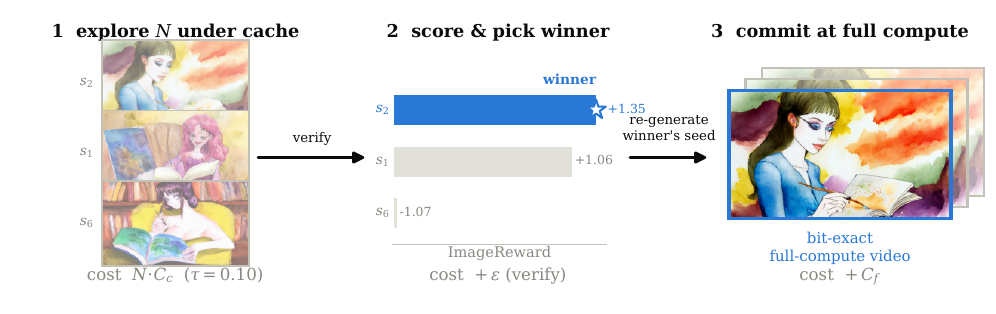}
    \caption{Overview of \cachedsearch{}. Cache every candidate, score the
    drafts, then re-generate the winning seed at full compute.}
    \label{fig:method-flow}
\end{figure}

\subsection{Adaptive transformation-vector caching}
\label{sec:caching}

We wrap the DiT with training-free caching for exploration. The wrapper uses
the \emph{transformation-vector} formulation of EasyCache
\citep{zhou2025easycache}. \reportonly{Section~\ref{sec:related-caching} surveys
this method family.}
\ifreport
% SHARED BLOCK (both editions; single source). Report edition: inlined in
% sections/method.tex Section Cached exploration (defines eq:delta / eq:skip).
% Conference edition: hosted in the appendix (app:method-details); the main
% text keeps an inline one-sentence statement of the same rule.
The wrapper intercepts every transformer call. Let $\vx$ denote the latent
input of the current call and $v_\theta(\vx)$ its output. Rather
than caching the output itself, we cache the transformation vector computed
at the most recent computed call, whose input we denote
$\vx_{\mathrm{ref}}$:
\begin{equation}
\Delta \;=\; v_\theta(\vx_{\mathrm{ref}}) - \vx_{\mathrm{ref}} .
\label{eq:delta}
\end{equation}
A skipped call returns the approximation
\begin{equation}
\hat v_\theta(\vx) \;=\; \vx + \Delta ,
\label{eq:skip}
\end{equation}
i.e., it assumes the residual $v_\theta(\vx) - \vx$ varies slowly
across adjacent steps even where $v_\theta$ itself does not.

\else
% Conference glue: same rule stated inline; display equations (eq:delta,
% eq:skip) live in Appendix app:method-details (sections/x_method_delta.tex).
The wrapper intercepts every transformer call. It caches the transformation
vector $\Delta = v_\theta(\vx_{\mathrm{ref}}) - \vx_{\mathrm{ref}}$ from the
most recent computed input $\vx_{\mathrm{ref}}$. A skipped call returns
$\hat v_\theta(\vx) = \vx + \Delta$. This approximation assumes that the
residual $v_\theta(\vx) - \vx$ changes slowly between adjacent steps
(Appendix~\ref{app:method-details}).
\fi

\paragraph{Adaptive skip rule.}
An accumulated input-change indicator controls skipping. At each call, the
wrapper adds the relative drift from the last computed input:
\begin{equation}
a \;\leftarrow\; a \;+\;
\frac{\lVert \vx - \vx_{\mathrm{ref}} \rVert_F}{\lVert \vx_{\mathrm{ref}} \rVert_F},
\label{eq:indicator}
\end{equation}
The wrapper skips while $a \le \tauc$. Once $a > \tauc$, it evaluates the
transformer and refreshes $\Delta$ and $\vx_{\mathrm{ref}}$. It then resets
$a \leftarrow 0$. \reportonly{Each skipped step adds the total drift since the
last computation. Thus, the indicator grows faster the longer a run of skips
continues and
ends long skip runs conservatively. The wrapper always computes the first
$K_w$ warmup steps and last $K_c$ cooldown steps. It also computes until the
first $\Delta$ exists. We use $K_w = K_c = 5$ with $T = 50$ throughout.} The
threshold $\tauc$ trades speed against fidelity
(Section~\ref{sec:experiments}).

\ifreport
% SHARED BLOCK (both editions; single source). Report edition: inlined at the
% end of sections/method.tex Section Cached exploration. Conference edition: hosted
% in the appendix under app:method-details (sections/appendix.tex).
\paragraph{Per-branch cache state.}
Under CFG the pipeline calls the transformer once per guidance branch per
step (Eq.~\ref{eq:cfg}); the conditional and unconditional inputs differ and
drift at different rates. The wrapper therefore keeps fully independent state
$(\Delta, \vx_{\mathrm{ref}}, a)$ per branch. It alternates between the two
guidance branches, keeping separate cache state so each branch schedules its
own skips.

\paragraph{Properties.}
The wrapper is training-free and model-agnostic. It replaces the transformer
module of an existing pipeline with no model changes. It also preserves
determinism. For fixed $(c, s, \tauc)$, the cached rollout
$G_{\tauc}(c, s)$ is deterministic, exactly like $G(c, s)$, which off mode
recovers. A static variant that computes every $k$-th step between warmup and
cooldown serves as a uniform-schedule baseline.

\else
% Conference glue: the per-branch state + properties paragraphs are hosted in
% Appendix app:method-details (sections/x_method_cache_extra.tex).
Under CFG, each guidance branch has independent cache state. The wrapper is
training-free and model-agnostic. It replaces the transformer without model
changes. Like $G(c, s)$, the cached rollout $G_{\tauc}(c, s)$ is deterministic
(Appendix~\ref{app:method-details}).
\fi

\subsection{CachedSearch}
\label{sec:algorithm}

\cachedsearch{} caches all $N$ exploration rollouts
(Algorithm~\ref{alg:cachedsearch}). It scores them and selects one seed.
\emph{Keep-draft} delivers the cached winner. \emph{Recommit} regenerates that
seed at full compute. Recommit stores no latents. Determinism ensures that
$G(c, s_{i^\star})$ exactly reproduces the winning seed's full-compute sample.
\reportonly{This is the sample full-compute best-of-$N$ would deliver after
selecting the same seed.} Caching therefore affects only seed selection in
recommit mode. We measure ranking fidelity in
Section~\ref{sec:experiments}. \reportonly{Our measure is per-prompt Spearman
$\rhosp$ between cached and full scores over matched seeds.} Appendix~\ref{app:theory}
connects ranking fidelity to regret and iso-cost gains
(Proposition~\ref{prop:isocost}).

\begin{algorithm}[t]
\confonly{\small}
\caption{\cachedsearch{}: cached exploration with optional full-compute commit}
\label{alg:cachedsearch}
\begin{algorithmic}[1]
\Require prompt $c$; seeds $s_1,\dots,s_N$; cache threshold $\tauc$;
verifier $V$; mode $\in \{\textsc{keep}, \textsc{commit}\}$
\For{$i = 1, \dots, N$} \Comment{exploration: $N$ cached rollouts, cost $N C_c$}
  \State reset cache state; \; $y_i \gets G_{\tauc}(c, s_i)$
    \Comment{Eqs.~\ref{eq:delta}--\ref{eq:indicator}}
  \State $r_i \gets V(y_i, c)$
\EndFor
\State $i^\star \gets \argmax_{i} r_i$
\If{mode $=$ \textsc{commit}}
  \State \Return $G(c, s_{i^\star})$
    \Comment{delivery, cost $C_f$; exact full-compute sample (seed-deterministic)}
\Else
  \State \Return $y_{i^\star}$ \Comment{deliver cached winner as-is}
\EndIf
\end{algorithmic}
\end{algorithm}

\reportonly{Keep-draft is cheaper but delivers a cached sample. Its quality
therefore requires an audit. Frame-level verifiers can miss temporal
artifacts. Section~\ref{sec:abl-temporal} audits temporal metrics, optical
flow, and \LPIPSMetric{}. Static temporal quality survives, but aggressive caching
dampens motion (mean flow $-8.0\%$ at $\tauc{=}0.20$). We therefore recommend
recommit by default. \cachedsearch{} is independent of the search procedure.
It accelerates any sample-then-rank loop. It also composes with candidate
pruning such as latent-reward filtering \citep{latsearch2026}.
Section~\ref{sec:abl-stack} confirms that their savings multiply.}

\subsection{Cost analysis}
\label{sec:cost}

With verification cost shared and omitted, the strategies cost
% Report edition: annotated with overbrace term labels (ScaleRL Eq.-1
% pattern; design lane E). Conference edition keeps the plain form --
% same equation, same label, page budget untouched.
\ifreport
\begin{equation}
\begin{aligned}
\mathrm{cost}(\textsc{keep}) &=
\overbrace{N C_c}^{\textcolor{repaccent}{\text{explore cheap}}}, &\qquad
\mathrm{cost}(\text{best-of-}N) &=
\overbrace{N C_f}^{\textcolor{repaccent}{\text{all candidates full}}}, \\[0.6ex]
\mathrm{cost}(\textsc{commit}) &=
\overbrace{N C_c}^{\textcolor{repaccent}{\text{explore cheap}}}
+ \overbrace{C_f}^{\textcolor{repaccent}{\text{commit full}}}. &&
\end{aligned}
\label{eq:costs}
\end{equation}
\else
\begin{equation}
\mathrm{cost}(\textsc{keep}) = N C_c,
\qquad
\mathrm{cost}(\textsc{commit}) = N C_c + C_f,
\qquad
\mathrm{cost}(\text{best-of-}N) = N C_f .
\label{eq:costs}
\end{equation}
\fi
Recommit is cheaper than full-compute best-of-$N$ iff $N C_c + C_f < N C_f$,
i.e., beyond the break-even width
\begin{equation}
N \;>\; N^\ast \;=\; \frac{1}{1 - \gamma} .
\label{eq:breakeven}
\end{equation}
On \wan{} \citep{wan2025wan}, we measure $C_f = 68.3$\,s and $C_c = 34.7$\,s at
$\tauc = 0.10$. This is a $1.97\times$ per-rollout speedup. \reportonly{The
setting uses a resolution of $480{\times}832$, $81$ frames, $T{=}50$, and one
GH200 GPU. The
speedup is $1.58\times$ at $\tauc{=}0.05$ and $2.41\times$ at
$\tauc{=}0.20$.} Thus, $\gamma \approx 0.51$ and $N^\ast \approx 2.0$.
$N{=}2$ is a cost wash. At $N{=}8$, recommit costs
$8 C_c + C_f \approx 346$\,s, versus $8 C_f \approx 546$\,s ($63\%$).
\reportonly{Keep-draft costs $\approx 278$\,s ($51\%$). Both ratios approach
$\gamma$ as $N$ grows. Under a fixed budget $B$, full search evaluates
$B/C_f$ candidates. \cachedsearch{} explores $(B - C_f)/C_c$ candidates.
With $\gamma \approx 0.5$, this is nearly twice as many candidates.}

Caching does not increase peak memory. Its only state is the per-branch
$(\Delta, \vx_{\mathrm{ref}}, a)$ from Section~\ref{sec:caching}. This uses
two latent-sized tensors and one scalar per guidance branch. \cachedsearch{}
stores nothing across candidates. \reportonly{It regenerates the winner from
its seed instead of saving latents.} Peak memory therefore matches a plain
rollout.

% ======================================================================
% Experiments (owner: paper-experiments agent).
% Every number is measured (sources: PROGRESS.md, pillars/B PLAN.md) and was
% re-derived from the raw per-candidate jsonl by code/paper_figs/make_figs.py,
% which also renders figs/f1-f7 and cross-checks each value (all checks
% match). Data: results/b1_gate_{v0,tau005,tau020,vbench}/scores_shard*.jsonl.
% NOTE: b1_gate_vbench (E2) FINAL: n=944/946 fully covered prompts (seeds 0-7;
% 2 prompts never reached full coverage). Figures + numbers re-derived at final
% coverage by make_figs.py on 2026-07-07 (polish pass); all checks pass.
% ======================================================================
\section{Experiments}
\label{sec:experiments}  % section agents: KEEP this label - other sections \ref it

We test ranking preservation, delivered gain per unit cost, the reward cost of
ranking errors, and replication on the \VBenchMetric{} suite and
\VBenchTwoMetric{}.

\subsection{Setup}
\label{sec:exp-setup}

\paragraph{Model and suites.}
The $50$-prompt gate grid is the pilot and configuration set, and every
headline finding is re-measured on the full official suites: $946$-prompt
VBench ($15{,}104$ paired rollouts) and $1{,}013$-prompt VBench-2.0.
The gate uses \wan{} \citep{wan2025wan} at $480{\times}832$, $81$ frames,
$T{=}50$, guidance $5.0$, $8$ seeds, and one NVIDIA GH200 per rollout; its
$800$ full/cached videos at $\tauc=0.10$ support strategy simulation, with
$400$ cached videos at each of $\tauc\in\{0.05,0.20\}$ for the sweep.
The VBench \texttt{all\_dimension} list and the harder VBench-2.0 benchmark
use the same $8$-seed paired protocol at $\tauc=0.10$, and we save suite
videos for multi-verifier rescoring \citep{huang2023vbench,zheng2025vbench2}.
Full generation is seed-deterministic; ImageReward
\citep{xu2023imagereward}, averaged over $8$ uniformly spaced frames, scores
all candidates. The two CFG branches keep independent cache state.

\paragraph{Timing and strategies.}
At $\tauc=0.10$, $C_f=68.3\pm4.5$\,s and $C_c=34.7\pm2.5$\,s
(mean$\pm$sd, $n{=}400$ per arm), a $1.97\times$ candidate speedup; batch
size one is fastest (Section~\ref{sec:abl-batch}), and
Appendix~\ref{app:repro} gives the timing protocol. We compare single
($C_f$), best-of-$N$ full ($NC_f$), keep ($NC_c$), and commit
($NC_c+C_f$), averaging measured scores and latencies over every
$\binom{8}{N}$ seed subset without extrapolation.

\paragraph{Definitions and protocol.}
For prompt $c$, $S_i=V(G(c,s_i),c)$ and
$\hat S_i=V(G_{\tauc}(c,s_i),c)$ are candidate $i$'s full and cached scores.
Delivered value scores the returned video, using $S_i$ for single,
best-of-$N$, and commit and $\hat S_i$ for keep; commit reproduces the
winning seed's full trajectory, while keep remains a nominal value audited
in Section~\ref{sec:abl-temporal}. \emph{Gain} subtracts the mean single
baseline; \emph{regret} is $\max_iS_i-S_{\argmax_i\hat S_i}$;
\emph{capture} is retained best-of-$N$ gain, reported either as the ratio of
mean gains (Table~\ref{tab:main}) or the mean per-prompt ratio in
Eq.~\ref{eq:capture} (Table~\ref{tab:tau} and the VBench suites);
\emph{speedup} is $C_f/C_c$; and \emph{cost} is end-to-end delivery time
(Eq.~\ref{eq:costs}). At the default, the cached transformer's
${\approx}0.49$ forward-pass ratio matches wall-clock
$\gamma=C_c/C_f=0.508$. We fixed $\tauc=0.10$, commit, and $N=8$ on the
gate before official-suite evaluation and verified the implementation with a
bit-exact no-cache control. All reported uncertainty uses a prompt bootstrap
($B=10^4$).
Appendix~\ref{app:repro} gives the determinism and seed protocol.

\subsection{Candidate-ranking preservation}
\label{sec:exp-ranking}

% Panel bodies defined ONCE and reused by both editions (report: two-panel
% figure below; conference: rho-hist here, rank-scatter in the appendix via
% sections/x_fig_scatter.tex). Keep captions here - single source.
\newcommand{\PanelRankScatter}{%
    \centering
    \includegraphics[width=\linewidth]{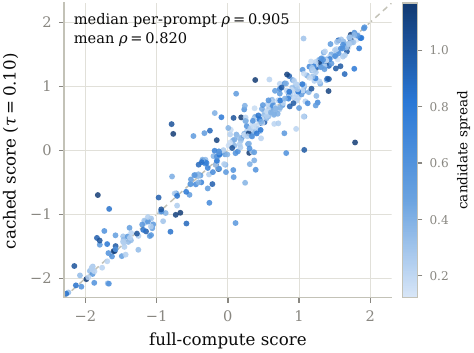}
    \caption{\textbf{Within-prompt candidate ranking survives caching:
    median Spearman $\rhosp=0.905$.} Cached vs.\ full scores for
    $50$ prompts and $8$ seeds at $\tauc=0.10$. Color encodes
    within-prompt score spread; each dot is one measured rollout pair.}
    \label{fig:rank-scatter}}
% Caption body single-sourced: used by \PanelRhoHist (report two-panel
% figure) and by the conference side-by-side layout below.
\newcommand{\RhoHistCap}{\textbf{Median $\rhosp$ stays at
    $0.86$--$0.90$ across caching levels and at $19\times$ prompt scale.}
    Per-prompt cached/full rank-correlation histograms for the gate sweep
    ($n=50$) and the \VBenchMetric{} suite at $\tauc=0.10$;
    Table~\ref{tab:tau} gives exact summaries.}
\newcommand{\PanelRhoHist}{%
    \centering
    \includegraphics[width=\linewidth]{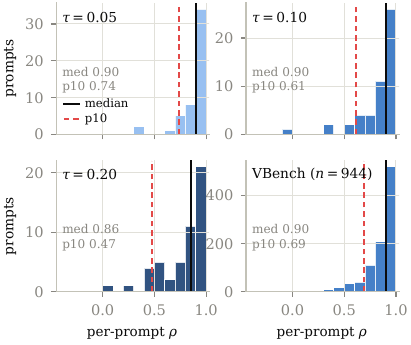}
    \caption{\RhoHistCap}
    \label{fig:rho-hist}}
\ifreport
\begin{figure}[!tb]
  \centering
  \begin{minipage}[t]{0.48\linewidth}
    \PanelRankScatter
  \end{minipage}\hfill
  \begin{minipage}[t]{0.48\linewidth}
    \PanelRhoHist
  \end{minipage}
\end{figure}
\fi
% Conference: fig:rho-hist rides beside tab:main in Section exp-main (below),
% so the histogram grid and the strategy table balance one another.

On the VBench suite, median within-prompt Spearman correlation is
\textbf{$\rhosp=0.905$}, top-1 agreement is \textbf{$72\%$}, and p10 is
$0.690$. The gate pilot gives the same median, with $64\%$ top-1, mean
$0.820$, p10 $0.614$, and $9/50$ prompts below $\rhosp=0.7$.
Figures~\ref{fig:rho-hist} and~\ref{fig:rank-scatter}\confonly{
(Appendix~\ref{app:extra-figs})} show the distributions; Sections
\ref{sec:exp-regret} and~\ref{sec:abl-corruption} price and explain the tail.

\takeaway{At a $1.97\times$ candidate
speedup, cached rollouts preserve within-prompt candidate ranking
(median $\rhosp = 0.905$); the corrupted minority concentrates where
candidates are near-tied.}

\subsection{Search gain versus wall-clock}
\label{sec:exp-main}

\newcommand{\MainTableCaption}{\textbf{Commit retains $94.7\%$ of
best-of-$8$ gain at $63.3\%$ of its cost.} Strategies over all seed subsets
of the $50\times8$ grid at $\tauc=0.10$. Columns report delivered gain,
wall-clock, capture, and relative cost; shading marks the default.}
\newcommand{\MainTableBody}{%
  \centering
  \ifreport\footnotesize\else\tabsize\fi
  \setlength{\tabcolsep}{2pt}
  \begin{tabular}{clcccc}
  \toprule
  $N$ & strategy & gain $\uparrow$ & cost (s) $\downarrow$ &
  capture $\uparrow$ & rel.\ cost $\downarrow$ \\
  \midrule
  -- & single & $+0.000$ & $68\ci{67,70}$ & -- & -- \\
  \midrule
  2 & full & $+0.304\ci{.24,.38}$ & $137\ci{134,140}$ &
  $100\%$ & $100\%$ \\
  2 & keep & $+0.286\ci{.22,.36}$ & $69\ci{68,71}$ &
  $94.1\%\ci{84.4,103.3}{}^\dagger$ & $50.8\%$ \\
  2 & commit & $+0.274\ci{.21,.34}$ &
  $138\ci{135,141}$ & $89.9\%\ci{85.0,93.5}$ & $100.8\%$ \\
  \midrule
  4 & full & $+0.514\ci{.39,.65}$ & $273\ci{268,279}$ &
  $100\%$ & $100\%$ \\
  4 & keep & $+0.496\ci{.37,.63}$ &
  $139\ci{136,142}$ & $96.5\%\ci{90.9,101.9}{}^\dagger$ & $50.8\%$ \\
  4 & commit & $+0.461\ci{.34,.59}$ &
  $207\ci{203,212}$ & $89.6\%\ci{84.0,93.8}$ & $75.8\%$ \\
  \midrule
  8 & full & $+0.752\ci{.53,1.00}$ & $547\ci{537,557}$ &
  $100\%$ & $100\%$ \\
  8 & keep & $+0.749\ci{.53,1.00}$ &
  $278\ci{272,284}$ & $99.6\%\ci{95.4,104.3}{}^\dagger$ & $50.8\%$ \\
  \rowcolor{blue!8}
  8 & \textbf{commit} & $\bm{+0.712}\ci{.48,.97}$ &
  $\bm{346}\ci{340,354}$ & $\bm{94.7\%}\ci{90.4,97.6}$ & $\bm{63.3\%}$ \\
  \bottomrule
  \end{tabular}}

\ifreport
\newcommand{\HeroCaption}{\textbf{Above break-even, every caching engine beats
full-compute best-of-$N$.} Delivered reward vs.\ end-to-end budget on the
$50$-prompt gate grid. Markers denote $N=2,4,8$; step truncation is the
matched-cost alternative.}
\newcommand{\HeroGraphic}{%
  \includegraphics[width=\linewidth]{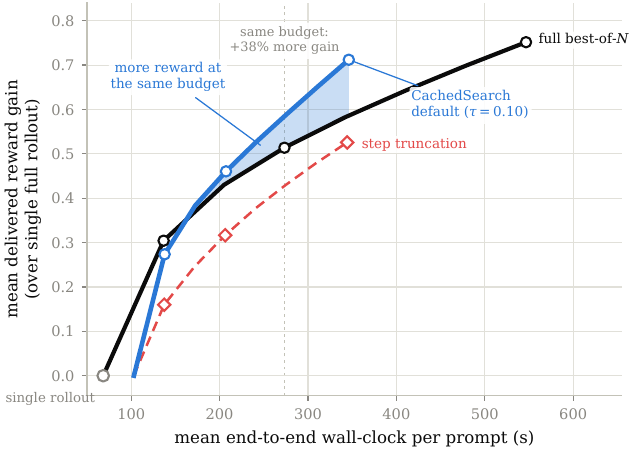}}

% The report places this panel beside Table~\ref{tab:main}. The conference
% keeps a wide standalone copy in the appendix.
\ifreport\else
\begin{figure}[!tb]
  \centering
  \includegraphics[width=0.84\linewidth]{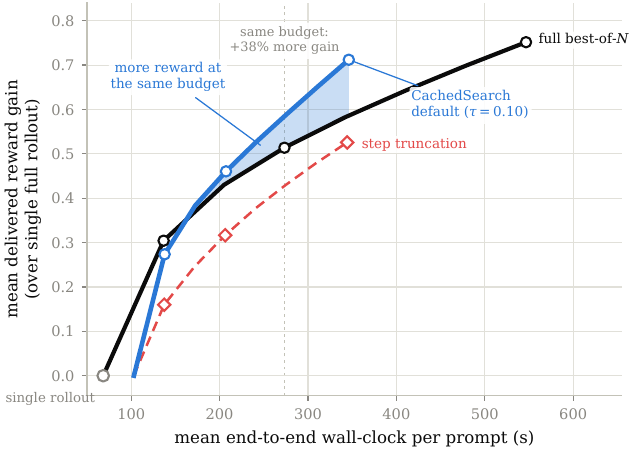}
  \caption{\HeroCaption}
  \label{fig:hero}
\end{figure}
\fi

\begin{figure}[!tb]
  \centering
  \begin{minipage}[t]{0.55\linewidth}
    \vspace{0pt}
    \captionof{table}{\MainTableCaption}
    \label{tab:main}
    \MainTableBody
  \end{minipage}\hfill
  \begin{minipage}[t]{0.43\linewidth}
    \vspace{0pt}
    \centering
    \HeroGraphic
    \captionof{figure}{\HeroCaption}
    \label{fig:hero}
  \end{minipage}
\end{figure}
\else
% Balanced pair: strategy table (left) and the rank-correlation histograms
% (right). Both are set to the same vertical extent by \vspace{0pt} tops.
\begin{figure}[!tb]
  \centering
  \begin{minipage}[t]{0.485\linewidth}
    \vspace{0pt}
    \centering
    \captionof{table}{\MainTableCaption}
    \label{tab:main}
    \resizebox{\linewidth}{!}{\MainTableBody}
  \end{minipage}\hfill
  \begin{minipage}[t]{0.485\linewidth}
    \vspace{0pt}
    \centering
    \includegraphics[width=\linewidth]{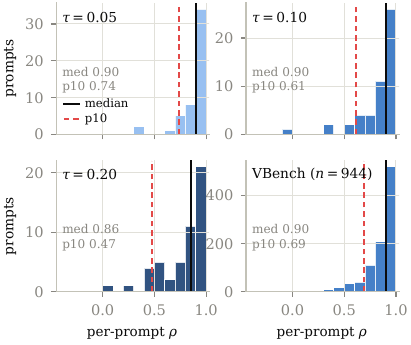}
    \captionof{figure}{\RhoHistCap}
    \label{fig:rho-hist}
  \end{minipage}
\end{figure}
\fi

Gain, capture, and cost follow Section~\ref{sec:exp-setup}; measured
latencies are $C_f=68.3$\,s and $C_c=34.7$\,s. Keep is cached-scored and
therefore nominal (Section~\ref{sec:abl-temporal}). A dagger marks ratios
above $100\%$, an unbounded ratio-of-means artifact rather than
evidence of gain beyond best-of-$N$.
Table~\ref{tab:main} gives the central result. Figure~\ref{fig:hero}
\confonly{} compares exploration engines;
Figure~\ref{fig:pareto}\reportonly{ (Appendix~\ref{app:extra-figs})} gives the
strategy-only view. Above the $N^\ast\approx2$ break-even, every cached curve
beats full-compute best-of-$N$.
At $N=8$, commit retains \textbf{$94.7\%$} of best-of-$8$ gain at
\textbf{$63.3\%$} of its cost. At ${\approx}300$\,s, it searches $8$ instead
of $4$ candidates and gains \textbf{$38\%$} more reward. Width extensions
reach $95.7\%$ capture at $N=16$ and $95.2\%$ at $N=32$ while cost falls to
$57.1\%$ and $53.9\%$ (full analysis: Appendix~\ref{app:scaling-detail}).
Keep roughly halves cost with near-perfect nominal capture, but learned
verifiers miss temporal artifacts, so commit remains the default
(Section~\ref{sec:abl-temporal}).

\confonly{
\begin{figure}[H]
  \centering
  \begin{minipage}[t]{0.485\linewidth}
    \vspace{0pt}
    \centering
    \includegraphics[width=\linewidth]{figs/f16_hero.pdf}
    \captionof{figure}{\textbf{Caching engines win beyond break-even.}
    Delivered reward vs.\ end-to-end budget on the gate grid.}
    \label{fig:hero}
  \end{minipage}\hfill
  \begin{minipage}[t]{0.485\linewidth}
    \vspace{0pt}
    \centering
    \includegraphics[width=\linewidth]{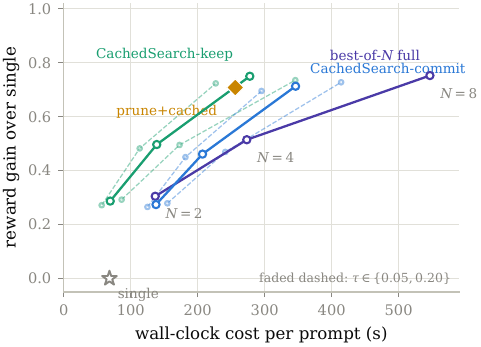}
    \captionof{figure}{\textbf{\cachedsearch{} dominates beyond
    $N^\ast\approx2$.} Reward gain vs.\ wall-clock for $N=2,4,8$;
    gold adds pruning.}
    \label{fig:pareto}
  \end{minipage}
\end{figure}
}

% fig:methods (f11) RETIRED (design-overhaul Lane B): its unique prune+cached
% point is folded into Figure~\ref{fig:pareto}; the bars duplicated
% Table~\ref{tab:main}. x_fig_methods.tex is now empty (see its header).

\takeaway{Explore cached, commit full: at
$N{=}8$ this keeps $94.7\%$ of the best-of-$N$ gain at
$63\%$ of its cost, and at a matched ${\approx}300$\,s budget cached
exploration searches twice as wide for $38\%$ more gain.}

\subsection{Regret under ranking errors}
\label{sec:exp-regret}

% The regret-CDF + corruption two-panel figure lives in
% sections/x_fig_regret_corruption.tex - inlined here in the report edition,
% hosted in the appendix (app:extra-figs) in the conference edition.
\ifreport
% SHARED FLOAT (both editions; single source). The regret-CDF (f6) +
% corruption-scatter (f5) two-panel figure. Report edition: inlined in
% sections/experiments.tex (Section exp-regret). Conference edition: hosted in the
% appendix (app:extra-figs); main-text references to fig:regret-cdf and
% fig:corruption resolve there.
\begin{figure}[t]
  \centering
  \begin{minipage}[t]{0.48\linewidth}
    \centering
    \includegraphics[width=\linewidth]{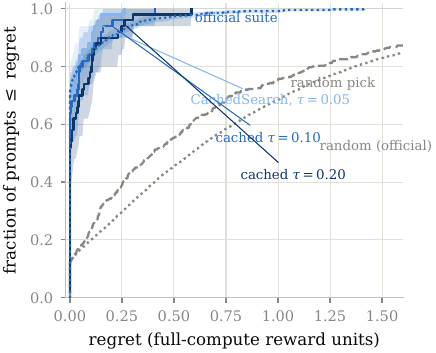}
    \caption{\textbf{Regret stays far below random picking at every
    caching level.} Per-prompt regret CDFs at $N=8$ against pooled random
    baselines. Dotted curves show the \VBenchMetric{} suite; zero regret occurs on
    $64\%$ of gate prompts and $72\%$ official.}
    \label{fig:regret-cdf}
  \end{minipage}\hfill
  \begin{minipage}[t]{0.48\linewidth}
    \centering
    \includegraphics[width=\linewidth]{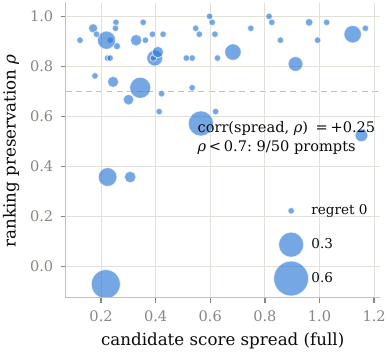}
    \caption{\textbf{Ranking corruption concentrates where candidate
    spread is low and wrong picks are cheap.} Per-prompt $\rhosp$ vs.\
    score spread at $\tauc=0.10$ on $50$ gate prompts; marker area is
    regret. The \VBenchMetric{}-suite correlation is
    $+0.31\ci{.25,.37}$.}
    \label{fig:corruption}
  \end{minipage}
\end{figure}

\fi

On the VBench suite, mean regret is \textbf{$0.056$} against a $0.79$
random-pick baseline, median regret is \textbf{$0$}, $72\%$ of prompts incur
zero regret, and mean per-prompt capture is $90.2\%$. The gate pilot gives
$0.040$ regret against $0.657$ random, $64\%$ zero regret, $93.9\%$
ratio-of-means capture, and $90.1\%$ mean per-prompt capture. Figures
\ref{fig:regret-cdf} and~\ref{fig:corruption} show that most errors swap
near-equivalent candidates; Section~\ref{sec:abl-corruption} tests the
mechanism at scale.

\subsection{VBench suite replication}
\label{sec:exp-vbench}

We repeat the core measurement on the VBench suite with complete paired
coverage across $8$ seeds at $\tauc = 0.10$. At this larger scale, median
$\rhosp$ remains \textbf{$0.905$}, mean $\rhosp$ is $0.859$, and p10 is
$0.690$. Top-1 agreement is
\textbf{$72\%$}. Mean regret is $0.056$ against a $0.79$ random baseline,
which gives \textbf{$90.2\%$} mean per-prompt capture. Candidate speedup is
$1.95\times$ ($67 \to 35$\,s). Only $11\%$ of prompts fall below
$\rhosp=0.7$, and $72\%$ incur zero regret. Figure~\ref{fig:rho-hist} shows
the ranking distribution. Relative to Section~\ref{sec:exp-ranking}, the
larger sample provides greater precision around the same conclusion.\reportonly{
Its regret distribution (Section~\ref{sec:exp-regret}) appears in
Figure~\ref{fig:regret-cdf}.} The smaller corruption tail also agrees with
the mechanism in Section~\ref{sec:abl-corruption}. Full intervals are in the
appendix. The VBench point estimates are favorable, but none differs significantly
from the gate grid at its $n{=}50$ resolution. Every gate estimate remains
statistically consistent with the VBench suite. This result adds
precision, not evidence of improvement over the gate grid. We also repeat
the measurement on VBench-2.0 \citep{zheng2025vbench2}, a harder suite that
tests intrinsic faithfulness through compositional interactions, physics,
commonsense, and camera control. This replication matters because it changes
both prompt structure and difficulty while preserving the paired-seed
protocol. Median $\rhosp=0.881$, top-1 is $69\%$, capture is $89.3\%$, and
speedup is $1.98\times$, all statistically consistent with the VBench suite
(Table~\ref{tab:vbench2}\reportonly{, Appendix~\ref{app:more-abl}}). The
agreement shows that ranking preservation is not specific to the original
suite's per-dimension prompt lists.

% The report shows the qualitative search result immediately after the main
% strategy result; the conference points to its appendix copy here.
\confonly{
Figure~\ref{fig:qual-search} (Appendix~\ref{app:qualitative}) shows
delivered videos for three prompts.}

\confonly{
\begin{table}[H]
\caption{\textbf{The result holds on a third, harder suite.}
Ranking and efficiency at $\tauc=0.10$ with $8$ seeds.}
\label{tab:vbench2}
\centering
\tabsize
\setlength{\tabcolsep}{3.5pt}
\resizebox{\linewidth}{!}{%
\begin{tabular}{lcccccccc}
\toprule
suite & $n$ & med $\rhosp$ $\uparrow$ & mean $\rhosp$ $\uparrow$ &
p10 $\uparrow$ & top-1 $\uparrow$ & zero-regret $\uparrow$ &
capture $\uparrow$ & speedup $\uparrow$ \\
\midrule
gate grid & $50$ & $0.905\ci{.83,.93}$ & $0.820\ci{.76,.87}$ &
$0.61\ci{.36,.71}$ & $64\%\ci{50,76}$ & $64\%\ci{50,76}$ &
$90.1\%\ci{82.6,95.9}$ & $1.97\ci{1.96,1.98}\times$ \\
official VBench & $944$ & $0.905$ & $0.859\ci{.850,.869}$ &
$0.69\ci{.64,.71}$ & $72\%\ci{69,75}$ & $72\%\ci{69,75}$ &
$90.2\%\ci{88.5,91.8}$ & $1.95\ci{1.95,1.95}\times$ \\
VBench-2.0 & $1{,}013$ & $0.881$ & $0.848\ci{.840,.857}$ &
$0.69\ci{.64,.69}$ & $69\%\ci{66,72}$ & $69\%\ci{66,72}$ &
$89.3\%\ci{87.9,90.7}$ & $1.98\ci{1.98,1.98}\times$ \\
\bottomrule
\end{tabular}}
\end{table}
}

\paragraph{Delivered-video metrics.}
On a held-out VBench subset, \cachedsearch{}-commit statistically matches
full best-of-$8$ across the standard metric suite at $64\%$ of its
wall-clock.
The full delivered-video audit is in Table~\ref{tab:multimetric} and
Appendix~\ref{app:more-abl}.

% ======================================================================
% Analysis & ablations (owner: paper-experiments agent).
% Measured sources: PROGRESS.md, pillars/B PLAN.md, README.md (batch bench),
% results/b1_gate_{v0,tau005,tau020}/scores_shard*.jsonl (tau sweep, corruption,
% E5 adaptive-tau simulation; re-derived by code/paper_figs/make_figs.py and
% code/experiments/b1_adaptive_tau.py), results/b1_temporal/records_shard*.jsonl
% + lpips_fixed.jsonl (E1 keep-vs-commit), results/b1_verifiers_vbench/ (E3,
% re-derived by code/paper_figs/e3_analysis.py, run 2026-07-07),
% results/a1_{teacache,pab,cfgcache,easycache-ours,none}/ (A-lite baseline
% grids, tab:baselines; re-derived by code/paper_figs/a1_analysis.py).
% ======================================================================
\section{Analysis and Ablations}
\label{sec:ablations}  % section agents: KEEP this label - other sections \ref it

\subsection{Caching aggressiveness}
\label{sec:abl-tau}

\confonly{
% ======================================================================
% tab:tau - the caching-threshold sweep.
% The horizontal table body is paired with tab:cheapexplore in
% x_tab_tau_cheapexplore.tex.
% ======================================================================
\newcommand{\TauTableCaption}{\textbf{Raising $\tauc$ trades a graceful
capture loss for more speedup.} Caching-threshold sweep ($50$ prompts
$\times\,8$ seeds, with shared full-compute references). Regret and capture
use $N{=}8$; capture follows Eq.~\ref{eq:capture}. Shading marks the default.}
\newcommand{\TauTableBody}{%
\begin{tabular}{ccccccc}
\toprule
$\tauc$ & speedup $\uparrow$ & median $\rhosp$ $\uparrow$ & p10 $\rhosp$ $\uparrow$ & top-1 $\uparrow$ & regret $\downarrow$ & capture $\uparrow$ \\
\midrule
$0.05$ & $1.58\ci{1.55,1.61}\times$ & $0.905\ci{.90,.93}$ & $0.738\ci{.63,.81}$ & $70\%\ci{56,82}$ & $0.025\ci{.011,.041}$ & $93.6\%\ci{88.6,97.5}$ \\
\rowcolor{blue!8}
$0.10$ & $1.97\ci{1.96,1.98}\times$ & $0.905\ci{.83,.93}$ & $0.614\ci{.35,.72}$ & $64\%\ci{50,76}$ & $0.040\ci{.019,.065}$ & $90.1\%\ci{82.6,95.9}$ \\
$0.20$ & $2.41\ci{2.37,2.45}\times$ & $0.857\ci{.80,.91}$ & $0.474\ci{.40,.60}$ & $52\%\ci{38,66}$ & $0.057\ci{.030,.090}$ & $88.3\%\ci{82.3,93.5}$ \\
\bottomrule
\end{tabular}
}

\begin{table}[H]
  \caption{\textbf{Raising $\tauc$ trades capture for speed.}
  Threshold sweep on $50$ prompts and $8$ seeds; shading marks the default.}
  \label{tab:tau}
  \centering
  \tabsize
  \setlength{\tabcolsep}{4pt}
  \newsavebox{\ConfTauTableBox}
  \sbox{\ConfTauTableBox}{\TauTableBody}
  \ifdim\wd\ConfTauTableBox>\linewidth
    \resizebox{\linewidth}{!}{\usebox{\ConfTauTableBox}}
  \else
    \usebox{\ConfTauTableBox}
  \fi
\end{table}
}

Raising the threshold buys speed at a gradual cost to ranking quality.
Candidate speedup grows from $1.58\times$ to $2.41\times$, while p10
$\rhosp$ falls from $0.74$ to $0.47$; Table~\ref{tab:tau}\confonly{ and
Figure~\ref{fig:baselines}}\reportonly{ and Figure~\ref{fig:baselines}} give
the remaining metrics. We use $\tauc=0.10$ by default, nearly doubling
exploration speed while retaining $90\%$ capture.

\subsection{Ranking preservation across published caching methods}
\label{sec:abl-baselines}

% tab:baselines lives in sections/x_tab_baselines.tex - report edition:
% inlined here; conference edition: hosted in app:baseline-details (its
% points + fitted frontier appear in the main-text master figure
% fig:scaling-master).
\ifreport
% ======================================================================
% tab:baselines - published caching methods as the exploration engine.
% Report edition: inlined in Section abl-baselines (sections/ablations.tex).
% Conference edition: hosted in the appendix (app:baseline-details); the
% main text keeps the section prose, the takeaway, and the frontier claim
% (whose points also appear in the master figure fig:scaling-master).
% Single-source rule: NOTES-editions-agent.md.
% ======================================================================
\ifreport
\begin{table}[!tb]
\else
\begin{table}[H]
\fi
\caption{\textbf{Published caches serve as interchangeable exploration
engines.} Four methods use the same $50$-prompt, $8$-seed gate protocol
and full references. Rows report candidate fidelity, speedup, and
best-of-$8$ commit cost; \emph{reuses} names the cheaply recomputed
component. Shading marks the default.}
\label{tab:baselines}
\centering
\ifreport\tabsize\else\small\fi
\setlength{\tabcolsep}{3pt}
\resizebox{\linewidth}{!}{%
\begin{tabular}{llccccccc}
\toprule
exploration engine & reuses & speedup $\uparrow$ & med $\rhosp$ $\uparrow$ & p10 $\uparrow$ & top-1 $\uparrow$ & regret $\downarrow$ & capture $\uparrow$ & cost (s) $\downarrow$ \\
\midrule
no caching (= full best-of-$N$) & - & $0.99\ci{0.96,1.02}\times$ & $1.000$ & $1.000$ & $100\%$ & $0.000$ & $100\%$ & $621$ \\
\midrule
\cachedsearch{} + PAB \citep{zhao2024pab} & attn.\ outputs & $1.29\ci{1.26,1.32}\times$ & $\underline{0.976}\ci{.97,1.00}$ & $\bm{0.950}\ci{.92,.96}$ & $\underline{90\%}\ci{81,98}$ & $\bm{0.004}\ci{.000,.010}$ & $\underline{99.5\%}\ci{98.9,100.0}$ & $493$ \\
\cachedsearch{} + CFG-Cache \citep{lv2024fastercache} & uncond.\ branch & $1.38\ci{1.35,1.41}\times$ & $\bm{1.000}\ci{.97,1.00}$ & $\underline{0.929}\ci{.91,.96}$ & $\bm{92\%}\ci{84,98}$ & $\bm{0.004}\ci{.000,.010}$ & $\bm{99.6\%}\ci{99.0,100.0}$ & $465$ \\
\cachedsearch{} + TeaCache \citep{liu2024teacache} & full model & $\underline{1.84}\ci{1.79,1.90}\times$ & $0.929\ci{.90,.95}$ & $0.783\ci{.61,.83}$ & $68\%\ci{54,80}$ & $\underline{0.029}\ci{.014,.048}$ & $93.2\%\ci{88.1,97.0}$ & $\underline{365}$ \\
\midrule
\rowcolor{blue!8}
\cachedsearch{} + EasyCache \citep{zhou2025easycache} (default) & full model & $\bm{1.96}\ci{1.90,2.02}\times$ & $0.905\ci{.83,.93}$ & $0.614\ci{.35,.72}$ & $64\%\ci{50,78}$ & $0.039\ci{.019,.065}$ & $90.1\%\ci{82.4,95.9}$ & $\bm{346}$ \\
\bottomrule
\end{tabular}}
\end{table}

\fi

Table~\ref{tab:baselines} compares published caching rules as alternative
exploration \emph{engines} inside \cachedsearch{}, not as rival end-to-end
methods. Search, selection, and delivery remain fixed.\reportonly{ We reproduce
one official implementation from each major reuse family.}
\ifreport
% SHARED BLOCK (both editions; single source). Report edition: continues the
% opening paragraph of Section abl-baselines in sections/ablations.tex. Conference
% edition: hosted in the appendix (app:baseline-details).
\textbf{TeaCache} \citep{liu2024teacache} runs the official
Wan2.1 \citep{wan2025wan} rule with the published 1.3B polynomial coefficients and threshold
$0.08$. Below threshold, it skips all transformer blocks and reuses the
cached token residual ($24/50$ steps per CFG branch). \textbf{PAB}
\citep{zhao2024pab} broadcasts attention \emph{outputs} across steps
(official VideoSys rule). Wan has no official configuration, so its unified
3D self-attention uses the CogVideoX \citep{yang2024cogvideox} ``spatial''
range $2$; cross-attention
uses OpenSora's \citep{zheng2024opensora} range $6$. FasterCache's
\textbf{CFG-Cache}
\citep{lv2024fastercache} is a verbatim port with the official
frequency-compensation constants.
It rebuilds $26/50$ unconditional forwards from conditional outputs with
boosted low/high FFT deltas. This isolates CFG reuse from attention reuse
(PAB) and full-model reuse (TeaCache and our EasyCache-style rule
\citep{zhou2025easycache}). A \emph{no-caching} control regenerates the
full-compute variant in the same harness. It reproduces all overlapping
scores bit-exactly, with every common record identical ($\rhosp = 1$,
speedup $1.00\times$). This validates seed determinism and the protocol.

\else

% Conference glue: table + reproduction details hosted in
% app:baseline-details.
A \emph{no-caching} control reproduces the reference scores bit-exactly
through the same harness, validating the protocol (Table~\ref{tab:baselines}
and reproduction details: Appendix~\ref{app:baseline-details}).
\fi

Every engine preserves candidate ranking.\reportonly{ PAB \citep{zhao2024pab}
and CFG-Cache \citep{lv2024fastercache} keep
median $\rhosp$ at $0.98$--$1.00$ and capture above $99\%$ at
$1.29\times$/$1.38\times$. The aggressive full-model engines reach
${\sim}1.8$--$2.0\times$ at $90$--$93\%$ capture.} Together with the $\tauc$
sweep, they form one capture-vs-speedup frontier
(Figure~\ref{fig:baselines}\confonly{; the same
fit appears in Figure~\ref{fig:scaling-master}}). Speedup predicts fidelity
to within a few points despite different reuse patterns. Thus
\cachedsearch{} is cache-agnostic. Our default takes the aggressive end for
maximum speedup; TeaCache \citep{liu2024teacache} is the balanced middle
choice when fidelity matters
more. Architecture-specific calibration appears in
Section~\ref{sec:abl-cog}. In delivered-value terms
(Figure~\ref{fig:engines}), TeaCache
delivers $96\%$ at $67\%$ cost, CFG-Cache and PAB $99.5\%$ at $85$--$90\%$,
ours $95\%$ at $63\%$.

\ifreport
% SHARED ROW: candidate-level and delivered-value views of caching engines.
% SHARED BLOCK (both editions; single source): the caching-method frontier
% figure (f11b, fig:baselines). Report edition: inlined in
% sections/ablations.tex (Section abl-baselines). Conference edition: hosted in the
% appendix (app:extra-figs) - the main-text master figure
% (fig:scaling-master, Section abl-scalingfigs) contains these operating points and
% the same fitted frontier, so the dedicated figure is appendix material
% under the page limit.
\newcommand{\BaselinesCaption}{\textbf{Caching engines share one
capture-vs-speedup frontier.} Filled markers show PAB \citep{zhao2024pab},
CFG-Cache \citep{lv2024fastercache}, TeaCache \citep{liu2024teacache}, and
EasyCache \citep{zhou2025easycache}; blue points sweep $\tauc$ and $\times$
marks the independent default rerun. The dashed fit covers six caching
points ($R^2=0.92$), each with $n=50$ prompts.}
\newcommand{\BaselinesGraphic}{%
  \includegraphics[width=\linewidth]{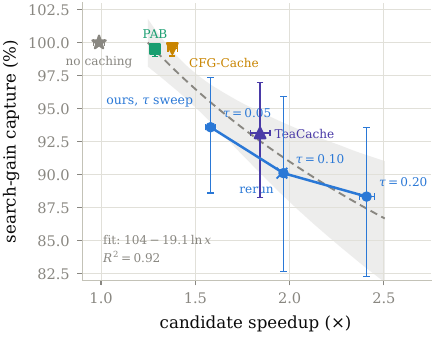}}

% SHARED BLOCK (both editions; single source): existing caching methods as
% exploration engines on end-to-end delivered-gain vs wall-clock axes (f15,
% fig:engines). Complements fig:baselines: that figure positions each engine's
% capture against its candidate speedup; this one prices the FULL pipeline
% (8 cached rollouts + the full-compute recommit) and shows what each engine
% delivers for what it costs. Report edition: inlined in
% sections/ablations.tex (Section abl-baselines, after fig:baselines).
% Conference edition: hosted in the appendix (app:extra-figs). Generated by
% code/paper_figs/make_figs2.py (f15_engines); all points are measured
% selections on the 8-seed gate grid.
\newcommand{\EnginesCaption}{\textbf{Every caching engine beats full-compute
search on delivered cost.} Best-of-$8$ gain capture vs.\ measured end-to-end
wall-clock at $N=8$, including recommit. Points compare
PAB \citep{zhao2024pab}, CFG-Cache \citep{lv2024fastercache},
TeaCache \citep{liu2024teacache}, EasyCache \citep{zhou2025easycache},
and the no-caching control on the $50$-prompt grid.}
\newcommand{\EnginesGraphic}{%
  \includegraphics[width=\linewidth]{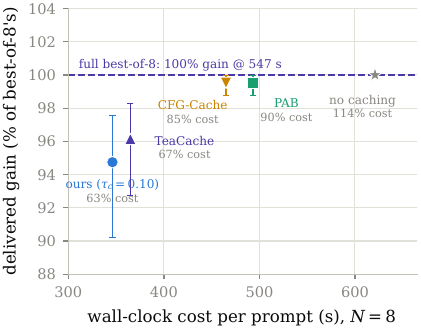}}

\ifreport
\begin{figure}[!tb]
\else
\begin{figure}[H]
\fi
  \centering
  \begin{minipage}[t]{0.485\linewidth}
    \vspace{0pt}
    \centering
    \BaselinesGraphic
    \captionof{figure}{\ifreport\BaselinesCaption\else
    \textbf{Caching engines share one capture-speed frontier.}
    Gain capture vs.\ candidate speedup for PAB \citep{zhao2024pab},
    CFG-Cache \citep{lv2024fastercache}, TeaCache \citep{liu2024teacache},
    EasyCache \citep{zhou2025easycache}, and the threshold sweep.\fi}
    \label{fig:baselines}
  \end{minipage}\hfill
  \begin{minipage}[t]{0.485\linewidth}
    \vspace{0pt}
    \centering
    \EnginesGraphic
    \captionof{figure}{\ifreport\EnginesCaption\else
    \textbf{Every caching engine beats full-compute search.}
    Best-of-$8$ gain capture vs.\ end-to-end cost, including recommit.\fi}
    \label{fig:engines}
  \end{minipage}
\end{figure}

\else

\fi

\subsection{Caching versus truncation}
\label{sec:abl-cheapexplore}

The obvious alternative to caching is to spend fewer denoising steps per
candidate. We measure that alternative under the identical gate protocol,
using the existing full-compute $T{=}25$ arm as cheap exploration and the
$T{=}50$ full-compute rollouts as references. Both arms deliver
${\sim}2\times$ speedup\confonly{ (Table~\ref{tab:cheapexplore})}.\reportonly{ Table~\ref{tab:cheapexplore}
reports the matched comparison.} Caching retains $90.1\%$ of search gain,
while truncation retains $72.6\%$; truncation's regret is nearly six times
higher and its worst-decile ranking is uncorrelated with the reference
ordering.

\confonly{
% ======================================================================
% tab:cheapexplore - caching versus step truncation at matched speedup.
% The horizontal table body is paired with tab:tau in
% x_tab_tau_cheapexplore.tex.
% ======================================================================
\newcommand{\CheapExploreCaption}{\textbf{Ranking preservation under two
cheap-exploration strategies.} Caching and step truncation at matched
${\sim}2\times$ speedup, using $T=50$ references on $50$ prompts and $8$
seeds. Truncation incurs ${\sim}6\times$ more regret; shading marks the
default.}
\newcommand{\CheapExploreBody}{%
\begin{tabular}{lcccccc}
\toprule
strategy & speedup $\uparrow$ & med $\rhosp$ $\uparrow$ &
p10 $\rhosp$ $\uparrow$ & top-1 $\uparrow$ & regret $\downarrow$ &
capture $\uparrow$ \\
\midrule
\rowcolor{blue!8}
caching ($\tauc{=}0.10$) &
$1.97\ci{1.96,1.98}\times$ &
$\bm{0.905}\ci{.83,.93}$ &
$\bm{0.614}\ci{.36,.71}$ &
$\bm{64\%}\ci{50,78}$ &
$\bm{0.039}\ci{.019,.064}$ &
$\bm{90.1\%}\ci{82.5,95.9}$ \\
step-truncation ($T{=}25$) &
$1.98\ci{1.95,2.02}\times$ &
$0.512\ci{.44,.69}$ &
$0.024\ci{-.04,.31}$ &
$44\%\ci{30,58}$ &
$0.226\ci{.088,.398}$ &
$72.6\%\ci{59.2,84.3}$ \\
\bottomrule
\end{tabular}
}

\begin{table}[H]
  \caption{\textbf{Ranking preservation under two cheap-exploration strategies.}
  Caching and truncation at matched ${\sim}2\times$ speedup on
  $50$ prompts and $8$ seeds.}
  \label{tab:cheapexplore}
  \centering
  \tabsize
  \setlength{\tabcolsep}{4pt}
  \newsavebox{\ConfCheapTableBox}
  \sbox{\ConfCheapTableBox}{\CheapExploreBody}
  \ifdim\wd\ConfCheapTableBox>\linewidth
    \resizebox{\linewidth}{!}{\usebox{\ConfCheapTableBox}}
  \else
    \usebox{\ConfCheapTableBox}
  \fi
\end{table}
}

Search needs the cheap score to predict the \emph{full-compute sample's}
rank. A truncated rollout is an honest sample of a different,
shorter-schedule distribution, so its verifier scores rank its own outputs,
not the $50$-step outputs that commit will deliver. Caching instead perturbs
the same seed-matched trajectory. That is precisely the property search
needs.

\takeaway{At the same $2\times$ saving, caching keeps $90\%$ of search's
value while truncation keeps $73\%$. What matters is preserving the
delivered sample's trajectory, not the budget.}

\ifreport
% SHARED ROW: the two compact horizontal ablation tables.
% ======================================================================
% tab:tau - the caching-threshold sweep.
% The horizontal table body is paired with tab:cheapexplore in
% x_tab_tau_cheapexplore.tex.
% ======================================================================
\newcommand{\TauTableCaption}{\textbf{Raising $\tauc$ trades a graceful
capture loss for more speedup.} Caching-threshold sweep ($50$ prompts
$\times\,8$ seeds, with shared full-compute references). Regret and capture
use $N{=}8$; capture follows Eq.~\ref{eq:capture}. Shading marks the default.}
\newcommand{\TauTableBody}{%
\begin{tabular}{ccccccc}
\toprule
$\tauc$ & speedup $\uparrow$ & median $\rhosp$ $\uparrow$ & p10 $\rhosp$ $\uparrow$ & top-1 $\uparrow$ & regret $\downarrow$ & capture $\uparrow$ \\
\midrule
$0.05$ & $1.58\ci{1.55,1.61}\times$ & $0.905\ci{.90,.93}$ & $0.738\ci{.63,.81}$ & $70\%\ci{56,82}$ & $0.025\ci{.011,.041}$ & $93.6\%\ci{88.6,97.5}$ \\
\rowcolor{blue!8}
$0.10$ & $1.97\ci{1.96,1.98}\times$ & $0.905\ci{.83,.93}$ & $0.614\ci{.35,.72}$ & $64\%\ci{50,76}$ & $0.040\ci{.019,.065}$ & $90.1\%\ci{82.6,95.9}$ \\
$0.20$ & $2.41\ci{2.37,2.45}\times$ & $0.857\ci{.80,.91}$ & $0.474\ci{.40,.60}$ & $52\%\ci{38,66}$ & $0.057\ci{.030,.090}$ & $88.3\%\ci{82.3,93.5}$ \\
\bottomrule
\end{tabular}
}

% ======================================================================
% tab:cheapexplore - caching versus step truncation at matched speedup.
% The horizontal table body is paired with tab:tau in
% x_tab_tau_cheapexplore.tex.
% ======================================================================
\newcommand{\CheapExploreCaption}{\textbf{Ranking preservation under two
cheap-exploration strategies.} Caching and step truncation at matched
${\sim}2\times$ speedup, using $T=50$ references on $50$ prompts and $8$
seeds. Truncation incurs ${\sim}6\times$ more regret; shading marks the
default.}
\newcommand{\CheapExploreBody}{%
\begin{tabular}{lcccccc}
\toprule
strategy & speedup $\uparrow$ & med $\rhosp$ $\uparrow$ &
p10 $\rhosp$ $\uparrow$ & top-1 $\uparrow$ & regret $\downarrow$ &
capture $\uparrow$ \\
\midrule
\rowcolor{blue!8}
caching ($\tauc{=}0.10$) &
$1.97\ci{1.96,1.98}\times$ &
$\bm{0.905}\ci{.83,.93}$ &
$\bm{0.614}\ci{.36,.71}$ &
$\bm{64\%}\ci{50,78}$ &
$\bm{0.039}\ci{.019,.064}$ &
$\bm{90.1\%}\ci{82.5,95.9}$ \\
step-truncation ($T{=}25$) &
$1.98\ci{1.95,2.02}\times$ &
$0.512\ci{.44,.69}$ &
$0.024\ci{-.04,.31}$ &
$44\%\ci{30,58}$ &
$0.226\ci{.088,.398}$ &
$72.6\%\ci{59.2,84.3}$ \\
\bottomrule
\end{tabular}
}

\begin{figure}[!tb]
  \centering
  \begin{minipage}[t]{0.49\linewidth}
    \vspace{0pt}
    \captionof{table}{\TauTableCaption}
    \label{tab:tau}
    \centering
    \tabsize
    \setlength{\tabcolsep}{4pt}
    \resizebox{\linewidth}{!}{\TauTableBody}
  \end{minipage}\hfill
  \begin{minipage}[t]{0.49\linewidth}
    \vspace{0pt}
    \captionof{table}{\CheapExploreCaption}
    \label{tab:cheapexplore}
    \centering
    \tabsize
    \setlength{\tabcolsep}{4pt}
    \resizebox{\linewidth}{!}{\CheapExploreBody}
  \end{minipage}
\end{figure}

\fi

\subsection{Self-limiting ranking corruption}
\label{sec:abl-corruption}

On the \VBenchMetric{} suite ($\tauc = 0.10$), per-prompt $\rhosp$ and full-compute
score spread have corr(spread, $\rhosp$) $= +0.31$
($p = 2{\times}10^{-22}$). Ranking fails
mainly when candidates are already near-tied.\reportonly{ Prompts with
$\rhosp < 0.7$ have lower spread than the rest (median $0.40$ vs.\ $0.56$;
Mann--Whitney
$p = 3{\times}10^{-10}$), and $79\%$ lie in the bottom half. Cached score
noise therefore changes ordering most when candidate separation, and hence
selection value, is small (Appendix~\ref{app:theory},
Eq.~\ref{eq:weightedcapture}). Figure~\ref{fig:corruption} shows the same
pattern on the gate grid.}

\ifreport
% SHARED BLOCK (both editions; single source). Report edition: second
% paragraph of Section abl-corruption in sections/ablations.tex. Conference edition:
% hosted in the appendix (app:corruption-detail).
The \VBenchMetric{} suite bears out both halves of this mechanism.
Selection is most reliable exactly where the most is at stake: across
spread quartiles, gain capture rises monotonically from $81\%$ (lowest
quartile, median $\rhosp = 0.857$) to $96\%$ (highest quartile, median
$\rhosp = 0.952$) while the attainable per-prompt gain more than triples
($0.36 \to 1.27$ reward units), so realized regret does not grow with the
stakes (corr(spread, regret) $= -0.11$). The
corrupted tail is not free (the $11\%$ of prompts with $\rhosp < 0.7$
average $0.161$ regret versus $0.043$ elsewhere, i.e.\ $32\%$ of the
total regret), but it is bounded: half of these prompts ($51\%$) still
pick the exact best candidate, and even within this
subset cached selection retains $68\%$ of the attainable gain. Caching
therefore fails most often where failure costs least; this is why mean
regret ($0.056$ on \VBenchMetric{}, $0.040$ on the gate) sits an order of magnitude below
the random-pick baseline, and it matches the zero median regret of
Section~\ref{sec:exp-regret}. It also suggests spread as a cheap online signal
for a per-prompt adaptive $\tauc$: an idea we test, and reject, next.

\else
% Conference glue: full quartile analysis in app:corruption-detail; the
% report keeps this analysis inline.
The full regret-quartile analysis is in
Appendix~\ref{app:corruption-detail}.
\fi

\takeaway{Corruption is self-limiting: it
concentrates on near-tied prompts (corr(spread, $\rhosp$) $= +0.31$),
capture rises with the stakes ($81\% \to 96\%$ across spread quartiles),
and even the corrupted $11\%$ of prompts retain $68\%$ of their attainable
gain.}

\paragraph{Adaptive threshold (negative result).}
A spread-probe policy is statistically indistinguishable from and pointwise
never better than fixed $\tauc=0.10$: at matched $1.97\times$ speedup,
capture changes by $-0.4$ points. The full policy and frontier are in
Section~\ref{sec:abl-adaptive},
Appendix~\ref{app:more-abl}; we retain one global threshold.

\subsection{Keep-draft versus recommit}
\label{sec:abl-temporal}

Keep's clear cost is \emph{motion dampening}: mean optical flow falls
$-3.1\%$ at $\tauc=0.10$ and $-8.0\%$ at $\tauc=0.20$.
\ImageRewardMetric{} and \VideoScoreMetric{} both prefer the cached outputs,
so neither learned judge detects
the artifact (Section~\ref{sec:abl-verifier}). Use keep only at
$\tauc\leq0.10$ when motion is secondary; use commit for aggressive caching
or motion-critical prompts. The temporal metrics, \LPIPSMetric{}, fidelity examples,
and flow maps are in Appendices~\ref{app:temporal-detail}
and~\ref{app:qualitative}.

\paragraph{Batching control.}
Batching does not explain the saving: full-compute throughput drops from
$0.92$ to $0.90$ videos/min at batch $2$, cached batching gains only $8\%$,
and batch $8$ exceeds memory. Section~\ref{sec:abl-batch} in
Appendix~\ref{app:more-abl} gives the full control and determinism audit.

\paragraph{Verifier robustness.}
VideoScore-v1.1 \citep{he2024videoscore} bounds caching-attributable
cross-verifier loss at $-2.5$ points of gain capture, a statistically
marginal difference far below the verifiers' standing disagreement.
It also prefers motion-dampened cached outputs, so direct temporal metrics
and commit delivery remain necessary; the full ablation is
Section~\ref{sec:abl-verifier} in Appendix~\ref{app:more-abl}.

\subsection{Model generality and scale}
\label{sec:abl-cog}

We evaluate \wan{} \citep{wan2025wan}, Wan2.1-T2V-14B
\citep{wan2025wan}, Wan2.2-TI2V-5B \citep{wan2025wan22,wan2025wan},
CogVideoX-5B \citep{yang2024cogvideox}, HunyuanVideo-13B
\citep{kong2024hunyuanvideo}, and LTX-Video-2B
\citep{hacohen2024ltxvideo}.

\ifreport
\begin{table}[!tb]
\else
\begin{table}[H]
\fi
\caption{\textbf{Ranking fidelity varies by model family at fixed
$\tauc=0.10$.} Six models from four families
\citep{wan2025wan,wan2025wan22,yang2024cogvideox,kong2024hunyuanvideo,hacohen2024ltxvideo},
each with $50$ prompts and
$8$ seeds under its model-card recipe. \emph{skip} is reused denoising
steps; bold and underline mark the two best fidelity values.}
\label{tab:scale}
\centering
\tabsize
\setlength{\tabcolsep}{2pt}
\newsavebox{\ScaleTableBox}
\sbox{\ScaleTableBox}{%
\begin{tabular}{lccccccc}
\toprule
model & params & med.\ $\rhosp$ $\uparrow$ & p10 $\uparrow$ &
top-1 $\uparrow$ & capture $\uparrow$ & skip & speedup $\uparrow$ \\
\midrule
\logocell{wan13b} & 1.3B & $\bm{0.905}\ci{.83,.93}$ & $\bm{0.61}\ci{.35,.72}$ & $\bm{64\%}\ci{50,76}$ & $\bm{90.1\%}\ci{82.6,95.9}$ & $51\%$ & $1.97\ci{1.96,1.98}\times$ \\
\logocell{ltx} & 2B & $0.536\ci{.44,.60}$ & $0.14\ci{-.03,.20}$ & $50\%\ci{36,64}$ & $67.6\%\ci{53.6,80.2}$ & $64\%$ & $2.63\ci{2.62,2.65}\times$ \\
\logocell{wan22} & 5B & $\underline{0.881}\ci{.82,.91}$ & $\underline{0.60}\ci{.46,.67}$ & $46\%\ci{32,60}$ & $86.0\%\ci{78.8,92.4}$ & $55\%$ & $2.05\ci{2.03,2.06}\times$ \\
\logocell{cog} & 5B & $0.762\ci{.66,.83}$ & $0.35\ci{.26,.46}$ & $48\%\ci{34,62}$ & $75.2\%\ci{64.8,84.0}$ & $53\%$ & $2.06\ci{2.05,2.07}\times$ \\
\logocell{hunyuan} & 13B & $0.762\ci{.72,.83}$ & $0.44\ci{.23,.60}$ & $54\%\ci{40,68}$ & $79.9\%\ci{68.8,89.3}$ & $57\%$ & $2.19\ci{2.17,2.20}\times$ \\
\logocell{wan14b} & 14B & $\bm{0.905}\ci{.84,.93}$ & $0.52\ci{.46,.73}$ & $\underline{58\%}\ci{44,72}$ & $\underline{87.5\%}\ci{80.1,93.8}$ & $52\%$ & $2.05\ci{2.03,2.06}\times$ \\
\bottomrule
\end{tabular}
}
\ifdim\wd\ScaleTableBox>\linewidth
  \resizebox{\linewidth}{!}{\usebox{\ScaleTableBox}}
\else
  \usebox{\ScaleTableBox}
\fi
\end{table}

At fixed $\tauc=0.10$, Table~\ref{tab:scale} separates family from size:
Wan2.1-1.3B and 14B both reach median $\rhosp=0.905$ with
$90.1\%/87.5\%$ capture, and Wan2.2 reaches $0.881/86.0\%$. Off-family
ports sit lower regardless of size: CogVideoX and Hunyuan reach median
$0.762$ with $75.2\%/79.9\%$ capture, while over-skipped LTX is the boundary
at $0.536/67.6\%$. Architecture sets fidelity; parameter count sets the
absolute saving, from $6$\,s on LTX-2B to $175$\,s on Wan-14B.

Calibration recovers much of the family gap: CogVideoX reaches $85.9\%$
capture at $\tauc=0.05$ and $1.78\times$, while every Wan backbone reaches
at least $86\%$ at its selected threshold. CogVideoX and Hunyuan use
$\tau^*=0.05$, Wan2.2 and Wan14B use $0.10$, Wan1.3B qualifies at $0.20$,
and no measured LTX threshold reaches $85\%$ capture; the full recipe and
per-model table are in Table~\ref{tab:taucal} in
Appendix~\ref{app:scaling-detail}. The mechanism transfers across a
$10.8\times$ scale gap and four model families, but the operating point
does not. A $25$-prompt pilot is sufficient to recalibrate $\tauc$ for a
new family.

\paragraph{Schedule length and resolution.}
At fixed $\tauc$, $25/50/100$-step trajectories skip $32/51/66\%$ of steps
and reach $1.41/1.97/2.80\times$ at $95.8/90.1/82.5\%$ capture.
Resolution leaves the operating point unchanged ($720$p:
$2.05\times$, $90.8\%$); Section~\ref{sec:abl-schedule} in
Appendix~\ref{app:more-abl} gives the full ablation.

\paragraph{Application: multiplying pruned-search methods.}
Pruning-based search methods reduce how many candidates pay full price;
\cachedsearch{} reduces what each candidate costs. The savings therefore
multiply, and the measured composition confirms the prediction: the measured
$1.96\times$ (caching) and $1.59\times$ (mid-trajectory pruning) factors
produce a $3.11\times$ exploration speedup over full-compute best-of-$8$ at
$88.6\%$ capture, with median regret still zero
(Section~\ref{sec:abl-stack}, Appendix~\ref{app:more-abl}).

\subsection{Scaling behavior}
\label{sec:abl-scalingfigs}

% The master operating-points figure (f13, fig:scaling-master) appears in
% the MAIN TEXT of both editions; the full analysis around it is report
% main flow / conference appendix (app:scaling-detail + app:more-abl).
% All fit numbers: code/paper_figs/scaling_analysis.py (grep "[PROSE]").
\begin{figure}[!tb]
  \centering
  \includegraphics[width=\ifreport 0.99\else 0.80\fi\linewidth]{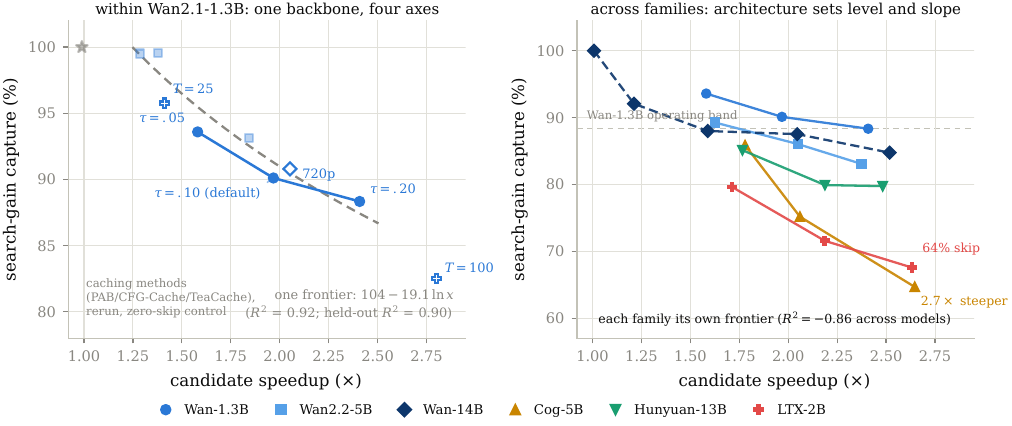}
  \caption{\textbf{One frontier fits reuse choices within a backbone, but
  architecture shifts the frontier.} Gain capture vs.\ candidate speedup;
  up-left is better. Panel A varies reuse axes on Wan2.1-1.3B
  \citep{wan2025wan}; Panel B
  sweeps $\tauc$ across six models. Each point has $44$--$50$ prompts.}
  \label{fig:scaling-master}
\end{figure}

Figure~\ref{fig:scaling-master} collects threshold, engine, schedule,
resolution, width, and model measurements. On Wan2.1-1.3B, the frontier
fitted to threshold and engine predicts held-out schedules, $720$p, and a
rerun within $2.0$ points ($R^2=0.90$); it fails across backbones
($R^2=-0.86$), where each family traces its own curve. Delivered capture
rises toward a ${\sim}95\%$ width plateau, and absolute savings grow with
model cost. Appendix~\ref{app:scaling-detail} contains the operating-point
registry, frontier fits, width and model figures, calibration recipe, and
per-axis analyses.

% ===========================================================================
% CONCLUSION - owned by paper-skeleton agent. Honest limitations, no hype.
% ===========================================================================
\section{Conclusion}
\label{sec:conclusion}

We presented the first study of training-free caching in video diffusion
test-time search\reportonly{ and the method it leads to}. On \wan{}
\citep{wan2025wan}, seed-matched
results show aggressive caching preserves verifier-induced
rankings.\reportonly{ Median Spearman $\rhosp = 0.905$ at a ${\sim}2\times$
per-candidate speedup. This holds on the 50-prompt gate grid and the VBench
suite \citep{huang2023vbench}, with $72\%$ top-1 agreement. The cached winner retains
$90$--$94\%$ of the reward gain from full-compute search.
Errors concentrate on nearly tied candidates, where mistakes cost less.}
\cachedsearch{} explores under caching, then regenerates
only the winner at full compute. It captures $94.7\%$ of best-of-8 gains at
$63\%$ of the cost.\reportonly{ It captures $95.7\%$ of best-of-16 gains at
$57\%$ of the cost, or searches about twice as wide at a fixed budget. Across
six models and four architecture families ($1.3$B--$14$B), it holds at $14$B
within the same family. It transfers to three further architectures
after per-family calibration of $\tauc$. It also stacks with pruning-based
search, reaching a $3.11\times$ exploration speedup at $88.6\%$ capture.
It therefore complements existing search methods.} Table~\ref{tab:manual}
in Appendix~\ref{app:scaling-detail} condenses the practical guidance into
one lookup.

The lesson extends beyond caching.\reportonly{ Any lossy accelerator can support
exploration if it passes the same ranking-plus-regret audit.} Discarded
candidates need not be faithful. They need only be honest about ordering.\reportonly{
Allocating budgets across such accelerators is the next question.}

% Invisible anchor for the page-budget check (end of conclusion = end of the
% ICLR-limit-relevant main text; the optional repro/ethics statements follow).
\phantomsection\label{endofconclusion}

\subsubsection*{Reproducibility statement}
All experiments use publicly released checkpoints: \wan{} \citep{wan2025wan}, Wan2.1-T2V-14B \citep{wan2025wan}, Wan2.2-TI2V-5B \citep{wan2025wan22,wan2025wan}, CogVideoX-5B \citep{yang2024cogvideox}, HunyuanVideo-13B \citep{kong2024hunyuanvideo}, and LTX-Video-2B \citep{hacohen2024ltxvideo}. We use ImageReward \citep{xu2023imagereward} and VideoScore \citep{he2024videoscore} as verifiers with a training-free caching wrapper; no model is trained or fine-tuned.\reportonly{ Generation is deterministic given (prompt, seed, caching threshold $\tauc$), so every candidate and every delivered video in the paper can be re-materialized exactly.} Appendix~\ref{app:repro} gives exact configurations and measured costs. We will release the caching wrapper, search harness, evaluation scripts, prompt lists, seeds, and all per-candidate scores.

\subsubsection*{Ethics statement}
This work reduces the inference cost of an existing capability (text-to-video generation with test-time search) rather than introducing new generative capabilities; risks are those of the underlying open models. Cheaper high-quality sampling lowers the cost of both beneficial and malicious video synthesis. Provenance mitigations (watermarking, content credentials) are orthogonal to and unaffected by our method: caching changes how exploration candidates are computed, not what is delivered\reportonly{, and in the recommended recommit mode the delivered video is a standard full-compute sample}.

% Invisible anchor recording the last page of main text (read from the .aux
% by the page-budget check; adds no visible output to either edition).
\phantomsection\label{endofmain}\message{^^J[[ENDOFMAIN PAGETOTAL=\the\pagetotal of \the\textheight]]^^J}

\bibliography{references}
\bibliographystyle{iclr2026_conference}

\appendix
% =============================================================================
% APPENDIX - owned by the paper-appendix agent (see PROGRESS.md paper table).
% sections/theory.tex (Appendix A, app:theory) is owned by the paper-method
% agent and only \input here. Everything below theory is this agent's.
%
% Data provenance (NO fabricated numbers):
%   - Theory-vs-measurement: verified constants from theory.tex %-comments /
%     code/experiments/b1_theory_check.py; measured values re-derived from raw
%     jsonl by code/paper_figs/appendix_tables.py (matches PROGRESS.md).
%   - Category table + CogVideoX-5B numbers: computed by
%     code/paper_figs/appendix_tables.py from
%     results/b1_gate_{v0,cog5b_v0}/scores_shard*.jsonl (E4 grid complete,
%     800/800 records, verified 2026-07-07).
%   - Verifier-bias numbers: E1, results/b1_temporal/ (as quoted in
%     sections/ablations.tex sec:abl-temporal).
%   - Configs/costs: code/experiments/b1_gate{,_cog}.py defaults,
%     code/videogen1/{gen,gen_cog,caching}.py, code/slurm/*.sbatch, TASKS.md.
% =============================================================================

% =============================================================================
% SHARED APPENDIX HOSTING. Detailed blocks deferred from both main flows live
% here once. Narrow conference-only guards remain around method details that
% the report still presents inline.
% =============================================================================
\clearpage
\section*{Appendix contents}
\startcontents[appendices]
\titlecontents{section}
  [0em]
  {\small}
  {\contentslabel{2.6em}}
  {}
  {\titlerule*[0.6pc]{.}\contentspage}
\printcontents[appendices]{l}{1}{\setcounter{tocdepth}{1}}

\section{Extended related work}
\label{app:related}

This appendix hosts the full survey and positioning argument, pointed to
from Section~\ref{sec:intro}.

% SHARED BLOCK (both editions; single source). Report edition: main-text
% \section{Related Work} (sections/related.tex). Conference edition: hosted in
% the appendix under app:related (sections/appendix.tex). Do not edit headers.

\subsection{Training-free caching for video diffusion}
\label{sec:related-caching}

Training-free caching is the dominant inference-time accelerator for video diffusion transformers. These methods exploit a common pattern: adjacent-step feature and attention differences are large near the trajectory ends but small in the middle. This U-shaped redundancy appears across Open-Sora \citep{zheng2024opensora}, Latte \citep{ma2024latte}, Wan~2.1 \citep{wan2025wan}, and HunyuanVideo \citep{kong2024hunyuanvideo}. PAB \citep{zhao2024pab} broadcasts attention outputs across steps with reuse ranges ordered cross~$>$~temporal~$>$~spatial. TeaCache \citep{liu2024teacache} replaces uniform intervals with output-change estimates from timestep-embedding-modulated inputs. FasterCache \citep{lv2024fastercache} limits subtle feature loss from adjacent-step reuse. Its CFG-Cache also exploits redundancy between classifier-free-guidance branches at the same timestep. EasyCache \citep{zhou2025easycache} reuses per-step transformations while an online accumulated change indicator stays below $\tauc$. It needs no offline calibration, and we adopt this adaptive wrapper. BWCache \citep{liang2025bwcache} gates reuse at DiT-block granularity through relative-$L_1$ similarity. FastCache learns a linear approximation of cached features, and NaviCache self-calibrates its schedule \citep{liu2025fastcache,navicache2026}. LeMiCa \citep{lemica2025} casts scheduling as a shortest-path problem and uses lexicographic minimax optimization to bound globally accumulated error. ReCache \citep{aliev2026recache} instead learns budget-aware schedules with REINFORCE \citep{williams1992reinforce} against uncached targets. These methods move from local reuse heuristics toward global error budgets. The guidance computation that these caches reuse is itself an active design axis. Rectified-CFG++ \citep{saini2025rectifiedcfgpp} replaces the extrapolation step of classifier-free guidance with a predictor-corrector update that keeps rectified-flow samples on the learned manifold, and LGDM \citep{saini2025lgdm} reads guidance-conditioned latent-diffusion features as a perceptual quality signal. Both change what a trajectory or a score means rather than what it costs, so they act on the sampler and the verifier that cached exploration takes as given. Our work instead studies how these cache perturbations affect selection among multiple candidates.

Single-model training-free caching plateaus at roughly $2$--$3\times$ speedup at near-lossless quality. Prior work evaluates \emph{single-sample fidelity} to a full-compute reference through \PSNRMetric{}, \SSIMMetric{}, \LPIPSMetric{}, or absolute \VBenchMetric{} scores. These measurements do not reveal whether small perturbations reorder nearly tied candidates. We instead measure relative candidate ranking, which determines whether caching is safe inside test-time search.

\subsection{Test-time scaling and diffusion search}
\label{sec:related-search}

A complementary literature spends more inference compute to improve quality. \citet{ma2025inference} recast inference-time scaling as verifier-guided search over sampling noise after showing that extra denoising steps saturate. With modest search compute, a 0.6B model can outperform a 12B model without search. They also show that a single verifier can reward-hack, improving its own metric while degrading held-out metrics. We test this failure mode in Section~\ref{sec:abl-verifier}. Video-T1 \citep{liu2025videot1} applies multimodal reward models in a Tree-of-Frames search. EvoSearch \citep{he2025evosearch} uses evolutionary denoising to make Wan~1.3B competitive with Wan~14B. LatSearch \citep{latsearch2026} scores partially denoised latents and prunes before VAE decoding, matching EvoSearch-level gains at a fraction of the compute. Stream-T1 \citep{streamt1_2026} explores streaming generation with few-step chunks. Early-failure detection \citep{earlyfailure2026} aborts candidates from cheap RGB previews. Temporal backtracking search \citep{jun2026tbs} reallocates compute over time by restarting full-quality generation from verified prefixes instead of resampling whole rollouts. Unlike our method, each reduces search cost without reducing every candidate's per-step cost.

These methods prune or truncate candidates, but every survivor pays full per-step denoising cost. We reduce each candidate's rollout cost with training-free caching. This lever composes multiplicatively with existing search methods instead of competing with them, and we demonstrate stacking with latent-RM-style pruning. The composition preserves their pruning logic while lowering the cost of candidates that remain. Our question is whether cache approximation corrupts verifier ranking and silently breaks search.

\subsection{Speculative draft-verify sampling}
\label{sec:related-speculative}

Our explore-cheap/commit-full scheme is closest to speculative execution. \citet{debortoli2025speculative} extend speculative sampling to diffusion with acceptance rules for draft--target denoising pairs. For video, a 1.3B drafter and 14B autoregressive target achieve $1.59\times$ speedup at $98\%$ quality retention through reward-based routing \citep{sdvg2026}. FlowCast \citep{flowcast2026} uses a velocity-MSE acceptance test for training-free self-speculative flow matching. In image generation, T-Stitch \citep{pan2024tstitch} assigns early and late sampling to small and large models. HybridStitch \citep{hybridstitch2026} extends this split across pixels and timesteps. Unlike our method, these approaches accelerate a single trajectory through draft--target verification.

These methods verify within one trajectory, at a step or segment boundary, using a smaller or distilled draft model. We use the same model under aggressive caching, so the draft follows the same latent trajectory statistics by construction. Verification occurs at the candidate level: the search verifier scores cached rollouts, then only the winner is regenerated at full compute. We turn best-of-$N$ search into draft-verify instead of accelerating one trajectory.

\subsection{Composition of efficiency techniques}
\label{sec:related-composition}

Stacking inference accelerators is not automatically safe. Video-BLADE \citep{videoblade2025} shows that training-free composition of pretrained sparse attention with a step-distilled model degrades quality because distillation ignores the sparsity pattern. Joint training fixes the mismatch. QuantCache \citep{quantcache2025} combines caching, quantization, and pruning through shared heuristics. TurboDiffusion \citep{turbodiffusion2025} reaches $100$--$200\times$ acceleration by training sparse-linear attention, low-bit quantization, and consistency distillation together. Sol \citep{sol2026} tunes caching, sparsity, pruning, and quantization per model. Prior work combines methods only along the cost axis. It neither adds the quality axis of verifier-guided search nor tests whether approximation distorts the selection signal. We compose those axes and measure the effect on selection.

% SHARED BLOCK (both editions; single source). Report edition: closes the
% main-text \section{Related Work}. Conference edition: closes the appendix
% Extended Related Work section (app:related).
\paragraph{Positioning.} \cachedsearch{} is a multiplier layer between
training-free caching and test-time search. It uses cached rollouts as
candidate-level drafts, scores them with the search verifier, and restores a
full-compute output by recommitting the winning seed. Its ranking and regret
audit measures when this composition preserves the selection decision and how
much value is lost when it does not.

% =============================================================================
% APPENDIX: RANKING-NOISE ANALYSIS - owned by the paper-method agent.
% Include from sections/appendix.tex via \input{sections/theory}
% (appendix.tex is owned by the paper-appendix agent - coordinate there).
%
% All formulas verified numerically:
%   scripts: code/experiments/b1_theory_check.py (props 1-3, cost model) and
%   code/experiments/b1_theory_check_n.py (width subsection app:theory-width:
%   measured capture-vs-N, attenuation-corrected calibration, plug-in
%   predictions, bootstrap CIs; outputs quoted in % comments at each claim).
% Real-data checks use /scratch/09032/saini_2/VIDEO-GEN1/results/b1_gate_v0/
% (50 prompts x {full, cached tau=0.10}; original 8-seed pool, later extended
% to 16 seeds/prompt -- the width analysis uses the 16-seed grid).
% Measured constants from PROGRESS.md / pillars/B-budget-optimal/PLAN.md.
%
% Theorem environments: defined here only if main.tex has not already done so
% (guards below). Skeleton agent may move the \newtheorem lines to main.tex.
% =============================================================================

% - - environment guards (preamble-safe: kernel \newtheorem, no amsthm needed)
\makeatletter
\@ifundefined{proposition}{\newtheorem{proposition}{Proposition}}{}
\@ifundefined{corollary}{\newtheorem{corollary}{Corollary}}{}
\@ifundefined{assumption}{\newtheorem{assumption}{Assumption}}{}
\@ifundefined{remark}{\newtheorem{remark}{Remark}}{}
\@ifundefined{proof}{%
  \newenvironment{proof}[1][Proof]%
    {\par\noindent\emph{#1.}\ }{\hfill$\square$\par\medskip}}{}
\makeatother

\section{Ranking-noise model for cached exploration}
\label{app:theory}

This appendix models search value lost when cached scores select the winner.
Standard Gaussian-copula and order-statistic tools predict gain capture,
regret, and top-1 agreement from Section~\ref{sec:experiments}. We compare
them with data to support the empirical analysis.

\subsection{Setup and assumptions}
\label{app:theory-setup}

For prompt $c$ and width $N$, candidate $i$ has full score
$S_i = V(G(c, s_i), c)$ and cached score
$\hat S_i = V(G_{\tauc}(c, s_i), c)$ (Section~\ref{sec:exp-setup}). Both are
deterministic given seed $s_i$; randomness is over i.i.d.\ seeds.
\cachedsearch{} selects $\hat\imath = \argmax_i \hat S_i$. Recommit delivers
$S_{\hat\imath}$ by seed-determinism (Section~\ref{sec:algorithm}); the
full-compute oracle attains $\max_i S_i$. We use three assumptions:

\begin{assumption}[Exchangeable candidates]
\label{ass:iid}
The pairs $(S_i, \hat S_i)$, $i = 1, \dots, N$, are i.i.d.\ across
candidates. This reflects the protocol: seeds are drawn independently and
enter symmetrically.
\end{assumption}

\begin{assumption}[Gaussian-copula ranking noise]
\label{ass:copula}
Within a prompt, $(S_i, \hat S_i)$ is bivariate normal with
$S_i \sim \gN(\mu, \sigma^2)$ and correlation $r \in [0, 1]$. Only the
\emph{rank} structure of the cached scores matters for selection ($\argmax$
is invariant to strictly increasing transforms of $\hat S$), so the marginal
of $\hat S$ is irrelevant and Assumption~\ref{ass:copula} is really two
parts: (i) a Gaussian dependence structure (copula) with parameter $r$
between full and cached scores, and (ii) a Gaussian marginal for the full
score. Part (ii) is what lets us convert rank statements into reward units.
\end{assumption}

\begin{assumption}[Verifier as value]
\label{ass:verifier}
The full score $S$ is taken as the value of a candidate. Verifier
mis-specification (reward hacking) is a real but orthogonal failure mode of
all verifier-guided search, treated empirically via verifier ensembles in
Section~\ref{sec:experiments}; the analysis here isolates the \emph{additional}
error introduced by caching.
\end{assumption}

Under Assumption~\ref{ass:copula}, write $S_i = \mu + \sigma Z_i$ and
standardize $\hat S_i$ as $\hat Z_i$. Each $(Z_i, \hat Z_i)$ is standard
bivariate normal with correlation $r$, and
$\hat\imath = \argmax_i \hat Z_i$. Define
\begin{equation}
e_N \;=\; \E\!\left[\max_{i \le N} Z_i\right]
 \;=\; \int_{-\infty}^{\infty} z\, N \phi(z) \Phi(z)^{N-1}\, dz
\label{eq:eN}
\end{equation}
as the expected maximum of $N$ standard normals
($e_2 = 1/\sqrt{\pi} \approx 0.564$; $e_3 \approx 0.846$;
$e_4 \approx 1.029$; $e_6 \approx 1.267$; $e_8 \approx 1.424$;
$e_{16} \approx 1.766$).
% VERIFIED (quadrature, verify_theory.py): e_2=0.5642 e_3=0.8463 e_4=1.0294
% e_6=1.2672 e_8=1.4236 e_12=1.6292 e_16=1.7660; MC (2e6 draws) agrees to 3dp.

\subsection{Value of the noisy argmax}

The first result quantifies the value lost to noisy selection:

\begin{proposition}[Selection under rank noise]
\label{prop:capture}
Under Assumptions~\ref{ass:iid}--\ref{ass:copula},
\begin{equation}
\E\big[S_{\hat\imath}\big] \;=\; \mu + \sigma\, r\, e_N,
\qquad
\E\big[\max_i S_i\big] \;=\; \mu + \sigma\, e_N .
\end{equation}
Consequently the expected regret of \cachedsearch{} relative to full-compute
best-of-$N$ is
\begin{equation}
\gR(N, r) \;=\; \E\big[\max_i S_i - S_{\hat\imath}\big]
 \;=\; \sigma\,(1 - r)\, e_N ,
\label{eq:regret}
\end{equation}
and the expected \emph{gain capture} (the fraction of the best-of-$N$
improvement over a random candidate that survives noisy selection) is
\begin{equation}
\mathrm{capture}(N, r)
 \;=\; \frac{\E[S_{\hat\imath}] - \mu}{\E[\max_i S_i] - \mu}
 \;=\; r,
\qquad \text{independent of } N .
\label{eq:capture}
\end{equation}
\end{proposition}

\begin{proof}
Decompose $Z_i = r \hat Z_i + \sqrt{1 - r^2}\, \eta_i$ with
$\eta_i \sim \gN(0,1)$ independent of $(\hat Z_j)_{j \le N}$; this reproduces
the joint law of Assumption~\ref{ass:copula}. The index
$\hat\imath$ is a function of $(\hat Z_j)$ alone, so
$\E[\eta_{\hat\imath}] = \E\big[\E[\eta_{\hat\imath} \mid (\hat Z_j)]\big]
= 0$ and
$\E[Z_{\hat\imath}] = r\, \E[\hat Z_{\hat\imath}] = r\, \E[\max_i \hat Z_i]
= r\, e_N$, using that $\hat Z$ is standard normal. The oracle term is
Eq.~\eqref{eq:eN}; subtract and rescale by $\sigma$. For
Eq.~\eqref{eq:capture}, a random (or the first) candidate has expected value
$\mu$.
\end{proof}
% VERIFIED (MC, 2.5e5 trials/cell): r=0.832: capture 0.834/0.832/0.832 at
% N=2/4/8; r=0.912: capture 0.912/0.911/0.912. E[Z_sel] matches r*e_N to 3dp.

Rank noise removes a constant gain fraction instead of compounding with width,
so cached exploration can widen search (Section~\ref{sec:cost}). At $r = 0$,
selection is uniform and $\E[S_{\hat\imath}] = \mu$; for $r \ge 0$, expected
value never falls below the no-search baseline. Recommit still delivers a
full-compute sample.

We calibrate latent $r$ from the observed per-prompt Spearman $\rhosp$:

\begin{corollary}[Spearman calibration]
\label{cor:spearman}
Under Assumption~\ref{ass:copula}, the classical Gaussian rank-correlation
relationship \citep{kruskal1958} maps the measured median
$\rhosp=0.905$ to $r\approx0.913$ and the mean $\rhosp=0.820$ to
$r\approx0.833$.
\end{corollary}
% VERIFIED: r=0.912 -> simulated Spearman 0.9046 (5e5 pairs); target 0.905.

On the extended $16$-seed grid, Proposition~\ref{prop:capture} with
median-calibrated $r$ predicts capture $\approx 0.91$, flat in $N$. Measured
capture is $91.1\%$, $90.9\%$, $92.2\%$,
and $95.7\%$ at $N = 2, 4, 8, 16$. The $N{=}4 \to 16$ gain is
$+4.8$ points. A $32$-seed extension remains at $95.2\%$
(Figure~\ref{fig:width-scaling}).
Section~\ref{app:theory-width} separates finite-sample level bias from the
Gaussian-marginal shape error in Assumption~\ref{ass:copula}(ii), and
Table~\ref{tab:widthpred} corrects the prediction.
Proposition~\ref{prop:spread} explains the aggregate level.

\subsection{Top-1 agreement and the zero-regret median}

Top-1 agreement determines the probability of zero regret:

\begin{proposition}[Top-1 agreement]
\label{prop:top1}
Under Assumptions~\ref{ass:iid}--\ref{ass:copula}, with
$\phi_2(\cdot, \cdot\,; r)$ and $\Phi_2(\cdot, \cdot\,; r)$ the standard
bivariate normal density and CDF,
\begin{equation}
p_1(N, r) \;=\; \Pr\big[\argmax_i Z_i = \argmax_i \hat Z_i\big]
 \;=\; N \int_{\R^2} \phi_2(z, \hat z; r)\,
       \Phi_2(z, \hat z; r)^{N-1}\, dz\, d\hat z ,
\label{eq:top1}
\end{equation}
which is increasing in $r$ and decreasing in $N$. Since ties have measure
zero, $\Pr[\text{regret} = 0] = p_1(N, r)$, so the \emph{median} regret is
$0$ if and only if $p_1(N, r) \ge \tfrac12$.
\end{proposition}

\begin{proof}
By exchangeability, $p_1 = N \Pr[Z_1 = \max_i Z_i,\,
\hat Z_1 = \max_i \hat Z_i]$. Conditioning on $(Z_1, \hat Z_1) = (z, \hat z)$,
the remaining $N - 1$ pairs are i.i.d.\ bivariate normal, so the conditional
probability is $\Phi_2(z, \hat z; r)^{N-1}$; integrate against
$\phi_2$. Monotonicity in $N$ is immediate; monotonicity in $r$ follows from
the usual Gaussian-comparison (Slepian-type) argument, and we also confirm
it numerically over the range used here.
\end{proof}
% VERIFIED (MC, 4e6 draws): p1(2,.833)=.812 p1(4,.833)=.680 p1(8,.833)=.581
%                           p1(2,.913)=.866 p1(4,.913)=.767 p1(8,.913)=.689

At $N = 8$, $p_1(8, r) = 0.58$ for mean-calibrated $r = 0.833$ and $0.69$ for
median-calibrated $r = 0.913$. Measured top-1 agreement at $N = 8$ and
$\tauc = 0.10$ is $64\%$, within this bracket. Both exceed $\tfrac12$, matching median
regret $0$ at $N = 8$ and $64\%$ zero-regret prompts. Yet
$\rhosp \approx 0.9$ still implies $\sim\!30\%$ top-1 disagreement. Top-1
safety depends on disagreement value, not agreement alone.
Proposition~\ref{prop:capture} bounds that value by $\sigma (1 - r) e_N$;
the next result locates it.

\subsection{Heterogeneous-prompt regret}

We aggregate Proposition~\ref{prop:capture} over prompts with different
score spreads $\sigma_p$ and ranking fidelities $r_p$:

\begin{proposition}[Spread-weighted capture]
\label{prop:spread}
Let prompts $p$ have parameters $(\mu_p, \sigma_p, r_p)$ satisfying
Assumptions~\ref{ass:iid}--\ref{ass:copula} prompt-wise. Then the expected
per-prompt regret is $\gR_p = \sigma_p (1 - r_p) e_N$, and the aggregate
capture over prompts is the \emph{spread-weighted} mean correlation
\begin{equation}
\overline{\mathrm{capture}}
 \;=\; \frac{\E_p[\sigma_p r_p]}{\E_p[\sigma_p]}
 \;=\; \bar r + \frac{\mathrm{Cov}_p(\sigma_p, r_p)}{\E_p[\sigma_p]} ,
\label{eq:weightedcapture}
\end{equation}
where $\bar r = \E_p[r_p]$. Hence if spread and ranking fidelity are
positively associated, $\mathrm{Cov}_p(\sigma_p, r_p) > 0$, the aggregate
capture strictly exceeds the mean per-prompt capture, and the regret mass
concentrates on low-spread prompts, where by
$\gR_p \le \sigma_p e_N$ the attainable loss is small in absolute terms.
\end{proposition}

\begin{proof}
Sum Proposition~\ref{prop:capture} over prompts: aggregate gain of the
oracle is $\E_p[\sigma_p] e_N$ and of \cachedsearch{} is
$\E_p[\sigma_p r_p] e_N$; divide. The covariance identity is
$\E[\sigma r] = \E[\sigma]\E[r] + \mathrm{Cov}(\sigma, r)$. The bound
$\gR_p \le \sigma_p e_N$ is Eq.~\eqref{eq:regret} with $r_p \ge 0$.
\end{proof}

Ranking corruption is self-limiting on low-spread prompts, where mistakes are
cheap. Spread and $\rhosp$ have Spearman association $+0.25$; the $9/50$
prompts with $\rhosp < 0.7$ have low spread. Aggregate capture ($94\%$) exceeds
mean-calibrated capture ($0.83$) and its spread-weighted correction ($0.85$).
% VERIFIED on real gate data (b1_gate_v0, 50 prompts x 8 seeds, tau=0.10):
%   corr(spread, rho) = +0.250 (Spearman)  [matches PROGRESS.md +0.25]
%   measured mean regret 0.039, median 0.000, zero-regret 64% of prompts
%   predicted mean regret sum sigma_p(1-r_p)e_8 / P = 0.107  (conservative)
%   measured aggregate capture 0.940; predicted spread-weighted 0.853,
%   unweighted 0.829; 64% of total regret mass on below-median-spread prompts.

\begin{remark}[The model is conservative]
\label{rem:conservative}
The model correctly predicts median-zero regret and self-limiting corruption,
but it predicts mean regret $0.107$ instead of $0.039$ and capture $0.85$
instead of $0.94$. It overestimates loss because it spreads disagreement
across the score range, whereas observed swaps cluster among near-tied
mid-pack candidates. This error is conservative because measured selection is
more stable than the model predicts.
\end{remark}

\subsection{Capture versus search width}
\label{app:theory-width}

Equation~\eqref{eq:capture} remains flat in $N$ under prompt heterogeneity:
$e_N$ cancels from Eq.~\eqref{eq:weightedcapture} for any joint distribution
of $(\sigma_p, r_p)$. Measurements disagree. Across all $\binom{16}{N}$
subsets of the extended $16$-seed grid, capture is $91.1\%$, $90.9\%$,
$92.2\%$, and $95.7\%$ at
$N = 2, 4, 8, 16$ (Section~\ref{sec:exp-main},
Figure~\ref{fig:width-scaling}). The $N{=}4$ to $16$ increase is $+4.8$
points. Heterogeneity cannot create this
$N$-dependence, so a within-prompt idealization fails. Two effects explain
the level and direction, with an unmodeled remainder.
% VERIFIED (b1_theory_check_n.py): measured capture(N) on 16-seed grid
%   91.1 / 90.9 / 92.2 / 95.7 % at N=2/4/8/16 (exact subset-weight formulas
%   reproduce b1_simulate.py enumeration to 3e-16; 8-seed restriction
%   reproduces tab:main 89.9/89.6/94.7). Paired bootstrap (4000 resamples):
%   cap(16)-cap(4) = +4.8 [+1.1,+8.2]; cap(16)-cap(8) = +3.5 [+0.9,+5.9];
%   cap(8)-cap(2) = +1.1 [-1.4,+3.6] (n.s.). Heterogeneity flat: 0.879 all N.

First, width-independence comes from the Gaussian marginal, not rank noise:

\begin{proposition}[Capture under a general score marginal]
\label{prop:width}
Keep the Gaussian copula of Assumption~\ref{ass:copula}(i) with parameter
$r > 0$, but let the full score have an arbitrary continuous marginal $F$
with mean $\mu_F$ and $\E|S| < \infty$, i.e.\ $S_i = g_F(Z_i)$ with
$g_F = F^{-1} \circ \Phi$ nondecreasing. With
$M_N = \max_{i \le N} \hat Z_i$ and $\eta \sim \gN(0,1)$ independent,
\begin{equation}
\mathrm{capture}(N, r, F)
 \;=\; \frac{\E\big[g_F\big(r M_N + \sqrt{1 - r^2}\,\eta\big)\big] - \mu_F}
            {\E\big[g_F(M_N)\big] - \mu_F} .
\label{eq:capturef}
\end{equation}
If $F$ is Gaussian, capture equals $r$ as in
Proposition~\ref{prop:capture}; if $F$ is bounded above by
$b=\operatorname{ess\,sup}S$, both selected and best-candidate values
approach $b$, so
\begin{equation}
\mathrm{capture}(N, r, F) \;\longrightarrow\; 1
\qquad (N \to \infty).
\end{equation}
The bounded-score case gets the measured direction right but underestimates
the rise, which makes its width prediction conservative.
\end{proposition}

\begin{proof}
As in Proposition~\ref{prop:capture},
$Z_{\hat\imath} \overset{d}{=} r M_N + \sqrt{1 - r^2}\,\eta$ and
$\max_i S_i = g_F(\max_i Z_i)$ with $\max_i Z_i \overset{d}{=} M_N$; a
random candidate has mean $\mu_F$. This gives Eq.~\eqref{eq:capturef}, and
the Gaussian case is immediate. For the bounded case, couple widths by taking
maxima over the first $N$ terms of one i.i.d.\ sequence, so
$M_N \uparrow \infty$ a.s.\ and both
arguments in Eq.~\eqref{eq:capturef} increase a.s. Then
$b - g_F(\cdot) \ge 0$ decreases pointwise to $0$ (as $g_F(z) \to b$ when
$z \to \infty$), is integrable at $N = 2$ (since
$\E|S_{\hat\imath}| \le \E\sum_i |S_i| < \infty$), and monotone
convergence gives both expectations $\uparrow b$; the ratio tends to
$(b - \mu_F)/(b - \mu_F) = 1$.
\end{proof}

Verifier scores are bounded in practice (rewards lie in $[-2.3, 2.3]$;
Section~\ref{sec:exp-setup}), so capture must rise with width. Flat
extrapolation of Eq.~\eqref{eq:capture} is conservative. Yet plugging the
observed $16$-score vectors into Eq.~\eqref{eq:capturef} predicts less than
one point of growth from $N{=}2$ to $16$.

Second, the calibration:

\begin{remark}[Finite-sample calibration attenuation]
\label{rem:attenuation}
Finite-sample Spearman correlation understates the latent relationship:
at $n=16$, population $\rhosp=0.905$ has expected sample value $0.875$
\citep{kruskal1958}. Correcting this attenuation raises median calibrated
$r$ from $0.900$ to $0.927$, raises the flat capture prediction from $0.879$
to $0.903$, and reduces regret overprediction from $2.7\times$ to
$2.0\times$. The correction improves the level but leaves the width trend
flat, so the remaining error still overstates risk.
\end{remark}
% VERIFIED (b1_theory_check_n.py Section 7,9): E[r_S] identity MC 0.8744 vs formula
% 0.8749 at n=16, r=0.913; median r 0.900->0.927; flat 0.879->0.903;
% 8-seed pool: predicted regret 0.107->0.077 vs measured 0.039 (2.7x->2.0x).

\newcommand{\WidthPredCaption}{\textbf{The parameter-free rank-noise model matches capture
through $N=8$ but is conservative at $N=16$.} Predicted and measured
gain capture on the $50$-prompt, $16$-seed grid at $\tauc=0.10$.}
\newcommand{\WidthPredBody}{%
\begin{tabular}{ccccc}
\toprule
$N$ & measured & predicted (model sd) & meas.\,$-$\,pred. & model quantile \\
\midrule
$2$  & $91.1\%$ & $90.3\%$ ($1.0$) & $+0.8$ & $0.81$ \\
$4$  & $90.9\%$ & $90.4\%$ ($1.2$) & $+0.5$ & $0.65$ \\
$8$  & $92.2\%$ & $90.0\%$ ($1.9$) & $+2.2$ & $0.87$ \\
$16$ & $95.7\%$ & $90.8\%$ ($3.3$) & $+4.9$ & $0.94$ \\
\bottomrule
\end{tabular}}

Table~\ref{tab:widthpred} applies both corrections without free parameters.
It preserves each score vector, adds Gaussian rank noise at corrected $r_p$,
and uses the ratio-of-mean-gains estimator over every seed subset.
The ``model sd'' column reports estimator sampling noise from one modeled
noise realization per prompt. Predictions match the level through
$N = 8$ within sampling noise. At $N = 16$, the model is
conservative by $4.9$ points: $90.8\%$ predicted versus $95.7\%$ measured,
at model quantile $0.94$. It captures the direction but not the full rise.
% VERIFIED (b1_theory_check_n.py Section 7-8): M2c plug-in 90.3/90.4/90.0/90.8 at
% N=2/4/8/16, estimator sd 1.0/1.2/1.9/3.3, measured at quantiles
% .81/.65/.87/.94. Raw-calibration plug-in for reference: 87.8/88.0/87.8/88.7.

% The report edition used a wraptable here; on a page that ends with a
% proposition it left roughly half a page blank, so both editions now use a
% normal top float (user 2026-07-25).
\begin{table}[!htbp]
\caption{\WidthPredCaption}
\label{tab:widthpred}
\centering
\small
\WidthPredBody
\end{table}

Measured rank errors protect the top better than exchangeable noise. At
$N=16$, the cached winner is in the true top two for $92\%$ of prompts,
versus $86\%$ under the corrected Gaussian copula, so the model predicts
$90.8\%$ capture instead of the measured $95.7\%$. We therefore keep the
headline empirical and use the model only as a conservative envelope.
% VERIFIED (b1_theory_check_n.py Section 5-6,8): pool top-1/top-2 measured .70/.92
% vs M2c .72/.86; additive-noise top-1 .50, capture 91.6/90.9/90.3/91.0;
% t-copula nu in {3..20}: best top-2 .89, capture(N=16) <= 91.9 (nu=3,
% corrected). Measured regret .029/.051/.057/.041 vs corrected Gaussian
% prediction .030/.054/.075/.093 at N=2/4/8/16.

\subsection{Iso-cost comparison and break-even}

We now combine the value and cost models from Section~\ref{sec:cost}:

\begin{proposition}[Break-even and iso-cost widening]
\label{prop:isocost}
Let $\gamma = C_c / C_f$. (i) Recommit costs less than full-compute
best-of-$N$ iff $N > 1/(1 - \gamma)$. (ii) At a fixed budget $B \ge C_f$,
full-compute search affords $N_f = B / C_f$ candidates while \cachedsearch{}
(recommit) affords $N_c = (B - C_f)/C_c = (N_f - 1)/\gamma$; under
Assumptions~\ref{ass:iid}--\ref{ass:copula}, \cachedsearch{} attains higher
expected value at equal cost iff
\begin{equation}
r\, e_{N_c} \;>\; e_{N_f} .
\label{eq:isocost}
\end{equation}
\end{proposition}

\begin{proof}
(i) is Eq.~\eqref{eq:breakeven}. (ii) substitutes each width into
Proposition~\ref{prop:capture}: expected gains over no-search are
$\sigma r e_{N_c}$ and $\sigma e_{N_f}$.
\end{proof}

For measured $\gamma = 0.508$ and median-calibrated $r = 0.913$, $N_f = 2$
gives $N_c \approx 2$ and $r e_2 = 0.52 < e_2 = 0.56$. Cached search loses
because recommit consumes the savings. At $N_f = 3$, $N_c \approx 4$ and
$r e_4 = 0.94 > e_3 = 0.85$; at $N_f = 4$, $N_c \approx 6$ and
$r e_6 = 1.16 > e_4 = 1.03$. The margin grows with $B$ because $e_N$ grows
as $\sim\!\sqrt{2 \ln N}$ while $r$ stays constant. These two constants
$(\gamma, \rhosp)$ reproduce the strategy simulation in
Section~\ref{sec:experiments}: a wash at $N = 2$ and an iso-cost advantage
from $N \approx 4$ onward.
% VERIFIED: gamma=.508 -> N*=2.03; iso-cost cells: N_f=2: r*e_2=.515<.564;
% N_f=3: r*e_4=.940>.846; N_f=4: r*e_6=1.157>1.029. Matches measured
% simulation (PLAN.md): N=2 wash; N=8 capture 94.7% at 63% cost.

\paragraph{Scope and limitations.}
The parameter $r$ belongs to the (model, cache threshold $\tauc$, verifier)
triple and requires per-prompt measurement. The $\tauc$ sweep moves median
$\rhosp$ from $0.905$ ($\tauc \le 0.10$) to $0.857$ ($\tauc = 0.20$).
Assumption~\ref{ass:copula} is conservative in the direction documented by
Remark~\ref{rem:conservative}. Assumption~\ref{ass:verifier} excludes reward
hacking, which caching neither causes nor fixes. Keep-draft falls outside
these guarantees because it delivers a cached sample
(Section~\ref{sec:algorithm}). Choosing $\tauc$ online from the
spread in Proposition~\ref{prop:spread} does not clear the fixed-$\tauc$
frontier (Section~\ref{sec:abl-adaptive}); its noisy two-sample probe costs
more than the flat capture-versus-speedup curve returns.

% ======================================================================
\section{Theory and measurement}
\label{app:appendix}  % section agents: KEEP this label - other sections \ref it

Appendix~\ref{app:theory} predicts outcomes from two measured constants,
$\rhosp = 0.905$ and $\gamma = C_c/C_f = 0.508$, without fitting. The former
maps to $r = 0.913$ through Corollary~\ref{cor:spearman}.
Table~\ref{tab:theory-meas} compares each prediction with the $50 \times 8$
grid at $\tauc = 0.10$ (Sections~\ref{sec:experiments}
and~\ref{sec:ablations}).

\begin{table}[h]
\caption{\textbf{The rank-noise model predicts the main grid's qualitative
behavior but is conservative on regret.} Predictions vs.\ measurements
for $50$ prompts, $8$ seeds, and $\tauc=0.10$; source results are in
Tables~\ref{tab:main} and~\ref{tab:tau}.}
\label{tab:theory-meas}
\centering
\small
\begin{tabular}{@{}p{0.25\linewidth}p{0.20\linewidth}p{0.20\linewidth}p{0.24\linewidth}@{}}
\toprule
quantity & predicted & measured & source \\
\midrule
gain capture at $N{=}8$ & $0.83$--$0.91$ ($\approx r$) & $94.7\%$ & Prop.~\ref{prop:capture}; Tab.~\ref{tab:main} \\
capture trend in $N$ & flat ($=r$, $\forall N$) & $89.9/89.6/94.7\%$ at $N{=}2/4/8$ & Prop.~\ref{prop:capture}; Tab.~\ref{tab:main} \\
top-1 agreement $p_1(8, r)$ & $0.58$--$0.69$ & $64\%$ & Prop.~\ref{prop:top1}; Section~\ref{sec:exp-ranking} \\
median regret & $0$ (since $p_1 > \tfrac12$) & $0$ ($64\%$ zero-regret) & Prop.~\ref{prop:top1}; Section~\ref{sec:exp-regret} \\
mean regret (reward units) & $0.107$ & $0.040$ & Eq.~\eqref{eq:regret}; Section~\ref{sec:exp-regret} \\
aggregate capture (per-prompt plug-in) & $0.829$ / $0.853$ (spread-wtd.) & $0.940$ & Prop.~\ref{prop:spread} \\
sign of corr(spread, $\rhosp$) & $> 0$ (self-limiting) & $+0.25$ & Prop.~\ref{prop:spread}; Section~\ref{sec:abl-corruption} \\
regret mass on below-median-spread prompts & majority & $64\%$ & Prop.~\ref{prop:spread} \\
break-even width $N^\ast$ & $2.03$ & $N{=}2$ is a wash & Prop.~\ref{prop:isocost}; Tab.~\ref{tab:main} \\
iso-cost winner at $N_f \ge 3$ & \cachedsearch{} ($r e_{N_c} > e_{N_f}$) & $+38\%$ gain at ${\approx}300$\,s & Prop.~\ref{prop:isocost}; Section~\ref{sec:exp-main} \\
\bottomrule
\end{tabular}
\end{table}

Ambiguous predictions show the bracket from mean-calibrated
$r=0.833$ and median-calibrated $r=0.913$ copula parameters.
\paragraph{Model agreement.} The structural predictions hold.
Capture is nearly flat in $N$, top-1 agreement falls inside the calibration
bracket, and low-spread prompts absorb most corruption. The cost model also
recovers the $N{=}2$ wash and growing iso-cost advantage. Median regret is
zero because $p_1(8, r) > \tfrac12$.

\paragraph{Regret over-prediction.}
Equation~\eqref{eq:regret}, evaluated with measured $(\sigma_p, \rhosp_p)$,
predicts $0.107$ mean regret against $0.040$ measured. It therefore
over-predicts regret by $2.7\times$ and under-predicts capture ($0.85$ vs.\
$0.94$). A Gaussian copula spreads rank errors uniformly across the score
range (Remark~\ref{rem:conservative}), while measured swaps cluster among
near-tied mid-pack candidates (Figure~\ref{fig:corruption}). The argmax is
therefore more stable than $\rhosp$ suggests. Appendix~\ref{app:theory}
provides a conservative deployment guide: realized regret is lower at all
three $\tauc$ values. Headline results should still come from measurement.

% ======================================================================
% SHARED: analyses, ablation subsections, and figures deferred from both main
% flows. Blocks still present in the report main flow retain narrow
% conference-only guards.
\section{Additional Experiments}
\label{app:more-abl}

This section collects analyses and figures summarized in
Section~\ref{sec:experiments}--Section~\ref{sec:ablations}.

\subsection{Main-result figures}
\label{app:extra-figs}

Figure~\ref{fig:pareto} expands the strategy comparison in
Section~\ref{sec:exp-main}. The remaining exhibits provide the supporting
views referenced from the main analysis.

\ifreport\else
% SHARED FLOAT (single source, conference edition only as a standalone float).
% Report edition: the same panel appears via \PanelRankScatter inside the
% two-panel figure of sections/experiments.tex (Section exp-ranking). Conference
% edition: this standalone float is hosted in the appendix (app:extra-figs).
% The panel body (graphics+caption+label fig:rank-scatter) is defined ONCE as
% \PanelRankScatter in sections/experiments.tex.
\begin{figure}[t]
  \centering
  \begin{minipage}[t]{0.48\linewidth}
    \PanelRankScatter
  \end{minipage}
\end{figure}

\fi

\ifreport
% ======================================================================
% fig:pareto (f3) and fig:adaptive (f7), paired in the appendix.
% ======================================================================
\begin{figure}[!tb]
  \centering
  \begin{minipage}[t]{0.485\linewidth}
    \vspace{0pt}
    \centering
    \includegraphics[width=\linewidth]{figs/f3_pareto.pdf}
    \captionof{figure}{\textbf{\cachedsearch{} dominates best-of-$N$ beyond
    the $N^\ast\approx2$ break-even.} Reward gain vs.\ wall-clock; up-left is
    better. Curves compare full, keep, and commit for $N=2,4,8$ on
    $50$ prompts; the gold diamond adds pruning.}
    \label{fig:pareto}
  \end{minipage}\hfill
  \begin{minipage}[t]{0.485\linewidth}
    \vspace{0pt}
    \centering
    \includegraphics[width=\linewidth]{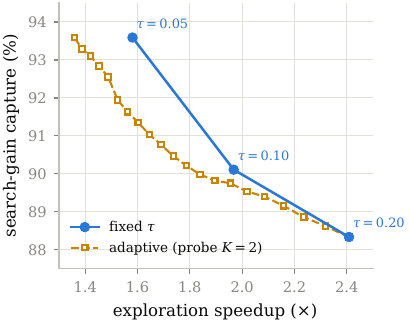}
    \captionof{figure}{\textbf{Adaptive per-prompt $\tauc$ never beats the
    fixed frontier.} Gain capture vs.\ exploration speedup at $N=8$ and
    $n=50$. Filled blue points are fixed thresholds; open gold points sweep
    the $K=2$ probe policy's decision threshold.}
    \label{fig:adaptive}
  \end{minipage}
\end{figure}

\else
\begin{figure}[!tb]
  \centering
  \includegraphics[width=0.55\linewidth]{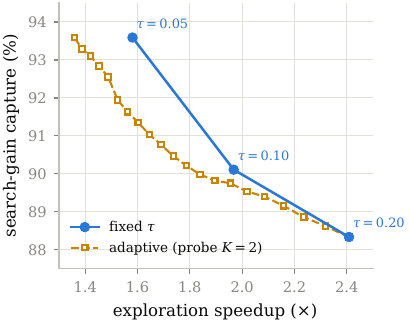}
  \caption{\textbf{Adaptive per-prompt $\tauc$ never beats the fixed
  frontier.} Gain capture vs.\ exploration speedup at $N=8$ and $n=50$.}
  \label{fig:adaptive}
\end{figure}
\fi

% Delivered-video suite and VBench-2.0 replication are appendix material in
% both editions; their shared main-flow summaries point here.
% SHARED APPENDIX BLOCK (both editions; single source). Both main flows keep
% the two-sentence headline and point here for the full metric suite.
% Measured source: results/multimetric/* -- selection, external scoring, and
% aggregation by code/paper_figs/multimetric.py (VBench official dims via the
% isolated vbench-eval env; RAFT flow + LPIPS in the main env). Numbers JSON:
% results/multimetric/multimetric_numbers.json.
\subsection{Additional metrics}
\label{sec:exp-multimetric}

So far, the search verifier also evaluates each strategy. We instead
re-evaluate the delivered videos under a standard metric suite on a held-out
\VBenchMetric{} subset from the validated population of
Section~\ref{sec:abl-verifier}. The strategies deliver a single,
best-of-$8$, keep, or commit video. We score the original rollouts
without regeneration. The suite includes six \VBenchMetric{}
dimensions\footnote{Computed in
\texttt{custom\_input} mode, which applies every dimension to every video;
the \VBenchMetric{} protocol computes temporal flickering only on static-scene
prompt subsets, so the \VBenchMetric{} rows are strategy-\emph{comparisons}
on a fixed prompt set, not leaderboard-comparable absolutes.},
\VideoScoreMetric{}, \ImageRewardMetric{}, mean RAFT optical-flow magnitude,
and \LPIPSMetric{} against the
same-seed full reference. Costs follow Section~\ref{sec:exp-setup}\reportonly{
(text encoding and VAE decodes are excluded from the FLOPs ratio but included
in wall-clock)}.

\begin{table}[!tb]
\caption{\textbf{Commit matches best-of-$8$ across the standard metric
suite at lower cost.} Four delivery strategies on a held-out
\VBenchMetric{} subset
at $\tauc=0.10$, $N=8$. Rows are strategies and columns are quality,
fidelity, and cost metrics; green cells are statistically consistent with
best-of-$8$.}
\label{tab:multimetric}
\centering
\tabsize
\setlength{\tabcolsep}{3.5pt}
\resizebox{\linewidth}{!}{%
\begin{tabular}{l cc cccccc cc cc}
\toprule
& \multicolumn{2}{c}{Quality} & \multicolumn{6}{c}{\VBenchMetric{} dimensions} & \multicolumn{2}{c}{Fidelity} & \multicolumn{2}{c}{Cost (\% bo8)} \\
\cmidrule(lr){2-3}\cmidrule(lr){4-9}\cmidrule(lr){10-11}\cmidrule(lr){12-13}
method & IR$\uparrow$ & VS$\uparrow$ & subj$\uparrow$ & bg$\uparrow$ & mot$\uparrow$ & flk$\uparrow$ & aes$\uparrow$ & img$\uparrow$ & flow & \LPIPSMetric{}$\downarrow$ & wall$\downarrow$ & FLOPs$\downarrow$ \\
\midrule
single (no search) & $-0.06$ & $2.58$ & $0.970$ & $0.975$ & $\mathbf{0.992}$ & $\mathbf{0.990}$ & $0.554$ & $65.3$ & $1.85$ & $0$ & $\mathbf{12.5}$ & $\mathbf{12.5}$ \\
best-of-$8$ full & $\mathbf{+0.81}$ & $2.60$ & $\mathbf{0.978}$ & $\mathbf{0.980}$ & $0.991$ & $0.987$ & $\mathbf{0.615}$ & $\mathbf{67.7}$ & $1.56$ & $0$ & $100$ & $100$ \\
\midrule
\rowcolor{blue!8}
\cachedsearch{}-keep & $+0.79$ & $\mathbf{2.67}$ & $0.976$ & $0.978$ & $0.991$ & $0.987$ & $0.604$ & $66.7$ & $1.96$ & $0.114$ & $51.4$ & $48.5$ \\
\rowcolor{blue!8}
\cachedsearch{}-commit & $+0.75$ & \cellcolor{green!12}$2.60$ & \cellcolor{green!12}$0.976$ & $0.979$ & \cellcolor{green!12}$0.991$ & \cellcolor{green!12}$0.987$ & \cellcolor{green!12}$0.610$ & \cellcolor{green!12}$67.3$ & $1.97$ & \cellcolor{green!12}$0$ & $63.9$ & $61.0$ \\
\bottomrule
\end{tabular}}
\end{table}

\begin{figure}[!tb]
  \centering
  \includegraphics[width=\linewidth]{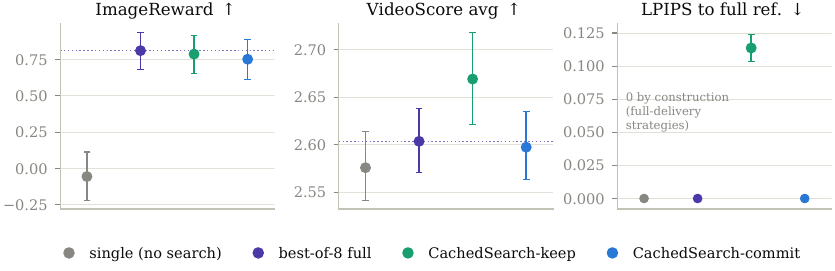}
  \caption{\textbf{Commit matches best-of-$8$ quality at $61\%$ of its
  transformer FLOPs.} Mean \ImageRewardMetric{}, \VideoScoreMetric{}, and
  same-seed \LPIPSMetric{} on the held-out \VBenchMetric{} subset; bars are
  $95\%$ prompt-bootstrap intervals.
  Figure~\ref{fig:multimetric-full} gives all metrics.}
  \label{fig:multimetric}
\end{figure}

\begin{figure}[!tb]
  \centering
  \includegraphics[width=\linewidth]{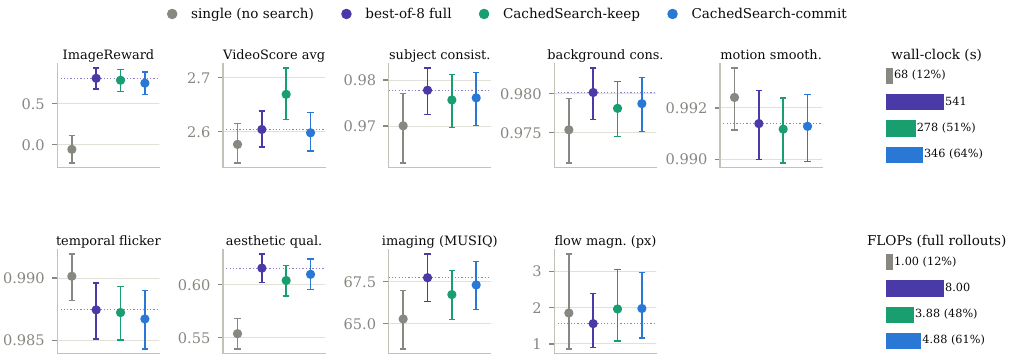}
  \caption{\textbf{Commit tracks best-of-$8$ across the full metric suite.}
  Four strategies on the held-out \VBenchMetric{} subset, with raw metric scales and
  $95\%$ prompt-bootstrap intervals. Dotted lines mark best-of-$8$; the last column
  reports wall-clock and transformer FLOPs.}
  \label{fig:multimetric-full}
\end{figure}

Three results matter. \textbf{(1) Commit is faster and no worse on metrics
it never optimized.} Its paired deltas from best-of-$8$ are statistically
zero across the quality suite\reportonly{: \VideoScoreMetric{} avg
$-0.006\,\ci{-.030,.017}$ and all \VBenchMetric{} dimensions within $0.6\%$ of
scale. Background consistency differs by only $-0.001$ on a
${\sim}0.98$ scale}. \ImageRewardMetric{} concedes
$0.06\ci{.04,.09}$ of the $0.87$ gain, giving $93.2\%$ capture, consistent
with Table~\ref{tab:main}.
\textbf{(2) Reference-free scorers favor keep despite its fidelity cost.}
Keep beats best-of-$8$ on \VideoScoreMetric{} by
$+0.065\ci{.028,.106}$, including
$+0.103$ on dynamics. It beats its full re-render by
$+0.036\ci{.004,.071}$ on the frame verifier, matching the $+0.037$ bias
in Section~\ref{sec:abl-temporal}. Yet keep has \LPIPSMetric{} $0.114$, aesthetic
$-0.012\,\ci{-.019,-.004}$, and imaging $-1.0\,\ci{-1.8,-.2}$. Its
temporal signature is mild at this $\tauc$ (Section~\ref{sec:abl-temporal}):
flicker improves by
$+0.0005\ci{.0002,.0009}$, while the $-0.7\%$ winner-level flow change is
within the noise of the $-3.1\%$ population effect.
\textbf{(3) Efficiency does not depend on the scorer.} The rollout costs
are unchanged under any metric. Commit spends $64\%$ of the wall-clock and
$61\%$ of the transformer FLOPs, while keep roughly halves cost but pays a
small frame-quality price that the reference-free scorers reward.

\ifreport
% SHARED APPENDIX BLOCK (both editions; single source). Section exp-vbench
% keeps one numerical replication sentence and points here.
% Numbers: code/paper_figs/vbench2_analysis.py (grep "[VB2]"); CIs:
% ci_numbers.json group vbench2.
% FINAL (2026-07-08): full coverage -- all 1,013 unique suite prompts at
% complete 8xfull+8xcached coverage (16,208 records). The prompts file has
% 1,013 lines (wc -l undercounts: no trailing newline). Data frozen.
\paragraph{\VBenchTwoMetric{} replication.}
Both suites so far share \VBenchMetric{}'s prompt style. As a structurally
different replication we repeated the full measurement on
\VBenchTwoMetric{}, which probes \emph{intrinsic faithfulness}
through compositional object interactions, physics, commonsense, and camera
control. It is substantially harder for current generators than
\VBenchMetric{}'s per-dimension lists. Every headline quantity is
statistically consistent with the \VBenchMetric{} suite
(Table~\ref{tab:vbench2}): median
$\rhosp = 0.881$ (one $8$-candidate lattice step below $0.905$; the
mean, $0.848\ci{.840,.857}$ vs.\ $0.859\ci{.850,.869}$, differs by
$-0.011$), top-1 $69\%\ci{66,72}$, and
$89.3\%\ci{87.9,90.7}$ capture at a reproduced $1.98\times$. The new
suite also stress-tests the corruption mechanism from the direction
that should hurt: its harder prompts produce significantly
\emph{tighter} candidate spreads (mean $0.49$ vs.\ $0.56$ official,
with a statistically reliable difference), and low spread is exactly where
ranking corrupts (Section~\ref{sec:abl-corruption}); the coupling itself
replicates (corr(spread, $\rhosp$) $= +0.29\ci{.23,.34}$,
$p = 8{\times}10^{-21}$). Yet the suite-level cost is
undetectable (capture moves by $-0.9$ points, statistically consistent with
zero) because the corruption that
the tighter spreads do cause stays concentrated on the prompts where a
wrong pick is cheapest: the self-limiting property is a mechanism, not
a prompt-set accident.

\begin{table}[!tb]
\caption{\textbf{The result holds on a third, harder suite.} Ranking,
regret, capture, and speedup across the gate, the \VBenchMetric{} suite, and
\VBenchTwoMetric{} at $\tauc=0.10$ with $8$ seeds. All
\VBenchTwoMetric{} headline quantities are statistically consistent with the
\VBenchMetric{} suite.}
\label{tab:vbench2}
\centering
\tabsize
\setlength{\tabcolsep}{3.5pt}
\resizebox{\linewidth}{!}{%
\begin{tabular}{lcccccccc}
\toprule
suite & $n$ & med $\rhosp$ $\uparrow$ & mean $\rhosp$ $\uparrow$ & p10 $\uparrow$ & top-1 $\uparrow$ & zero-regret $\uparrow$ & capture $\uparrow$ & speedup $\uparrow$ \\
\midrule
gate grid & $50$ & $0.905\ci{.83,.93}$ & $0.820\ci{.76,.87}$ & $0.61\ci{.36,.71}$ & $64\%\ci{50,76}$ & $64\%\ci{50,76}$ & $90.1\%\ci{82.6,95.9}$ & $1.97\ci{1.96,1.98}\times$ \\
official \VBenchMetric{} & $944$ & $0.905$ & $0.859\ci{.850,.869}$ & $0.69\ci{.64,.71}$ & $72\%\ci{69,75}$ & $72\%\ci{69,75}$ & $90.2\%\ci{88.5,91.8}$ & $1.95\ci{1.95,1.95}\times$ \\
\VBenchTwoMetric{} & $1{,}013$ & $0.881$ & $0.848\ci{.840,.857}$ & $0.69\ci{.64,.69}$ & $69\%\ci{66,72}$ & $69\%\ci{66,72}$ & $89.3\%\ci{87.9,90.7}$ & $1.98\ci{1.98,1.98}\times$ \\
\bottomrule
\end{tabular}}
\end{table}

\else
\paragraph{\VBenchTwoMetric{} replication.}
Both suites so far share \VBenchMetric{}'s prompt style. As a structurally
different replication we repeated the full measurement on
\VBenchTwoMetric{}, which probes \emph{intrinsic faithfulness}
through compositional object interactions, physics, commonsense, and camera
control. It is substantially harder for current generators than
\VBenchMetric{}'s per-dimension lists. Every headline quantity is
statistically consistent with the \VBenchMetric{} suite
(Table~\ref{tab:vbench2}): median
$\rhosp = 0.881$ (one $8$-candidate lattice step below $0.905$; the
mean, $0.848\ci{.840,.857}$ vs.\ $0.859\ci{.850,.869}$, differs by
$-0.011$), top-1 $69\%\ci{66,72}$, and
$89.3\%\ci{87.9,90.7}$ capture at a reproduced $1.98\times$. The new
suite also stress-tests the corruption mechanism from the direction
that should hurt: its harder prompts produce significantly
\emph{tighter} candidate spreads (mean $0.49$ vs.\ $0.56$ official,
with a statistically reliable difference), and low spread is exactly where
ranking corrupts (Section~\ref{sec:abl-corruption}); the coupling itself
replicates (corr(spread, $\rhosp$) $= +0.29\ci{.23,.34}$,
$p = 8{\times}10^{-21}$). Yet the suite-level cost is
undetectable (capture moves by $-0.9$ points, statistically consistent with
zero) because the corruption that
the tighter spreads do cause stays concentrated on the prompts where a
wrong pick is cheapest: the self-limiting property is a mechanism, not
a prompt-set accident.
\fi

% SHARED BLOCK (both editions; single source). Report edition: inlined in
% sections/ablations.tex (Section abl-adaptive, whole subsection). Conference edition: hosted in the appendix
% (app:more-abl) via sections/appendix.tex.
\subsection{Adaptive per-prompt thresholds}
\label{sec:abl-adaptive}

Spread predicts which prompts tolerate aggressive caching and can be
estimated from two cheap rollouts. We simulate this policy on the measured
$3$-$\tauc$ grid, whose arms share prompts and seeds. No generation is added.
\emph{Policy:} probe $K{=}2$ candidates at $\tauc = 0.20$. If their score
gap exceeds $t$, run the remaining $6$ at $\tauc = 0.20$; otherwise run all
$8$ at $\tauc = 0.05$, treating the probes as sunk cost. High spread signals
robust ranking (Section~\ref{sec:abl-corruption}). Sweeping $t$ traces the
frontier in Figure~\ref{fig:adaptive}. The simulation directly compares the
natural adaptive policy with fixed thresholds under identical prompts,
seeds, and measured rollouts.

The adaptive frontier never exceeds the fixed-$\tauc$ curve. At matched
$1.97\times$ speedup, it captures $89.7\%$ versus $90.1\%$ for
$\tauc = 0.10$. The paired difference is $\Delta = -0.4$ points and is
statistically indistinguishable from zero. Its conservative endpoint reaches
$93.6\%$ capture only at $1.36\times$, versus
$1.58\times$ for fixed $\tauc = 0.05$. The ratio-of-means estimator
(Section~\ref{sec:exp-setup}) and the reversed low-spread rule give the
same result. No tested adaptive variant clears the fixed frontier. The
matched difference is statistically indistinguishable from zero and
pointwise favors the fixed threshold.

The fixed curve leaves little headroom: capture changes only $5.3$ points
across a $1.5\times$ speedup range (Section~\ref{sec:abl-tau}). A two-sample
spread estimate is noisy, and fallback pays pure probe overhead. Adaptivity
needs a free signal, such as the wrapper's internal cache-error indicator
during the first rollout. Such a signal could spend no extra rollouts merely
to decide how to save rollouts. A global $\tauc = 0.10$ remains the stronger
default.

\subsection{Keep-draft versus recommit}
\label{app:temporal-detail}

Detail for Section~\ref{sec:abl-temporal}: the direct temporal metrics on the
$300$ materialized winner pairs, and the \LPIPSMetric{} comparison between the
keep/commit pairs.
% SHARED BLOCK (both editions; single source). Report edition: mid-paragraph
% of Section abl-temporal in sections/ablations.tex (between the frame-level reward
% result and the motion-dampening sentence). Conference edition: hosted in
% the appendix (app:temporal-detail).
The temporal question is answered by scoring all $300$ materialized winner
videos with \VBenchMetric{}-style temporal metrics (DINO subject consistency, CLIP
background consistency, warping-error motion smoothness, temporal
flickering, and mean optical-flow magnitude as a dynamics measure).
Static temporal quality is preserved under keep-draft: motion smoothness
is unchanged ($|\Delta| \leq 0.0003$), flickering is marginally
\emph{better} for keep ($+0.0007$ to $+0.0015$), background consistency is
equal, and subject consistency dips by at most $-0.005$.

% SHARED BLOCK (both editions; single source). Report edition: mid-paragraph
% of Section abl-temporal in sections/ablations.tex (after the motion-dampening
% sentence). Conference edition: hosted in the appendix (app:temporal-detail).
Perceptually, keep and commit are genuinely different videos rather than
near-duplicates: \LPIPSMetric{}(keep, commit) over the winner pairs rises
monotonically with cache aggressiveness (mean $0.122$ / $0.142$ / $0.170$
at $\tauc = 0.05$ / $0.10$ / $0.20$; $n{=}50$ each), consistent with the
motion-dampening account above: the cached trajectory diverges
perceptibly from its full-compute twin even when frame-level and most
temporal scores match.

Fidelity examples and flow maps appear in
Appendix~\ref{app:qualitative}.

% SHARED APPENDIX BLOCK (both editions; single source).
\subsection{Batching control}
\label{sec:abl-batch}

A natural objection is whether batching can replace caching. We benchmarked
candidate batch sizes $\{1,2,4,6,8\}$ at the study configuration on an
NVIDIA GH200 GPU. Batching yields \emph{no} throughput: full-compute
throughput \emph{drops} from
$0.92$ to $0.90$ videos/min going from batch $1$ to $2$, cached rollouts
gain only ${+}8\%$, and batch $8$ is infeasible. At $32$K tokens per
sample, Wan-1.3B \citep{wan2025wan} saturates the GPU's compute at
batch~$1$. Batching also breaks seed
determinism: the batched rollout of a seed is not bit-identical to its
single-sample rollout (max first-frame pixel deviation $0.6$), which would
invalidate the exactness of the commit step. Two consequences for honest
efficiency accounting: per-rollout wall-clock is the right cost unit for
this workload, and best-of-$N$ cannot close its $2\times$ cost gap through
batching. The saved FLOPs of cached exploration are real rather than an
artifact of underutilized hardware.

% SHARED BLOCK (both editions; single source). Report edition: inlined in
% sections/ablations.tex (Section abl-verifier, whole subsection). Conference edition: hosted in the appendix
% (app:more-abl) via sections/appendix.tex.
\subsection{Video-native verifier robustness}
\label{sec:abl-verifier}

Prior ranking and regret results use \ImageRewardMetric{}. We rescore them
with \VideoScoreMetric{} v1.1, a video-native judge. It evaluates
visual quality, temporal consistency,
dynamic degree, text alignment, and factual consistency over $24$ frames.
Unlike \ImageRewardMetric{}, it observes motion and can test whether the findings
survive a video-native verifier.
The data include all $300$ winner-pair videos
(Section~\ref{sec:abl-temporal}) and $23{,}542$ \VBenchMetric{}-suite rollouts
(Section~\ref{sec:exp-vbench}). They yield $n = 467$ prompts with validated,
complete $8{\times}$full${}+{}8{\times}$cached coverage. A score-consistency
check removes ambiguous prompt matches, and the retained subset reproduces
the \VBenchMetric{} suite's median $\rhosp$, top-1 agreement, and capture.

\textbf{Ranking preservation is weaker when \VideoScoreMetric{} barely separates
candidates.} Cached-vs-full rank correlation has median $\rhosp = 0.762$
(p10 $0.31$), $54\%$ top-1 agreement, $65.6\%$ gain capture, and $55\%$
zero regret. These exceed chance but trail \ImageRewardMetric{} on the same videos
($0.905$ / $73\%$ / $91\%$); dimension medians span $0.73$--$0.79$.
Two measurements explain this drop. First, candidate scores are nearly tied.
\VideoScoreMetric{}'s median within-prompt spread is only $0.083$ on
$[1,4]$, versus $0.55$ for \ImageRewardMetric{}. Its rankings thus occupy the
near-tied regime of Section~\ref{sec:abl-corruption}, where rank noise is
large but mistaken picks are cheap. Second, the verifiers agree less with
each other than cached agrees with full: median
$\rhosp(\text{IR}, \text{VS})$ is $+0.24$ within prompts and $+0.13$ over
$7{,}701$ pooled videos, versus $0.905$ / $0.762$ for cached-to-full.
The perturbation from caching is therefore smaller than the standing
disagreement between the judges.
Verifier choice changes the ranking more than cached execution does. This
separates approximation error from ordinary evaluator disagreement.

\textbf{Caching adds at most marginal cross-verifier loss.} Under
full-compute \VideoScoreMetric{}, the cached-exploration IR winner captures $18.6\%$
of attainable gain, versus $21.2\%$ for the full-compute IR winner (mean
VS-regret $0.245$ vs.\ $0.230$; random baseline $0.263$; $n = 467$).
The paired difference is $-2.5$ capture points and is statistically
consistent with zero. Regret increases by $+0.015$ on
\VideoScoreMetric{}'s $[1,4]$
scale. Both effects are small beside the verifier choice itself:
full-compute \ImageRewardMetric{} selection captures only $21.2\%$ of
\VideoScoreMetric{} gain. Thus \VideoScoreMetric{} judges
\ImageRewardMetric{}'s cached and full-compute picks almost identically.
Most cross-verifier loss comes from choosing \ImageRewardMetric{}, not
from caching its search.
This is the decision-relevant comparison because delivery uses the selected
seed, not the full candidate ranking.

\textbf{\VideoScoreMetric{} prefers keep and misses motion dampening.} Across the
$50$ winner pairs per $\tauc$, keep wins on every dimension and threshold.
The average margin grows $+0.036$ / $+0.058$ / $+0.129$ at
$\tauc = 0.05/0.10/0.20$ (Wilcoxon $p \le 0.004$ each). At
$\tauc = 0.20$, dynamic degree favors keep by $+0.176$ on $92\%$ of pairs
($p < 0.001$), although optical flow is $8\%$ lower
(Section~\ref{sec:abl-temporal}). Appendix~\ref{app:verifier-bias} therefore
describes a broader learned-judge bias: the video-native judge also rewards
the dampened output, and the bias grows with aggressiveness. Learned judges
alone cannot audit lossy acceleration.
Use direct temporal metrics such as optical flow and \LPIPSMetric{}. Recommit remains
the safe delivery default because it regenerates the chosen seed at full
compute and avoids this bias by construction.

% SHARED APPENDIX BLOCK (both editions; single source).
% Measured source: results/b1_gate_{steps25,steps100,res720}/scores_shard*.jsonl
% (re-derived by experiments/gate5_analyze.py + paper_figs/scaling_analysis.py;
% CIs by paper_figs/bootstrap_ci.py; trend fits by scaling_analysis.py).
\subsection{Schedule length and resolution}
\label{sec:abl-schedule}

\begin{table}[!tb]
\caption{\textbf{Longer schedules expose more redundancy; resolution does
not move the operating point.} Wan2.1-1.3B \citep{wan2025wan} at
$\tauc=0.10$, with
$50$ prompts and $8$ seeds per row. \emph{skip} is reused denoising
steps; shading marks the default.}
\label{tab:schedule}
\centering
\small
\setlength{\tabcolsep}{2pt}
\begin{tabular}{lccccccc}
\toprule
config & med $\rhosp$ $\uparrow$ & p10 $\uparrow$ & top-1 $\uparrow$ & regret $\downarrow$ & capture $\uparrow$ & skip & speedup $\uparrow$ \\
\midrule
$25$ steps & $0.964\ci{.92,.98}$ & $0.85\ci{.77,.91}$ & $86\%\ci{76,94}$ & $0.019\ci{.003,.044}$ & $95.8\%\ci{90.9,99.3}$ & $32\%$ & $1.41\ci{1.40,1.42}\times$ \\
\rowcolor{blue!8}
$50$ steps (default) & $0.905\ci{.83,.93}$ & $0.61\ci{.35,.72}$ & $64\%\ci{50,76}$ & $0.040\ci{.019,.065}$ & $90.1\%\ci{82.6,95.9}$ & $51\%$ & $1.97\ci{1.96,1.98}\times$ \\
$100$ steps & $0.798\ci{.75,.81}$ & $0.46\ci{.25,.58}$ & $62\%\ci{48,76}$ & $0.126\ci{.051,.220}$ & $82.5\%\ci{71.7,91.6}$ & $66\%$ & $2.80\ci{2.79,2.81}\times$ \\
$720{\times}1280$, $50$ steps & $0.881\ci{.83,.90}$ & $0.69\ci{.52,.78}$ & $66\%\ci{52,78}$ & $0.059\ci{.022,.104}$ & $90.8\%\ci{84.5,96.1}$ & $53\%$ & $2.05\ci{2.04,2.06}\times$ \\
\bottomrule
\end{tabular}
\end{table}

Prior results use the model-card schedule ($50$ steps) and resolution
($480{\times}832$). Table~\ref{tab:schedule} varies both on
Wan2.1-1.3B at fixed $\tauc = 0.10$. Schedule length is a third axis of
effective aggressiveness, alongside threshold (Section~\ref{sec:abl-tau})
and architecture (Section~\ref{sec:abl-cog}). Finer schedules take smaller
steps, so more updates fall below the same drift threshold. From $25$ to
$100$ steps, skip rises $32\% \to 51\% \to 66\%$ and speedup rises
$1.41\times \to 1.97\times \to 2.80\times$. Ranking fidelity falls from
$0.964 \to 0.905 \to 0.798$ median $\rhosp$ and
$95.8\% \to 90.1\% \to 82.5\%$ capture.

Figure~\ref{fig:schedule-scaling} gives log-linear trends:
skip\% ${\approx} -46 + 24.6\ln T$ and capture
${\approx} 127 - 9.5\ln T$. We call these monotone patterns trends, not laws,
because they use only three schedule lengths. They match the U-shaped
redundancy profile in Section~\ref{sec:related-caching}: finer schedules
spend more steps in the flat middle of the trajectory. The skip fraction is
nearly deterministic given $T$, while capture retains prompt-level
variation. All three points
also lie within $2.0$ points of the Wan-1.3B frontier in
Figure~\ref{fig:scaling-master}, although the fit never saw this axis.
Changing schedule therefore moves the model along its frontier.

Long schedules offer the largest relative and absolute savings, exactly
where full-compute search is most expensive. Short schedules leave little
to skip: $25$ steps provide only $1.41\times$ at
$\tauc = 0.10$, but retain near-perfect fidelity. This is the same boundary
seen in few-step distilled samplers. Resolution is different.
At $720{\times}1280$ ($2.3\times$ the pixels), the operating point remains
near the default: $53\%$ skip, $2.05\times$ speedup, median
$\rhosp = 0.881$, $90.8\%$ capture, and $66\%$ top-1. This is the open
contrast in Figure~\ref{fig:schedule-scaling}. Recalibrate
$\tauc$ when the schedule changes, but retain it across resolution changes.

\begin{figure}[!tb]
  \centering
  \includegraphics[width=0.72\linewidth]{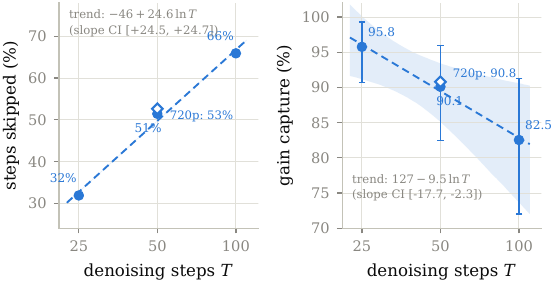}
  \caption{\textbf{Schedule length, not pixel count, controls redundancy.}
  Skip fraction (left) and gain capture (right) vs.\ denoising steps at
  $\tauc=0.10$ and $n=50$. Dashed lines are log-linear trends with
  bootstrap bands; open diamonds show the $720$p contrast.}
  \label{fig:schedule-scaling}
\end{figure}

% SHARED APPENDIX BLOCK (both editions; single source).
\subsection{Composition with pruned search}
\label{sec:abl-stack}

Pruning-based search methods reduce how many candidates pay full price;
\cachedsearch{} reduces what each candidate costs. The two levers are
orthogonal, so their savings should multiply: they prune
\emph{candidates} mid-trajectory, we cheapen every \emph{rollout}.
We test whether the two compose by running all $8$ candidates cached
($\tauc = 0.10$) only to step $20$ of $50$, scoring a $4$-frame preview
decoded from the sampler's internal $x_0$ estimate (a ${\sim}21\times$
cheaper probe than a full decode), continuing only the top-$4$ to
completion, and committing the winner ($50$ prompts, same grid as the
gate). The composition works as multiplication predicts: exploration cost
falls to $175$\,s per prompt, $3.11\times$ below full-compute
best-of-$8$ ($547$\,s) and $1.59\times$ below cached-only
exploration ($278$\,s), i.e.\ the measured $1.96\times$ (caching) and
$1.59\times$ (pruning) factors multiply to the observed $3.1\times$;
delivering the recommitted winner end-to-end (preview, decode, scoring,
and the $68$\,s recommit included) costs a measured $256$\,s, versus
$346$\,s for \cachedsearch{}-commit without pruning (both plotted on the
Pareto plane of Figure~\ref{fig:pareto}), while gain capture falls only from
$90.1\%$ (cached, no pruning) to $88.6\%$, with median regret still
exactly zero and the true-best candidate surviving the prune on
$86\%\ci{76,94}$ of prompts. Cached exploration is
thus a \emph{multiplier} on pruning-based search rather than an
alternative to it, substantiating the composability claim of
Section~\ref{sec:intro}.

\subsection{Scaling analysis}
\label{app:scaling-detail}

This subsection expands the scaling summary and master operating-points
figure (Section~\ref{sec:ablations}, Figure~\ref{fig:scaling-master}). It
covers frontier unification, width, model scale, and cross-architecture
calibration. Section~\ref{sec:abl-schedule} gives schedule and resolution
trends. The model grid contains \wan{} \citep{wan2025wan},
Wan2.1-T2V-14B \citep{wan2025wan}, Wan2.2-TI2V-5B
\citep{wan2025wan22,wan2025wan}, CogVideoX-5B
\citep{yang2024cogvideox}, HunyuanVideo-13B
\citep{kong2024hunyuanvideo}, and LTX-Video-2B
\citep{hacohen2024ltxvideo}.

% SHARED BLOCK (both editions; single source): the operating-plane /
% frontier-unification analysis of Section abl-scalingfigs (fig:scaling-master).
% Report edition: inlined in sections/ablations.tex. Conference edition:
% hosted in the appendix (app:scaling-detail); the conference main text
% carries the one-paragraph summary next to the figure.
% Numbers: code/paper_figs/scaling_analysis.py (grep "[PROSE]"), recomputed
% from the raw jsonl records of every grid.
\paragraph{One frontier per backbone.} Figure~\ref{fig:scaling-master}
collects every measured operating point of the paper on a single
capture-vs-speedup plane; its left panel isolates the Wan2.1-1.3B
\citep{wan2025wan}
backbone. Section~\ref{sec:abl-baselines} fits its frontier on six points
that vary the reuse rule ($\tauc$) or family
(PAB \citep{zhao2024pab}, CFG-Cache \citep{lv2024fastercache},
TeaCache \citep{liu2024teacache}):
capture(\%) $= 104.2 - 19.1\ln(\text{speedup})$, $R^2 = 0.92$. All other
points are evaluation. Four unseen configurations, the $25$- and
$100$-step schedules, $720$p, and an independent default rerun, land
within $2.0$ points (held-out $R^2 = 0.90$). Refitting all ten Wan-1.3B
points barely changes the curve: $104.2 - 19.9\ln x$, $R^2 = 0.94$.
Within one backbone, candidate speedup predicts selection fidelity across
thresholds, caching families, schedules, and resolutions.
This result makes reuse amount a sufficient statistic for fidelity within
the backbone. The mechanism used to obtain that reuse does not create a new
trade-off. Instead, each method moves the operating point along the same
curve.

This relation fails across backbones. At fixed $\tauc = 0.10$, the Wan
frontier over-predicts the six-model captures by $3$--$18$ points
($R^2 = -0.86$). A pooled fit over all $18$ non-degenerate points reaches
only $R^2 = 0.59$. CogVideoX \citep{yang2024cogvideox} follows its own
steeper frontier,
$114.8 - 52.2\ln x$ ($R^2 = 0.98$), with $2.7\times$ the Wan slope.
It is more caching-sensitive at every measured speedup, despite sharing the
same log-linear form.
Wan2.1-14B is the flattest backbone measured. Its panel~(B) dial spans
$\tauc = 0.005$--$0.20$, with a degenerate zero-skip control at $0.005$.
The four non-degenerate arms fit $93.3 - 9.2\ln x$ ($R^2 = 0.93$).
Skip fraction gives the same result without latency: it explains
$R^2 = 0.89$ within Wan-1.3B but $0.40$ across models. We therefore report
per-family frontiers rather than one universal law. Reuse determines
fidelity within a family, while architecture sets the trade-off's level and
slope. A new backbone therefore needs its own calibration. The $25$-prompt
pilot of Section~\ref{sec:abl-cog} locates that frontier before a full search
campaign.

% SHARED BLOCK (both editions; single source): part of Section abl-scalingfigs.
% Report edition: inlined in sections/ablations.tex at its original spot.
% Conference edition: hosted in the appendix (app:scaling-detail).
\begin{figure}[t]
  \centering
  \includegraphics[width=0.62\linewidth]{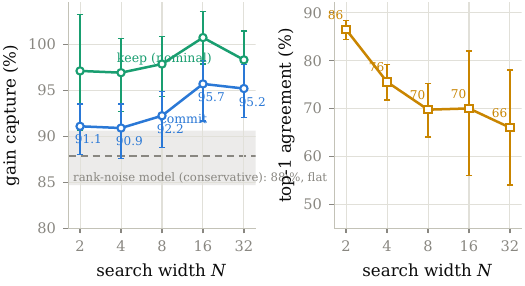}
  \caption{\textbf{Delivered capture rises toward $95\%$ with search width
  while top-1 falls.} Width scaling at $\tauc=0.10$ on $50$ prompts;
  open marks are seed-subset simulations through $N=16$, and $N=32$
  is measured directly. The gray band is the rank-noise prediction.}
  \label{fig:width-scaling}
\end{figure}

% SHARED BLOCK (both editions; single source): part of Section abl-scalingfigs.
% Report edition: inlined in sections/ablations.tex at its original spot.
% Conference edition: hosted in the appendix (app:scaling-detail).
\paragraph{Width.} Figure~\ref{fig:width-scaling} re-simulates all four
strategies over every $\binom{16}{N}$ seed subset of the extended
$16$-seed grid, and adds an $N{=}32$ point from a further extension to
$32$ seeds per prompt. The two panels move in opposite directions:
exact top-1 agreement decays from $86\%$ at $N{=}2$ to a
${\sim}65$--$70\%$ plateau at $N \geq 8$ (a larger pool gives more ways
to mis-rank the exact best), yet commit capture \emph{rises} from
$90.9\%$ at $N{=}4$ to $95.7\%$ at $N{=}16$ and holds at $95.2\%$ at
$N{=}32$ (the surface-\emph{some}-near-best mechanism
of Section~\ref{sec:exp-main}, saturating near $95\%$), while the commit cost
ratio simultaneously falls toward the pure exploration ratio
($75.8\% \to 63.3\% \to 57.1\% \to 53.9\%$ of best-of-$N$). The shape is
a rise to saturation: the $N{=}4 \to 16$ increase is $+4.8$ points
(Section~\ref{app:theory-width}), while
$N{=}16 \to 32$ moves $-0.5$ points (within noise), and a saturating fit
capture$(N) = c_\infty - \beta/N$ over the five widths puts the ceiling
at $c_\infty = 94.8\%$ ($R^2 = 0.62$, a consistency check from five
points rather than a fitted law). The flat spread-weighted
copula envelope overlaid in Figure~\ref{fig:width-scaling} ($88\%$;
Appendix~\ref{app:theory}, Remark~\ref{rem:conservative}) cannot produce
this rise. Section~\ref{app:theory-width} traces the disagreement to the
Gaussian-marginal idealization. Top-1
agreement is therefore the wrong scaling metric for cached search;
delivered capture improves exactly in the wide-search regime where the
method saves the most.

% SHARED BLOCK (both editions; single source): the model-scaling figure (f9).
% Report edition: inlined in sections/ablations.tex (Section abl-scalingfigs).
% Conference edition: hosted in the appendix (app:scaling-detail).
% Generated by code/paper_figs/make_figs2.py (f9_models).
\begin{figure}[!tb]
  \centering
  \includegraphics[width=0.97\linewidth]{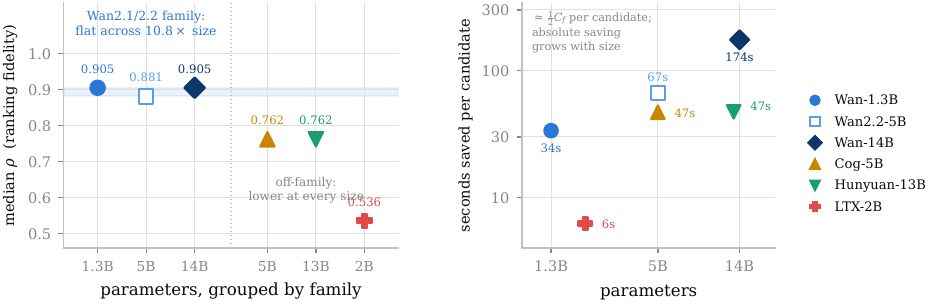}
  \caption{\textbf{Fidelity tracks architecture family, while absolute
  savings track model cost.} Median $\rhosp$ (A) and seconds saved per
  candidate (B) at $\tauc=0.10$, with $8$ seeds and $50$ prompts per
  model. Table~\ref{tab:scale} gives full values.}
  \label{fig:model-scaling}
\end{figure}

% SHARED BLOCK (both editions; single source): the model-scale paragraph of
% Section abl-scalingfigs. Report edition: inlined in sections/ablations.tex.
% Conference edition: hosted in the appendix (app:scaling-detail).
\paragraph{Model scale.} Figure~\ref{fig:model-scaling} plots median
$\rhosp$ by model, grouped by family, and the absolute per-candidate
saving against parameter count, at fixed $\tauc = 0.10$ for all six models;
capture and top-1 follow the same family split (Table~\ref{tab:scale}). Within the
Wan family \citep{wan2025wan}, a $10.8\times$ parameter increase leaves
median $\rhosp$
exactly unchanged ($0.905$) and moves capture only mildly
($90.1\% \to 87.5\%$; top-1 $64\% \to 58\%$, Table~\ref{tab:scale}),
with the cross-generation Wan2.2-5B \citep{wan2025wan22} point on the same
plateau ($0.881$
/ $86.0\%$); every off-family point sits below the Wan band at
\emph{every} size: LTX-2B \citep{hacohen2024ltxvideo} at the bottom
($0.536$ / $67.6\%$), CogVideoX-5B \citep{yang2024cogvideox} and
Hunyuan-13B \citep{kong2024hunyuanvideo} in between ($0.762$ each, capture
$75.2\%$ / $79.9\%$). Vertical position in the fidelity panel tracks
family, not scale\reportonly{, which is the graphical form of
Table~\ref{tab:scale}'s conclusion and direct support for hypothesis
H2 of Appendix~\ref{app:scaling}: cacheable redundancy and rank
fidelity are properties of the (architecture, $\tauc$) pair, not of
parameter count}. The savings panel prices the transfer in absolute
terms: at $\tauc = 0.10$ every backbone saves ${\approx}$half its
rollout cost per candidate ($49$--$54\%$ of $C_f$; the over-driven LTX
$62\%$), so the \emph{absolute} saving grows linearly with model
cost: $6$\,s per candidate on LTX-2B ($10$\,s rollouts), $34$\,s on
Wan-1.3B, $175$\,s on Wan-14B ($341$\,s rollouts). Parameter count
does not set the fidelity, but it does set the stakes: cached
exploration saves the most compute exactly where full-compute search
is most expensive.

% tau-calibration float retired (duplicated Fig 16 panel B, fig:scaling-master).
% SHARED BLOCK (both editions; single source): part of Section abl-scalingfigs.
% Report edition: inlined in sections/ablations.tex at its original spot.
% Conference edition: hosted in the appendix (app:scaling-detail).
\paragraph{Cross-architecture calibration.}
Figure~\ref{fig:scaling-master}(B) overlays the measured
capture-vs-speedup curves of all six backbones. All are monotone and
smooth, but they are \emph{shifted}: CogVideoX's
\citep{yang2024cogvideox} $\tauc = 0.10$ point
($75.2\%$ at $2.06\times$) lies well below Wan-1.3B's
\citep{wan2025wan} \emph{most
aggressive} setting, while its $\tauc = 0.05$ point ($85.9\%$ at
$1.78\times$) climbs back to the edge of Wan's operating band: the
same nominal threshold buys different effective aggressiveness on
different backbones. The Wan2.1-14B curve spans the full dial,
$\tauc = 0.005$ to $0.20$: at $\tauc = 0.005$ the drift indicator never
crosses the threshold, so \emph{zero} steps are skipped and the cached
arm reproduces full compute bit-exactly ($1.01\times$, capture $100\%$,
a useful end-to-end determinism control, as
Section~\ref{sec:abl-baselines}'s no-caching control was for the
harness), while $\tauc = 0.02/0.05/0.10/0.20$ buy
$1.21/1.59/2.05/2.52\times$ at $92.1/88.0/87.5/84.7\%$ capture: the
flattest curve of any backbone (its own fit $93.3 - 9.2\ln x$ over the
four non-degenerate arms, $R^2 = 0.93$, vs.\ slopes of $-12.5$ on
Wan-1.3B, $-16.1$ on Wan2.2-5B \citep{wan2025wan22}, $-16.7$ on
Hunyuan \citep{kong2024hunyuanvideo}, $-28.3$ on
LTX \citep{hacohen2024ltxvideo}, and
$-52.2$ on CogVideoX). The most expensive model in the set is also the
most caching-tolerant. LTX-Video traces the opposite extreme: its whole
curve sits below the $85\%$ band at every measured $\tauc$ (the
boundary row of Table~\ref{tab:taucal}). The figure turns the
calibration rule of Section~\ref{sec:abl-cog} into a procedure: pick a target
capture (horizontal line), and sweep $\tauc$ on a small pilot until the
model's curve crosses it; the abscissa then reports the speedup that
fidelity level costs on that architecture.

% SHARED BLOCK (both editions; single source): part of Section abl-scalingfigs.
% Report edition: inlined in sections/ablations.tex after x_par_calibscaling.
% Conference edition: hosted in the appendix (app:scaling-detail).
% Numbers: code/paper_figs/taucal_analysis.py (grep "[TAUCAL]"); CIs:
% ci_numbers.json groups {gate,cog,hunyuan,ltx,wan22,wan14b}_tau=*.
% WAVE-4 LIVE: the Wan2.2 tau=0.05/0.20 and Wan-14B tau=0.05/0.20 cells are
% refreshed from the corrected rerun dirs (wan22_tau{005,020}fix,
% wan14b_tau{05,20}fix) -- see NOTES-taucal-agent.md.
\paragraph{Per-model calibrated operating points.}
Table~\ref{tab:taucal} gives a recommended threshold $\tau^*$ per
backbone. We select the most aggressive measured $\tauc$ with at least
$85\%$ gain capture. Every Wan backbone
\citep{wan2025wan,wan2025wan22} qualifies at $\tauc = 0.10$.
Wan-1.3B also qualifies at $\tauc = 0.20$, with $88.3\%$ capture at
$2.41\times$. The 5B and 14B models miss at $\tauc = 0.20$ by only
$1$--$2$ points ($83.1\%$/$84.7\%$). The 14B dial is flat, moving
$92.1 \to 84.7\%$ across $\tauc = 0.02 \to 0.20$.\reportonly{ (We
nonetheless keep $\tauc = 0.10$ as the paper-wide default on Wan-1.3B:
one notch of conservatism buys a healthier tail (p10 $\rhosp$ $0.61$
vs.\ $0.47$) and keep-draft safety, Section~\ref{sec:abl-temporal}; $\tau^*$
marks the most aggressive \emph{validated} setting for commit-mode
exploration.)} CogVideoX-5B \citep{yang2024cogvideox} and
HunyuanVideo-13B \citep{kong2024hunyuanvideo} require
$\tau^* = 0.05$, where both reach ${\sim}1.8\times$ and $85$--$86\%$
capture. LTX-Video-2B \citep{hacohen2024ltxvideo} is the boundary: no
measured $\tauc$ qualifies.
At $\tauc = 0.02$, capture improves by $12$ points
($67.6\% \to 79.6\%$) at $1.71\times$. Yet LTX remains $5$--$6$ points
below CogVideoX and Hunyuan at ${\sim}44\%$ skip. Its frontier,
$94.5 - 28.3\ln x$ ($R^2 = 0.99$), reaches $85\%$ only near
$1.4\times$. Calibration helps, but LTX has a lower frontier.\reportonly{
Two to three cached-only arms scored against existing full-compute references
suffice to place a backbone's curve against the target band.}

\begin{table}[!tb]
\caption{\textbf{Calibrated thresholds recover at least $85\%$ capture on
five of six models
\citep{wan2025wan,wan2025wan22,yang2024cogvideox,kong2024hunyuanvideo,hacohen2024ltxvideo}.}
$\tau^*$ is the most aggressive qualifying measured
threshold; $\dagger$ marks the LTX boundary.
The left block gives each
sweep and the right its selected operating point.}
\label{tab:taucal}
\centering
\tabsize
\setlength{\tabcolsep}{3pt}
\resizebox{\linewidth}{!}{%
\begin{tabular}{lcccccccccc}
\toprule
 & \multicolumn{5}{c}{capture (\%) $\uparrow$ at measured $\tauc =$} & \multicolumn{5}{c}{operating point at $\tau^*$} \\
\cmidrule(lr){2-6} \cmidrule(lr){7-11}
model & $.005$ & $.02$ & $.05$ & $.10$ & $.20$ & $\tau^*$ & med $\rhosp$ $\uparrow$ & capture $\uparrow$ & skip & speedup $\uparrow$ \\
\midrule
\logocell{wan13b} & - & - & $93.6$ & $90.1$ & $\bm{88.3}$ & $0.20$ & $0.857\ci{.81,.91}$ & $88.3\%\ci{82.4,93.5}$ & $60\%$ & $2.41\ci{2.37,2.45}\times$ \\
\logocell{wan22} & - & - & $89.3$ & $\bm{86.0}$ & $83.1$ & $0.10$ & $0.881\ci{.82,.91}$ & $86.0\%\ci{78.9,92.3}$ & $55\%$ & $2.05\ci{2.04,2.06}\times$ \\
\logocell{wan14b} & $100$ & $92.1$ & $88.0$ & $\bm{87.5}$ & $84.7$ & $0.10$ & $0.905\ci{.85,.93}$ & $87.5\%\ci{80.2,93.8}$ & $52\%$ & $2.05\ci{2.04,2.05}\times$ \\
\logocell{cog} & - & - & $\bm{85.9}$ & $75.2$ & $64.7$ & $0.05$ & $0.810\ci{.71,.85}$ & $85.9\%\ci{76.1,93.3}$ & $44\%$ & $1.78\ci{1.74,1.82}\times$ \\
\logocell{hunyuan} & - & - & $\bm{85.1}$ & $79.9$ & $79.7$ & $0.05$ & $0.810\ci{.74,.89}$ & $85.1\%\ci{75.8,92.9}$ & $46\%$ & $1.77\ci{1.76,1.78}\times$ \\
\logocell{ltx} & - & $79.6$ & $71.6$ & $67.6$ & - & none$^\dagger$ & $0.786\ci{.69,.86}$ & $79.6\%\ci{68.6,89.2}$ & $43\%$ & $1.71\ci{1.70,1.72}\times$ \\
\bottomrule
\end{tabular}}
\end{table}

% Practitioner's operating manual (tab:manual). Single-sourced: report hosts it
% in the Conclusion (via \reportonly), the conference hosts it in the appendix
% (\ifreport\else block). Point estimates only (\ci hidden by \showcifalse).
\begin{table}[!tb]
\caption{\textbf{Deployment settings for quality, speed, motion, width,
and new backbones.} Measured on \wan{} \citep{wan2025wan} at $N=8$ unless
noted; cost is
relative to full best-of-$N$. Keep capture is nominal and commit restores
full-compute dynamics.}
\label{tab:manual}
\centering
\small
\setlength{\tabcolsep}{3.5pt}
\begin{tabular}{lllll}
\toprule
deployment goal & mode & $\tauc$ & expected capture $\uparrow$ & cost $\downarrow$ \\
\midrule
guaranteed quality (default) & commit & $0.10$ & $94.7\%\ci{90.4,97.6}$ & $63\%$ \\
maximum speed (drafts) & keep & $0.10$ & $99.6\%\ci{95.4,104.3}$ nominal & $51\%$ \\
motion-critical prompts & commit & $0.05$ & $93.6\%\ci{88.6,97.5}$ & $76\%$ \\
widest search, fixed budget & commit & $0.10$, $N{=}16$ & $95.7\%\ci{91.5,98.3}$ & $57\%$ \\
new architecture or schedule & commit & calibrate & see Tab.~\ref{tab:taucal}\reportonly{, Fig.~\ref{fig:scaling-master}, Tab.~\ref{tab:schedule}} & pilot calibration \\
\bottomrule
\end{tabular}
\end{table}

% SHARED APPENDIX BLOCK (both editions; single source).
% Measured source: results/b1_gate_{vqares,vqares_t25,lowres352}/scores_shard*.jsonl
% and results/b1_verifiers_vqares/combined.jsonl, re-derived with the
% gate5_analyze.py estimator (seeds 0-7, complete-coverage prompts,
% eq:capture mean-of-ratios); CIs are 2000-resample bootstrap percentiles.
\subsection{Reduced-resolution exploration and a second verifier}
\label{app:lowres-vqa}

Section~\ref{sec:abl-cheapexplore} argues that cheap exploration is useful
only when it preserves the delivered sample's trajectory. Truncation tests
that claim by shortening the schedule. Reducing the resolution tests it more
sharply, because the latent tensor itself changes shape, so a seed no longer
indexes a perturbed version of the same sample. It indexes a different one.

We regenerated the $50$-prompt gate grid at $352{\times}608$, roughly half
the pixels of the $480{\times}832$ default, and ranked those drafts against
the same full-compute references used elsewhere. Table~\ref{tab:lowresexplore}
places the arm beside the two strategies of Table~\ref{tab:cheapexplore}.
Reduced-resolution exploration is the cheapest of the three at
$2.33\times$, and it is the only one that recovers nothing: median
$\rhosp = 0.060$, $18\%$ top-1 agreement, and $-0.8\%$ capture, which is
statistically consistent with zero. Selecting on half-resolution drafts is
no better than selecting at random. The regenerated full-compute references
are bit-identical to the published ones, so this row shares the reference
grid of Table~\ref{tab:cheapexplore} exactly.

This does not contradict Section~\ref{sec:abl-schedule}, where caching at
$720{\times}1280$ holds its operating point. There the drafts and the
references share a resolution, and caching perturbs a trajectory that both
arms follow. Here the draft is generated at one resolution and the delivered
video at another, so the two arms follow different trajectories from the
same seed. Resolution is safe to change for the whole search and unsafe to
change between exploration and commit.

% ======================================================================
% tab:lowresexplore - the three cheap-exploration strategies on one grid.
% Rows 1-2 restate tab:cheapexplore; row 3 is the reduced-resolution arm.
% ======================================================================
\begin{table}[!tb]
\caption{\textbf{Reduced-resolution exploration recovers nothing.} The three
cheap-exploration strategies against $480{\times}832$, $T{=}50$ full-compute
references on $50$ prompts and $8$ seeds. The first two rows restate
Table~\ref{tab:cheapexplore}; shading marks the default. Reduced resolution
is the cheapest arm and the only one whose capture is statistically
consistent with zero.}
\label{tab:lowresexplore}
\centering
\ifreport\tabsize\else\small\fi
\setlength{\tabcolsep}{3pt}
\resizebox{\linewidth}{!}{%
\begin{tabular}{lcccccc}
\toprule
exploration arm & speedup $\uparrow$ & med $\rhosp$ $\uparrow$ &
p10 $\rhosp$ $\uparrow$ & top-1 $\uparrow$ & regret $\downarrow$ &
capture $\uparrow$ \\
\midrule
\rowcolor{blue!8}
caching ($\tauc{=}0.10$) &
$1.97\ci{1.96,1.98}\times$ &
$\bm{0.905}\ci{.83,.93}$ &
$\bm{0.614}\ci{.36,.71}$ &
$\bm{64\%}\ci{50,78}$ &
$\bm{0.039}\ci{.019,.064}$ &
$\bm{90.1\%}\ci{82.5,95.9}$ \\
step-truncation ($T{=}25$) &
$1.98\ci{1.95,2.02}\times$ &
$0.512\ci{.44,.69}$ &
$0.024\ci{-.04,.31}$ &
$44\%\ci{30,58}$ &
$0.226\ci{.088,.398}$ &
$72.6\%\ci{59.2,84.3}$ \\
reduced resolution ($352{\times}608$) &
$2.33\ci{2.33,2.33}\times$ &
$0.060\ci{-.18,.21}$ &
$-0.410\ci{-.51,-.24}$ &
$18\%\ci{8,28}$ &
$0.627\ci{.446,.827}$ &
$-0.8\%\ci{-27.2,23.5}$ \\
\bottomrule
\end{tabular}}
\end{table}

\paragraph{A second verifier on the same rollouts.}
Section~\ref{sec:abl-verifier} rescores the study with
\VideoScoreMetric{}. We repeat that test with \VQAScoreMetric{}, an
image-text alignment score, on the regenerated gate grid
($n = 49$ prompts with complete coverage after score-consistency filtering).
Ranking preservation under caching is weaker than for
\ImageRewardMetric{} but far from absent: median $\rhosp = 0.810$
(p10 $0.495$), $57\%$ top-1 agreement, and $80.7\%$ capture, against
$0.905$ / $0.581$ / $65\%$ / $90.2\%$ for \ImageRewardMetric{} on the same
videos.

The comparison that matters is not between those two columns. On
full-compute rollouts, where no caching is involved, the two verifiers rank
the same eight candidates at median $\rhosp = 0.548$. Both verifiers
therefore agree with their own cached rankings ($0.905$ and $0.810$)
substantially better than they agree with each other. The perturbation
caching introduces is smaller than the standing disagreement between two
reasonable definitions of quality, which is the same conclusion
Section~\ref{sec:abl-verifier} reaches from \VideoScoreMetric{}.

Regenerating the default arm also reproduces the published operating point
exactly ($0.905$ median $\rhosp$, $64\%$ top-1, $0.039$ regret, $90.1\%$
capture), as seed-deterministic generation on identical hardware should.

% ======================================================================
\section{Per-category analysis}
\label{app:categories}

We test whether the score-spread pattern in Section~\ref{sec:abl-corruption}
has semantic structure. We assign each of the $50$ prompts to one keyword
category: \emph{humans} ($n{=}18$),
\emph{animals} ($n{=}9$), \emph{nature} ($n{=}10$), \emph{objects}
($n{=}9$), or \emph{stylized} ($n{=}4$). Humans center people; nature covers
landscapes, weather, and natural phenomena. Objects cover vehicles,
machines, and manufactured items. Stylized prompts are non-photorealistic.
Table~\ref{tab:categories} reports ranking fidelity and regret at
$\tauc = 0.10$, $N = 8$.

\begin{table}[h]
\caption{\textbf{Ranking fidelity varies modestly across prompt
categories.} Keyword buckets over $50$ prompts at $\tauc=0.10$ and $N=8$.
\emph{Spread} is mean within-prompt score SD; regret uses reward units.
The stylized bucket has only $n=4$.}
\label{tab:categories}
\centering
\small
\setlength{\tabcolsep}{7pt}
\begin{tabular}{lccccc}
\toprule
category & $n$ & median $\rhosp$ $\uparrow$ & mean regret $\downarrow$ & top-1 $\uparrow$ & mean spread \\
\midrule
humans   & 18 & $0.881\ci{.82,.95}$ & $0.045\ci{.018,.078}$ & $56\%\ci{33,78}$ & $0.53\ci{.40,.68}$ \\
animals  & 9  & $0.833\ci{.61,.93}$ & $0.008\ci{.000,.022}$ & $78\%\ci{44,100}$ & $0.51\ci{.31,.73}$ \\
nature   & 10 & $0.929\ci{.89,.96}$ & $0.005\ci{.000,.014}$ & $80\%\ci{50,100}$ & $0.53\ci{.34,.73}$ \\
objects  & 9  & $0.857\ci{.35,.96}$ & $0.066\ci{.005,.160}$ & $56\%\ci{22,89}$ & $0.43\ci{.29,.57}$ \\
stylized & 4  & $0.917\ci{.57,1.00}$ & $0.112\ci{.000,.227}$ & $50\%\ci{0,100}$ & $0.54\ci{.30,.76}$ \\
\bottomrule
\end{tabular}
\end{table}

The results match the spread mechanism in Section~\ref{sec:abl-corruption}.
Nature and animal prompts have high top-1 agreement ($80\%$/$78\%$) and low
mean regret ($0.005$/$0.008$). Animals still have the lowest median $\rhosp$
($0.833$), indicating swaps among near ties. Human and object prompts have
more regret ($0.045$/$0.066$), where articulated motion and object integrity
create meaningful candidate spread. Stylized prompts have the largest mean
regret ($0.112$), but $n{=}4$ and two prompts contribute $0.302$ and $0.148$;
the other two contribute zero. The worst prompt is an object case (``robot
arm assembling electronics'', $\rhosp = -0.07$, regret $0.408$).
Non-photorealistic and fine-mechanism content may depend more on
mid-trajectory dynamics. Reusing transformation vectors can then perturb
verifier-relevant scores. Examples include stop-motion stutter, unfolding
origami, and articulated manipulation. This interpretation motivates
category-aware or cache-signal-aware adaptation.
These are diagnostics, not population claims. The tested spread probe does
not adapt successfully (Section~\ref{sec:abl-adaptive}). A powered analysis
could use the \VBenchMetric{} suite (Section~\ref{sec:exp-vbench});
Table~\ref{tab:categories} only identifies candidate failure modes.

% ======================================================================
\section{Verifier bias}
\label{app:verifier-bias}

For the same winning seed, \ImageRewardMetric{} scores the cached video above its
full-compute twin (Section~\ref{sec:abl-temporal}). Keep $-$ commit is
$+0.008$ / $+0.037$ / $+0.028$ at $\tauc = 0.05$ / $0.10$ / $0.20$
($n{=}50$ pairs each; keep $\geq$ commit on ${\sim}60\%$ of pairs).

\paragraph{Mechanism: motion dampening photographs well.} Cached rollouts
lose motion as $\tauc$ grows: mean optical flow falls $-3.1\%$ at
$\tauc{=}0.10$ and $-8.0\%$ at $\tauc{=}0.20$, with less flicker
(Section~\ref{sec:abl-temporal}). A frame verifier sees $8$ stills, where
slower scenes reduce blur and transient deformation. \LPIPSMetric{} of
$0.122$--$0.170$ confirms a preference between distinct outputs, not score
noise on duplicates. The verifier rewards motion reduction rather than
detecting caching itself. Slower scenes also produce more canonical poses,
which frame-level aesthetic and alignment models favor.

\paragraph{Implications for verifier-guided search.} A common bias across
$N$ candidates cancels in the within-prompt $\argmax$. Candidate-specific
motion loss can still cause rank swaps (Section~\ref{sec:exp-ranking}), which
the measured regret already includes. Regret and capture compare cached
selections with full-compute scores, so both include this selection effect.
The same verifier scores
\cachedsearch{}-keep (Table~\ref{tab:main}), so its ${\sim}99\%$ nominal
capture is an upper bound. This is a caching instance of verifier
over-optimization \citep{ma2025inference}. Recommit removes the bias from
delivered quality. Motion-aware or video-native verifiers
\citep{xu2024visionreward,he2024videoscore} may reduce it, but
\VideoScoreMetric{}
shares the preference, including on dynamic degree
(Section~\ref{sec:abl-verifier}). Deployments should report the
keep-vs-commit score delta and mean optical flow. This check needs one extra
scoring pass on $2 \times 50$ videos.

% ======================================================================
\section{Scaling discussion}
\label{app:scaling}

\cachedsearch{} depends on each (model, $\tauc$, verifier) triple's speedup
$C_f/C_c$ and ranking fidelity $r$; Section~\ref{sec:experiments} and the
cross-model results in Section~\ref{sec:abl-cog} show how these measures of
denoising redundancy change with scale.

\paragraph{Hypotheses.} \textbf{(H1) Cacheable redundancy grows with
scale.} Larger, higher-resolution, or longer generations may have flatter
mid-trajectory dynamics. More capacity may go to refinement across
spatially redundant tokens, so a fixed-$\tauc$ cache should skip more.
Forcing-KV \citep{forcingkv2026} supports the token-count axis: its
KV-compression speedup rises from $1.35$--$1.5\times$ at $480$p to
$2.82\times$ at $1080$p. \textbf{(H2) Ranking preservation does not follow
redundancy.} Fidelity asks whether cache error stays orthogonal to the
verifier-relevant score direction. It depends on candidate spread relative
to cache-induced score noise, not parameter count. \textbf{(H3) Self-limiting
corruption transfers across scale.} At 1.3B, low regret follows the positive
spread--fidelity association, not $\rhosp$ alone
(Proposition~\ref{prop:spread}). This covariance may survive even when
$\rhosp$ falls.

\paragraph{Cross-family data point: CogVideoX-5B \citep{yang2024cogvideox}.}
We ran the full
gate protocol (same $50$ prompts, $8$ seeds, full vs.\ cached
$\tauc{=}0.10$, \ImageRewardMetric{} verifier) on this model at its native
configuration ($49$ frames,
$480{\times}720$, $50$ steps, guidance $6.0$; its batch-concatenated CFG
requires the single-branch variant of the cache wrapper,
Section~\ref{sec:abl-cog}). Results, computed by the same script as
Table~\ref{tab:categories}: candidate speedup $2.06\times$
($91.3 \to 44.3$\,s); median $\rhosp = 0.762$ (mean $0.684$, p10 $0.352$;
$20/50$ prompts below $0.7$); top-1 agreement $48\%$; mean regret $0.124$
reward units against a $0.517$ random-pick baseline: \textbf{${\sim}76\%$} of search gains retained (ratio of means; $75.2\%$ as
mean per-prompt ratio), with median regret $0.008$ and $48\%$ of prompts at
exactly zero regret. The spread--fidelity association remains positive but
weaker (corr $= +0.17$).

CogVideoX is consistent with H1: speedup rises $1.97\times \to 2.06\times$
at the same $\tauc$ on a model with $3.8\times$ the parameters. It also
supports H2. At 5B, median $\rhosp$ falls $0.905 \to 0.762$, and capture
falls $94\% \to 76\%$. Architecture, training data, resolution, clip length,
and CFG implementation confound this cross-family comparison.

The within-family rung separates these effects. Wan2.1-14B
\citep{wan2025wan} matches its
$1.3$B sibling's median $\rhosp = 0.905$ under the same protocol
(Table~\ref{tab:scale}, Section~\ref{sec:abl-cog}). It reaches a reproduced
$2.05\times$ speedup and $87.5\%$ capture across a $10.8\times$ parameter
gap. Thus the CogVideoX drop is a family effect from an uncalibrated $\tauc$,
largely recovered by the per-model sweep in Section~\ref{sec:abl-cog}.
The six-model grid in Table~\ref{tab:scale} agrees.
HunyuanVideo-13B \citep{kong2024hunyuanvideo} and
LTX-Video-2B \citep{hacohen2024ltxvideo} also lose ranking fidelity at fixed
$\tauc$ (capture $79.9\%$ /
$67.6\%$), while matching or exceeding Wan speedups. Redundancy and fidelity
do not track parameter count.

We therefore measure $r$ for each model. A cheap pilot needs one prompt set
and two rollout modes, exactly the calibration in Appendix~\ref{app:theory}.
Even CogVideoX's uncalibrated $r$ retains three quarters of the search gain
at roughly half the exploration cost in \emph{recommit} mode. Delivery
quality is unchanged by construction.

The non-parameter axes provide direct H1 tests
(Section~\ref{sec:abl-schedule}). At fixed $\tauc = 0.10$ on Wan2.1-1.3B
\citep{wan2025wan},
the skip fraction rises $32\% \to 51\% \to 66\%$ from $25$ to $100$
denoising steps. The log-linear trend is ${\approx}{+}24.6$ skip points per
$e$-fold of $T$ (Figure~\ref{fig:schedule-scaling}). Schedule length is the
strong axis for step-wise caching. At $720{\times}1280$
($2.3\times$ the pixels), the skip fraction moves only $51\% \to 53\%$.
Forcing-KV's resolution gains use a different mechanism, AR KV compression.
Parameter count is also flat within the family: $51\%$ at $1.3$B and $52\%$
at $14$B. Future work should fit skip fraction, iso-quality speedup, and
ranking fidelity across parameters, resolution, and length. The CogVideoX
$2$B${\leftrightarrow}5$B rung remains open.

Together, these measurements make H1 axis-specific and support H2 across
model sizes. They provide limited evidence for the third hypothesis. A
systematic grid should vary one axis at a time and calibrate $\tauc$ per
model. It should separate
cacheability from ranking preservation and test whether the spread--fidelity
association persists.

% ======================================================================
\section{Reproducibility}
\label{app:repro}

Public checkpoints, fixed seeds, and the configurations below reproduce all
results. Generation is deterministic for (prompt, seed, $\tauc$) on a fixed
stack, so each cached or full video can be re-materialized exactly. We will
release the caching wrapper, search harness, analysis and figure scripts,
prompt lists, seeds, configurations, and evaluation scores.

\paragraph{Models and inference configurations.}
The checkpoints are \wan{} \citep{wan2025wan}, Wan2.1-T2V-14B
\citep{wan2025wan}, Wan2.2-TI2V-5B \citep{wan2025wan22,wan2025wan},
CogVideoX-5B \citep{yang2024cogvideox}, HunyuanVideo-13B
\citep{kong2024hunyuanvideo}, and LTX-Video-2B
\citep{hacohen2024ltxvideo}.
\begin{table}[h]
\caption{\textbf{Exact model-card generation configurations for all six
models.} Native resolution, length, guidance, VAE precision, and measured
full-to-cached latency at $\tauc=0.10$.}
\label{tab:repro-config}
\centering
\small
\setlength{\tabcolsep}{4pt}
\begin{tabular}{lcccc}
\toprule
model & resolution $\times$ frames & guidance & VAE &
$C_f \to C_c$ ($\tauc{=}0.10$) \\
\midrule
Wan2.1-T2V-1.3B & $480 \times 832 \times 81$ & $5.0$ & fp32 &
$68.3 \to 34.7$\,s \\
LTX-Video 2B & $480 \times 704 \times 161$ & $3.0$ & bf16 &
$10.1 \to 3.8$\,s \\
Wan2.2-TI2V-5B & $704 \times 1280 \times 121$ & $5.0$ & fp32 &
$130.1 \to 63.5$\,s \\
CogVideoX-5B & $480 \times 720 \times 49$ & $6.0$ & fp32 &
$91.3 \to 44.3$\,s \\
HunyuanVideo 13B & $480 \times 720 \times 61$ & $6.0$ & bf16 &
$87.1 \to 39.8$\,s \\
Wan2.1-T2V-14B & $480 \times 832 \times 81$ & $5.0$ & fp32 &
$341.3 \to 166.8$\,s \\
\bottomrule
\end{tabular}
\end{table}

Each model uses $T=50$ and its native resolution, length, and guidance
configuration. The cached and full arms use the same model-specific
precision.

\paragraph{Cache wrapper.} The adaptive cache follows the
transformation-vector formulation of EasyCache \citep{zhou2025easycache}
(Section~\ref{sec:caching}): cached quantity $\Delta = v(x) - x$; skip returns
$x + \Delta$. The skip rule accumulates the relative input drift
$a \mathrel{+}= \lVert x - x_{\text{ref}} \rVert_F / \lVert x_{\text{ref}}
\rVert_F$ against the last computed input $x_{\text{ref}}$ and
recomputes when $a > \tauc$ (resetting $a$, $x_{\text{ref}}$, $\Delta$).
Warmup and cooldown windows of $K_w = K_c = 5$ steps always compute, as
does any step with an empty cache. Defaults: $\tauc = 0.10$
($\{0.05, 0.20\}$ in the sweep). At $\tauc{=}0.10$ the wrapper skips
$26/50$ steps on Wan and $27/50$ on CogVideoX-5B. Table~\ref{tab:scale}
gives all six models' skip fractions, and Table~\ref{tab:schedule} gives
their dependence on schedule length and resolution. The wrapper is reset
between generations.

\paragraph{Protocol, seeds, verifier.} Grids: $50$ prompts $\times$ seeds
$0$--$7$ $\times$
$\{\text{full}, \text{cached}\}$, i.e.\ $800$ scored rollouts per model at
$\tauc{=}0.10$, plus $400$ cached rollouts per additional $\tauc$ (the full
references are $\tauc$-independent and reused). Noise is drawn from a
per-candidate \texttt{torch.Generator} seeded with the candidate's seed.
The \VBenchMetric{} suite and \VBenchTwoMetric{} use the same paired
$8$-seed protocol. Verifier: \ImageRewardMetric{} v1.0 averaged over $8$
uniformly spaced frames (frames converted to uint8 before scoring). The
temporal study (Section~\ref{sec:abl-temporal}) materializes the $50$ winner pairs per $\tauc$
($300$ videos) and scores DINO subject / CLIP background consistency,
motion smoothness, temporal flicker, mean optical flow, and \LPIPSMetric{}.

\paragraph{Hardware.} All rollouts and scoring ran on NVIDIA GH200 GPUs, one
GPU per rollout.

\section{Qualitative examples}
\label{app:qualitative}

The frame sequences themselves are the video evidence. The search comparison
shows the central single-versus-search result, and the gallery adds two
\VBenchMetric{} cases. The four-strategy grid includes a different-pick failure,
while the model strips test cross-model fidelity. The final figures show
same-seed appearance fidelity, flow reduction, and motion-focused
keep-versus-commit sequences.

\subsection{Search comparison on the \VBenchMetric{} suite}
% SHARED BLOCK (both editions; single source). The qualitative search
% comparison (fig:qual-search) plus its describing paragraph. Report edition:
% inlined in sections/experiments.tex (Section exp-vbench). Conference edition:
% hosted in the appendix (app:qualitative). Figure + annotations generated by
% code/paper_figs/make_figs_qual.py; every reward is the gate's measured
% ImageReward record.
\begin{figure}[!tb]
  \centering
  \includegraphics[width=\linewidth]{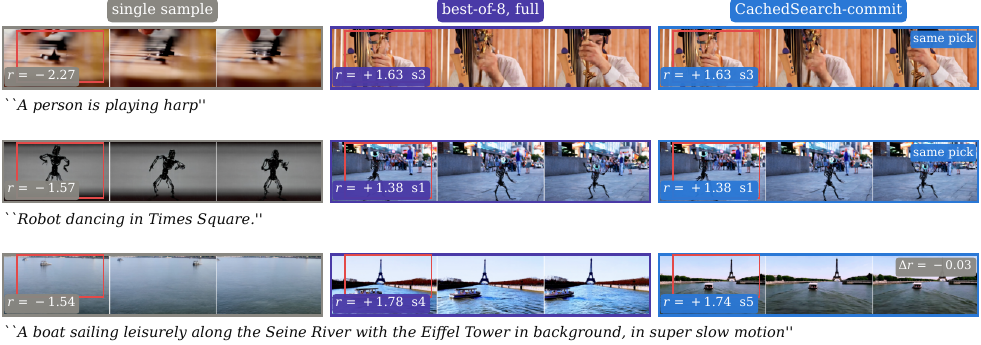}
  \caption{\textbf{Search yields visible gains, and cached exploration
  finds the same winner cheaper.} Selected \VBenchMetric{} prompts compare
  single, full best-of-$8$, and commit at $\tauc=0.10$. Chips give
  measured \ImageRewardMetric{} and seed; red boxes mark discriminating regions.}
  \label{fig:qual-search}
\end{figure}

Figure~\ref{fig:qual-search} makes the aggregate numbers concrete on selected
\VBenchMetric{} prompts. Where the single sample fails (a motion-blurred smear
in place of a harpist, a robot dancing on a black studio backdrop instead of
Times Square, an empty river with no Eiffel Tower), best-of-8 finds a candidate
that satisfies the prompt, worth $+2.9$ to $+3.9$ reward.
\cachedsearch{}-commit delivers the \emph{same} video as full-compute search on
two of the three prompts (identical winning seed, hence pixel-identical
delivery after recommit) and an equal-reward alternative on the boat, while its
exploration rollouts cost half as much.

Figure~\ref{fig:qual-gallery} collects the retrieval-validated picks
omitted from Figure~\ref{fig:qual-search}, using the same three columns.
\begin{figure}[!tb]
  \centering
  \includegraphics[width=\linewidth]{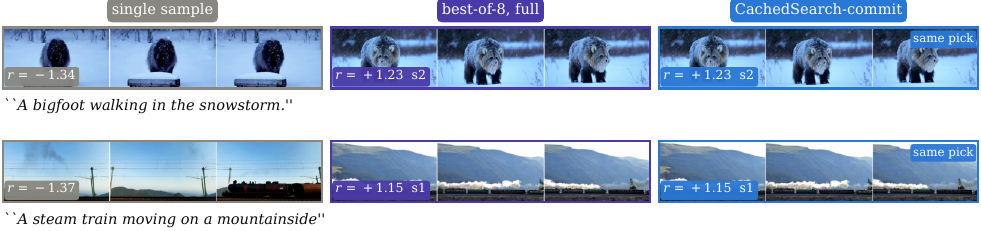}
  \caption{\textbf{Extended qualitative comparisons confirm the same search
  behavior.} Additional \VBenchMetric{} prompts under the
  Figure~\ref{fig:qual-search}
  protocol: single, full best-of-$8$, and commit at $\tauc=0.10$.
  Chips report measured \ImageRewardMetric{} and seed; both are same-pick
  cases.}
  \label{fig:qual-gallery}
  \label{fig:qual-search-gallery}
\end{figure}

\subsection{Different-pick failure}
% SHARED BLOCK (both editions; single source). The four-method delivery grid
% (fig:qual-methods) plus its describing paragraph. Report edition: inlined in
% sections/experiments.tex right after fig:qual-search. Conference edition:
% hosted in the appendix (app:qualitative). Generated by
% code/paper_figs/make_figs_qual.py (fig_methods); every reward and every cost
% fraction is a measured record, nothing hand-entered.
\begin{figure}[!tb]
  \centering
  \includegraphics[width=0.97\linewidth]{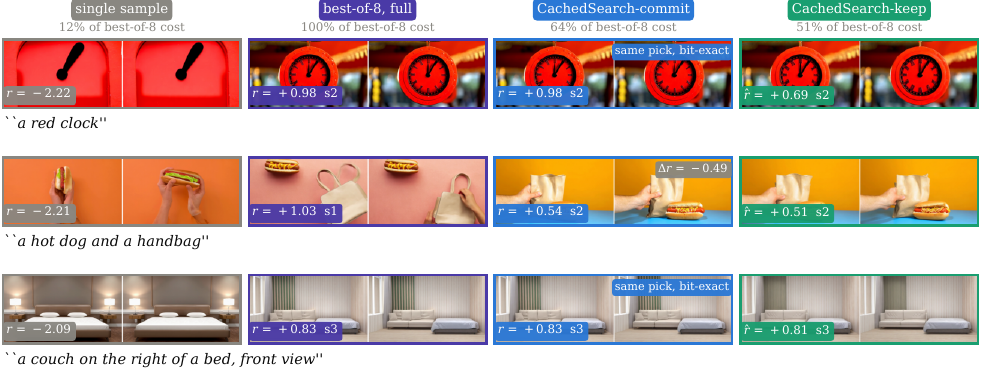}
  \caption{\textbf{All four delivery strategies compared under measured
  budgets.} Selected \VBenchMetric{} prompts form rows; columns are single,
  best-of-$8$, keep, and commit. Headers give cost relative to full
  best-of-$8$; corner chips give measured delivered \ImageRewardMetric{}.}
  \label{fig:qual-methods}
\end{figure}

Figure~\ref{fig:qual-methods} extends the comparison to all four delivery
strategies under their measured budgets. The pattern is the paper in one
image: the single sample misses the prompt (a melted clock face, a missing
hot dog, no couch at all), search finds a satisfying candidate, and cached
exploration delivers that same candidate at $64\%$ (commit, bit-exact) or
$51\%$ (keep, visually near-identical) of best-of-8's cost. The middle row is
the honest case: cached ranking errs, and the miss costs $0.49$ reward out of
a $+3.2$ search gain.

Figure~\ref{fig:qual-methods} is the failure example: its middle row shows
cached exploration selecting a different seed and losing $0.49$ reward from
a $+3.2$ search gain.

\subsection{Cross-model strips}
% SHARED BLOCK (both editions; single source). The cross-model qualitative
% figure (fig:qual-models). Report edition: inlined in sections/ablations.tex
% (Section abl-cog, before the takeaway). Conference edition: hosted in the appendix
% (app:qualitative) with NO main-text pointer (conference page budget).
% Generated by code/paper_figs/make_figs_qual.py from the seed-deterministic
% a0-qual-* spool regeneration (seed 0, model-card configs of
% Table tab:repro-config); rewards and skip counts read from the run's jsonl.
\begin{figure}[!tb]
  \centering
  \includegraphics[width=0.96\linewidth]{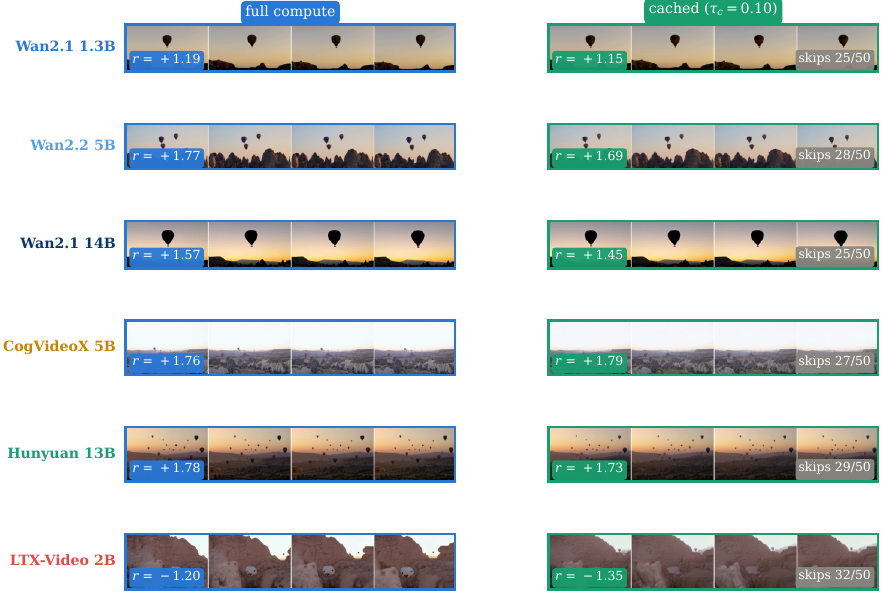}
  \caption{\textbf{Calibrated backbones track their full-compute twins;
  over-driven LTX \citep{hacohen2024ltxvideo} visibly drifts.} One shared prompt and seed across six
  models at $\tauc=0.10$, with full and cached rows under each model-card
  recipe. Chips report \ImageRewardMetric{} and reused denoising steps.}
  \label{fig:qual-models}
\end{figure}

Figure~\ref{fig:qual-models} completes the cross-model picture
qualitatively: the same wrapper at the same $\tauc$ produces cached
rollouts that remain faithful drafts on every backbone operating in its
calibrated skip regime, drifting only in instance-level detail, while
the one model the fixed threshold over-drives (LTX) is also the one
whose delivered content visibly diverges. Visual fidelity and ranking
fidelity degrade together, both governed by the skip fraction, not by
model identity.

\subsection{Fidelity and motion}
% SHARED BLOCK (both editions; single source). The same-seed full-vs-cached
% fidelity figure (fig:qual-fidelity). Report edition: inlined in
% sections/ablations.tex (Section abl-temporal, after the LPIPS block). Conference
% edition: hosted in the appendix (app:temporal-detail). Generated by
% code/paper_figs/make_figs_qual.py; rewards from the E1 records, LPIPS from
% the recomputed lpips_fixed values quoted in the LPIPS block.
\begin{figure}[!tb]
  \centering
  \includegraphics[width=\linewidth]{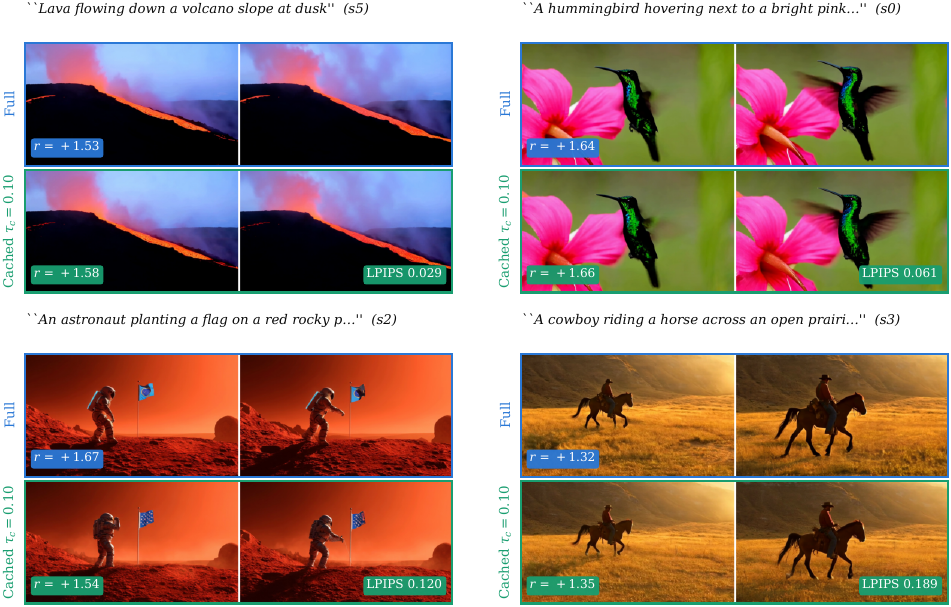}
  \caption{\textbf{Cached twins preserve subject, layout, and palette;
  differences concentrate in detail and motion phase.} Four same-seed
  winner pairs at $\tauc=0.10$, with two frames per video. Chips report
  measured \ImageRewardMetric{} and pairwise \LPIPSMetric{}.}
  \label{fig:qual-fidelity}
\end{figure}

Figure~\ref{fig:qual-fidelity} shows what these \LPIPSMetric{} values look like:
the selected pairs span the suite distribution (mean $0.142$, range
$0.021$--$0.343$, $n=50$), and the widest shown is $0.189$. Across that
range, the cached rollout and its full-compute twin
agree on subject, layout, and palette, and disagree in fine texture and
motion phase: the cached rollout is a faithful draft of the video the
verifier is asked to rank, which is the property candidate ranking needs.

% SHARED BLOCK (both editions; single source). Motion dampening made visible
% (fig:qual-flow): same-seed full-vs-cached winner pairs at the aggressive arm
% with temporal-mean optical-flow maps. Report edition: inlined in
% sections/ablations.tex (sec:abl-temporal, after fig:qual-fidelity).
% Conference edition: hosted in the appendix (app:qualitative, motion
% subsection). Generated by code/paper_figs/make_figs_qual.py (fig_flow); the
% flow statistics are computed live from the delivered videos.
\begin{figure}[!tb]
  \centering
  \includegraphics[width=\linewidth]{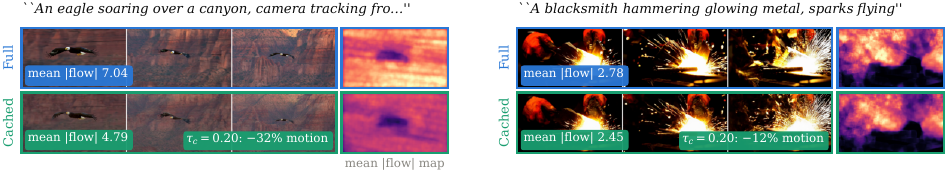}
  \caption{\textbf{Flow maps expose motion dampening that learned judges
  miss.} Two same-seed, high-motion winner pairs at $\tauc=0.20$:
  full compute above, cached below. Each block shares the color scale for
  temporal-mean Farneback flow magnitude.}
  \label{fig:qual-flow}
\end{figure}

Figure~\ref{fig:qual-flow} preserves layout and identity but reduces mean
flow by $32\%$ on the eagle and $12\%$ on the blacksmith, local examples of
the population's $8\%$ reduction at $\tauc=0.20$
(Section~\ref{sec:abl-temporal}).

Figure~\ref{fig:qualitative} shows the $\tauc = 0.20$ winner pairs with the
largest keep-vs-commit frame difference. Keep preserves layout, identity,
and style, but reduces displacement, consistent with the measured $-8\%$
mean flow (Section~\ref{sec:abl-temporal}).

\begin{figure}[!tb]
  \centering
  \includegraphics[width=\linewidth]{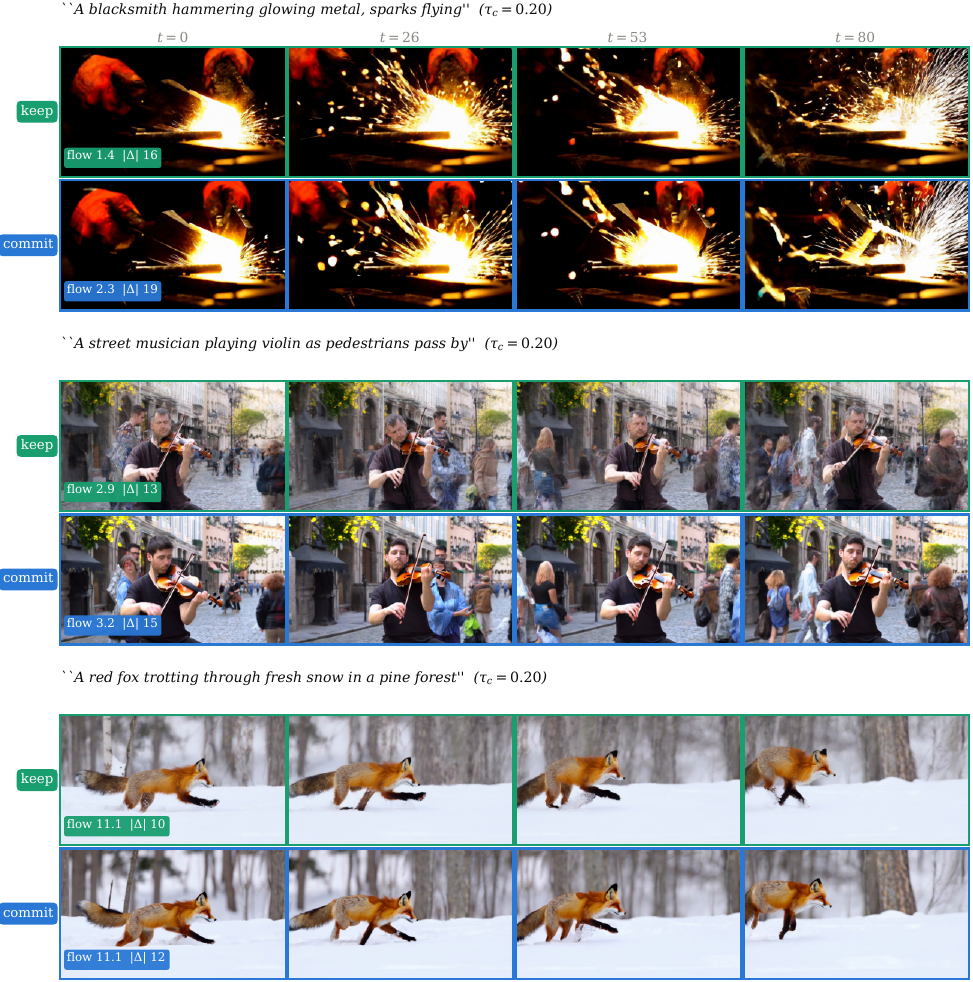}
  \caption{\textbf{Static content is preserved under keep-draft; motion
  is dampened ($-8\%$ mean flow at this $\tauc$,
  Section~\ref{sec:abl-temporal}).} Keep (top rows) vs.\ commit (bottom rows)
  winner videos at $\tauc = 0.20$, four evenly spaced frames;
  adversarial selection: pairs chosen by largest mean frame difference
  among the $50$ winner pairs.}
  \label{fig:qualitative}
\end{figure}

\section{Limitations}
\label{app:limitations}

We use \ImageRewardMetric{} as the verifier. A better temporally aligned
reward model would be preferable, but no publicly available one is reliable
enough for this role today. \VideoScoreMetric{} gives weaker ranking
preservation and shares the preference for motion-dampened cached outputs, so
recommit remains the default and motion-sensitive uses need direct temporal
audits, including \LPIPSMetric{}. The study covers six public checkpoints
from four families and best-of-$N$ seed selection through $N=16$; each new
architecture requires threshold calibration, and searches that verify
intermediate states need a separate audit.

\end{document}